\documentclass{llncs}
\newcommand{\hide}[1]{}  
\usepackage{times}
\usepackage{graphicx} 
\usepackage{changepage} 
\usepackage{latexsym} 
\usepackage{wrapfig} 
\usepackage{mathtools} 
\usepackage{xcolor}
\usepackage{soul}
\usepackage[utf8]{inputenc}
\usepackage[small]{caption}
\usepackage{enumitem} 
\usepackage{xcolor} 
\usepackage{amsmath}       
\usepackage{amssymb}            
\usepackage{tikz}   
\usetikzlibrary{arrows.meta,fit,decorations.markings,intersections,arrows,positioning,shapes.misc}  
\usepackage{tcolorbox} 
\tcbuselibrary{raster,skins}
\usepackage{rotating}    
    
\usepackage{tikz-cd}  
\newcommand{\FFADF}{\ensuremath{F^{\ADF}}} 
 
\newcommand{\FMMA}{\ensuremath{\mathcal{F}}}
\newcommand{\FFMMA}{\ensuremath{F}}

\newcommand{\FADF}{\ensuremath{\mathcal{F}^{\ADF}}}

\newcommand{\bundle}{\textsf{bundle}}

\newcommand{\maxi}{\textsf{twoVal}}

\newcommand{\qdot}{\ensuremath{\dot{Q}}}
\newcommand{\fqdot}{\ensuremath{f_{\qdot}}}

\newcommand{\ADF}{\textsf{ADF}}

\newcommand{\Dung}{\textsf{D}}

\newcommand{\inL}{\textsf{in}}
\newcommand{\outL}{\textsf{out}}
\newcommand{\undecL}{\textsf{undec}}

\newcommand{\pre}{\ensuremath{\textsf{pred}}}  
\newcommand{\suc}{\ensuremath{\textsf{succ}}}

\newcommand{\andC}{\textsf{and}}

\newcommand{\ryuta}[1]{{\color{purple}{#1}}}

\begin{document}
     \title{Broadening Label-based 
     Argumentation Semantics \\ with 
     May-Must Scales (May-Must Argumentation)} 

\author{Ryuta Arisaka \and Takayuki Ito\institute{Nagoya Institute of 
    Technology, Nagoya, Japan\\ email: ryutaarisaka@gmail.com, 
    ito.takayuki@nitech.ac.jp}}
    \maketitle 
\begin{abstract}     
   The semantics as to which set of arguments in a given argumentation 
   graph may be acceptable (acceptability 
   semantics) can be characterised in a few different 
   ways. Among them, labelling-based approach 
   allows for concise and flexible determination of 
   acceptability statuses  of arguments through 
   assignment of a label indicating 
   acceptance, rejection, or  
   undecided to each argument. 
   In this work, we contemplate a way of 
   broadening it 
   by accommodating may- and must- conditions 
   for an argument to be accepted or rejected, 
   as determined 
   by the number(s) of 
   rejected and accepted attacking arguments.   
   We show that the broadened label-based semantics 
   can be used to express
   more mild indeterminacy than 
    \mbox{inconsistency} 
   for acceptability judgement 
   when, for example, it may be the case that an 
   argument is accepted and when it 
   may also be the case that it is rejected.   
   We identify that 
   finding which conditions a labelling satisfies 
   for every
   argument can be an undecidable problem, 
   which has an unfavourable implication 
   to existence of a semantics. 
   We propose to address this problem by enforcing  
   a labelling to  maximally 
   respect the conditions, while keeping 
   the rest that would necessarily cause   
	non-termination labelled undecided. 
   Several semantics will be presented and 
	the relation among them will be noted. Towards the end, 
	we will touch upon possible research directions 
	that can be pursued further. 
\end{abstract}  
\section{Introduction}\label{section_introduction}  
Dung formal argumentation \cite{Dung95} 
provides an abstract view of argumentation as a 
graph of: nodes representing arguments; and 
edges representing  
attacks from the source arguments to the target arguments.   
Dung argumentation allows us to determine 
 which arguments are acceptable in a given argumentation.  

While the determination in Dung's seminal paper 
is through conflict-freeness: {\it no members of a set 
attack a member of the same set}, and 
defence: {\it a set of arguments defend an argument 
just when any argument attacking the argument 
is attacked by at least one member of the set},  
there are other 
known approaches. With labelling, 
a labelling function assigns 
a label 
indicating either of: acceptance, rejection, and undecided 
(see e.g.
\cite{Caminada06,Jakobovits99}) to each argument, 
which offers a fairly concise and also flexible 
(see e.g. \cite{Brewka13}) characterisation 
of arguments' acceptability, based, in case 
of \cite{Caminada06,Brewka13}, 
just on the labels of the arguments it is 
attacked by.  Acceptance and 
rejection conditions may be defined uniformly 
for every argument \cite{Caminada06,Jakobovits99}, or per 
argument, as in Abstract Dialectical 
Frameworks (\ADF) \cite{Brewka13}, where 
acceptance status of an argument is 
 uniquely 
determined for
each combination of the acceptance/rejection/undecided 
labels of associated arguments. 
\subsection{Labelling-approach  
with may-must scales}\label{section_labelling_approach} 
In this work, we aim to further explore 
the potential of a labelling-approach by 
broadening the labelling in \cite{Caminada06}  
with what we term {\it may-must acceptance scale} 
and {\it may-must rejection scale}, to be assigned 
to 
each argument. The may-must acceptance scale (resp. 
may-must rejection scale) of an argument 
is specifically a pair of natural numbers 
$(n_1, n_2)$ with $n_1$ indicating 
the minimum number of its attackers that need to be rejected 
(resp. accepted) in order that the argument 
can be accepted (resp. rejected) 
and $n_2$ the minimum number of 
its attackers that need to be rejected (resp. 
accepted) in  
order that it must be accepted (resp. rejected). 
That is, $n_1$ is the {\it may} condition, while
$n_2$ is the {\it must} condition, for acceptance (resp.
rejection) of 
the argument. 

Thus, not only can they express exact 
conditions for acceptance/rejection 
of an argument as with \cite{Caminada06,Brewka13}, 
they can additionally describe minimal requirements  
to be satisfied in order that the argument can be 
accepted/rejected. The may-must scales lead to 
the following distinction 
to acceptance and rejection of an argument. 
\begin{enumerate}[leftmargin=2.5cm] 
   \item[(1).] It may be accepted. $\quad\quad\quad\quad\quad$
       (2). It must be accepted. 
   \item[(i).] It may be rejected. $\quad\quad\quad\quad\quad\ $
       (ii). It must be rejected. 
\end{enumerate} 
Since each argument has its own may-must 
acceptance and rejection scales, depending on 
the specific numerical values given to them, 
we may have several combinations in  
\{(1) (2) neither\}$-$\{(i) (ii) neither\}.  
While (2)$-$(neither), (neither)$-$(ii), 
(2)$-$(ii) (i.e. the argument is both 
accepted and rejected at the 
same time) and 
(neither)$-$(neither) (i.e. it is neither accepted nor 
rejected) 
deterministically 
indicate acceptance, rejection, undecided and undecided 
for the acceptability status of the argument, 
the other combinations are more interesting.  

Let us consider for example (1)$-$(i). 
Unlike (2)$-$(ii) 
which leads to immediate logical inconsistency, 
(1)$-$(i) expresses milder indeterminacy, 
since we can simultaneously assume the possibility of the argument 
to be accepted and the possibility of the same argument 
to be rejected without logical 
contradiction. In fact,  
there may be more than 
one suitable label from 
among acceptance, rejection and undecided as 
an acceptability status of an argument 
without necessarily 
requiring the acceptance statuses of the arguments 
attacking it to also change. That is, in a not-non-deterministic
labelling, once the labels of the attackers are determined, 
there is only one label suitable for the attacked, but 
it is not always the case with a non-deterministic labelling. Such non-deterministic 
labels of argument(s) can trigger disjunctive 
branches to the labels of those arguments 
attacked by them. 

\subsection{Motivation for may-must 
scales}\label{section_motivation_for}     
In real-life argumentation, an argument which is attacked 
by a justifiable argument but by no other arguments 
can be seen differently from an argument which is 
attacked by a justifiable argument  and which 
is also attacked by a lot more defeated 
(rejected) arguments. 
For example, if that argument is a scientific theory, 
one interpretation of the two cases is that, 
in the first case, it meets an objection 
without it having 
stood any test of time, and, in the second case,
even though it is not defended against one objection, 
it withstood all the other objections, a lot more of them 
in number. Such an interpretation takes us 
to the supposition that an 
argument, if found 
out to withstand an objection, attains greater credibility, 
that is to say, that an attacker being rejected 
has a positive, or at least a non-negative,
impact on its acceptance.  

Coupled with the other more standard intuition that 
an attacker of an argument being accepted has a non-positive 
impact on the argument's acceptance, we see that
\begin{itemize}
   \item the larger 
the number of rejected attackers is, 
the more likely it can become that the 
argument is accepted, and 
    \item the larger the number of accepted attackers is,
the more likely it can become that the argument is rejected, 
\end{itemize} 
until there comes a moment where both acceptance and rejection 
of the argument become so compelling, with 
sufficient numbers of rejected and accepted attackers, 
that its acceptance 
status can no longer be determined. 
As with any reasonable real-life phenomenon, 
the acceptance and rejection judgement 
can be somewhat blurry, too. Introduction of  
the may- conditions allows the softer boundaries 
of acceptance and rejection to be captured 
based on the number(s) of accepted and rejected attacking 
arguments.  \\

Moreover, with studies of argumentation 
expanding 
into multi-agent systems, 
for argumentation-based 
negotiations (Cf. two surveys \cite{Rahwan03,Dimopoulos14}
for mostly two-party negotiations 
and a recent study \cite{Arisaka2020b} for concurrent negotiations among 
3 and more parties), 
strategic dialogue games and persuasions \cite{AHT19,ArisakaSatoh18,Grossi13,Hadjinikolis13,Hadoux15,Hadoux17,Hunter18,Kakas05,Parsons05,Rahwan09,Rienstra13,Riveret08,Sakama12,Thimm14}, 
and others, it is preferable 
that 
an argumentation theory 
be able to accommodate a different nuance 
of arguments' acceptability {\it locally} 
per argument, 
and yet somehow in a {\it principled} manner 
with a rational explanation to the cause of the nuances. 
Future applications 
into the domain in mind, may-must scales 
are given to each argument, like 
local constraints in \ADF, ensuring 
the locality. Like in argumentation 
with graded acceptability \cite{Grossi19} 
(see below for comparisons), 
however, may-must conditions 
are rooted in `endogenous' information 
of an argumentation graph, 
to borrow the expression in \cite{Grossi19}, 
namely the cardinality of attackers, to retain 
a fair level of abstractness defining 
monotonic conditions. Specifically, a may- or a must- 
condition 
is satisfied minimally with $n$ 
accepted or rejected attacking arguments, 
but also with any $(n \leq)\ m$ accepted or 
rejected attacking arguments. As we will show, 
it offers an easy expression of, for example, 
possibly accepting an argument when 
80\% of attacking arguments are rejected; 
accepting an argument when 90\% of attacking arguments 
are rejected; possibly rejecting an argument 
when 40\% of attacking arguments (but at least 1) are accepted; 
and rejecting an argument when 50\% of attacking 
arguments (but at least 1) are accepted. 
\subsection{Related work}\label{section_related_work}  
Resembling situations are 
rather well-motivated in the literature. 
Argumentation with graded 
acceptability \cite{Grossi19} 
relaxes conflict-freeness and defence in Dung abstract 
argumentation. For conflict-freeness, it permits 
a certain number $k_1$ of attackers to be accepted simultaneously 
with the attacked (this idea alone is preceded 
by set-attacks \cite{Nielsen06}
and conflict-tolerant argumentations \cite{Gabbay14,Leite11,Arieli12,Dunne11}). For defence, it allows 
the defence by a set of arguments for the attacked 
to occur when a certain number $k_2$ of its attackers 
are attacked by a certain number $k_3$ of members 
of the set. Our work follows the general idea 
of conditionalising acceptance statuses of 
arguments on the cardinality of accepted 
and rejected attacking arguments. 
Indeed, $k_1+1$ corresponds to the must- condition of a may-must 
rejection scale in this work. 
On the other hand, unlike in \cite{Grossi19} 
where dependency of 
acceptability status of 
an argument on the attackers of its attackers 
is enforced due to $k_2$ and $k_3$, 
we are more conservative 
about the information necessary 
for determining acceptability status(es) of an argument. 
We have it obtainable purely from  
its immediate attackers. Also, may- conditions  
are not considered in \cite{Grossi19}. 
In particular, 
while both may- and must- conditions of 
a may-must acceptance scale of an argument interact with 
those of its may-must rejection scale (see section 
\ref{section_labelling_approach}), the interaction 
between the non-positive and the non-negative effects
on the acceptance of the argument  
is, as far as we can fathom, not primarily 
assumed in \cite{Grossi19}.

Ranking-based argumentations (Cf. a recent 
survey \cite{Bonzon16}) order arguments 
by the degree of acceptability. There are 
many conditions around the ordering, 
giving them various flavour. Ones that 
are somewhat relevant to our setting 
(see section \ref{section_motivation_for}) are 
in a discussion in \cite{Cayrol05b}, 
where we find the following descriptions: 
\begin{itemize} 
  \item the more defence branches an argument has, 
       the more acceptable it 
       becomes. 
  \item the more attack branches an argument has, 
       the less acceptable it 
       becomes.  
\end{itemize} 
Here, a branch of an argument is a chain 
of attacking arguments having the argument as the last 
one attacked in the chain, and an attack 
branch (respectively a defence branch) is a branch 
with an odd (respectively even) number of attacks.   
With the principle of reinstatement 
(that an attacker of an attacker 
of an argument has a propagating positive effect 
on the acceptance 
of the argument) assumed, these two conditions 
are clearly reasonable. 
By contrast, our approach assigns may- and 
must- acceptance and 
rejection conditions to each argument; thus, 
the reinstatement 
cannot be taken for granted, 
which generally makes it inapplicable to 
propagate argumentation ranks 
(which 
can be numerical values \cite{daCostaPereira11,Baroni15,Gabbay14,Hunter01,Janssen08}) 
through branches by a set of globally uniform 
propagation rules. 
The cardinality precedence \cite{Bonzon16}: the 
greater the number 
of immediate attackers of an argument, the weaker 
the level of its acceptability is, 
which in itself does not assume 
acceptability statuses of the immediate attackers, 
does not always hold good with our approach, either. \\

For label-based argumentation, 
non-deterministic labelling in
argumentation 
as far 
back as we can see 
is discussed in \cite{Jakobovits99}, where 
an argument may be labelled as either rejected ($\{-\}$) or  
`both accepted and rejected' ($\{+,-\}$) when, 
for example, it is attacked by just one attacker labelled $\{+,-\}$. 
Nonetheless, the criteria of label assignments 
are global (an argument may be accepted ($\{+\}$) just when 
the label(s) of all its attackers contain $-$; 
and may be rejected just when there exists at least one 
argument whose label contains $+$), not covering 
the various nuances to follow from 
locally given criteria. 
On a more 
technical point, while 
letting only $\{+,-\}$ (both accepted 
and rejected) be `undecided' is sufficient  
in \cite{Jakobovits99}, 
that is not enough in a general case, 
as we are to show in this paper. Indeed, with some 
argumentation graph and some may-must acceptance 
and rejection scales, it can happen that 
whether, for example, an argument is accepted, 
or both accepted and rejected, is itself an 
undecidable 
question. 

Abstract Dialectical Frameworks (\ADF) \cite{Brewka13}
 is another labelling approach which accommodates, 
with 3 values \cite{Brewka13} (which 
  has been recently extended to multi-values
   \cite{Brewka18}), 
local acceptance, rejection and 
undecided conditions. The label of an argument is 
determined 
into only one of the 3 labels 
for a given combination of its attackers' 
labels. 
Since its label is determined for every 
combination of its attackers' labels, 
the {\ADF} labelling is concrete and specific.  
By contrast, 
the may-must conditions 
are more abstract in that they only specify, 
like in \cite{Caminada06,Grossi19}, 
the numbers of attackers but not exactly 
which ones. Apart from  
the abstract specification being in line with 
\cite{Caminada06}, 
the level
of abstractness is more favourable   
in our setting, since 
the may- and must- acceptance (or rejection) conditions,  
once satisfied with $n$ rejected (or $n$ accepted) 
attacking arguments should remain 
satisfied 
with $m$ rejected (or $m$ accepted) 
attacking arguments 
so long as $n \leq m$, which they 
can handle in more principled a manner. Moreover, 
acceptance status of an argument 
is evaluated both for acceptance 
and for rejection with the two 
scales. The independent criteria 
are fitting for many real-life 
decision-makings, since it is common that  
assessments as to why 
a proposal (a suspect) may/must be accepted 
(guilty) 
and as to why it (the suspect) 
may/must be rejected (acquitted) 
are separately made by two guidelines (by a prosecutor 
and a defence lawyer) before, based on them,
a final decision is delivered. 
The may- must- 
conditions based on 
the cardinality of 
accepted or rejected attacking arguments 
are, as far as we are aware, 
not considered in {\ADF} including 
\cite{Brewka18,Bogaerts19}. 
Some more technical semantic comparisons to {\ADF} semantics are 
found in a later part of this paper, in \mbox{Section 
\ref{section_semantics}}. 


Fuzziness as a varying attack strength \cite{Janssen08}
and as a varying degree of acceptability of an 
argument 
\cite{daCostaPereira11,Brewka18,Bogaerts19} 
have been discussed in 
the literature, both of which are closely related  
to ranking-based argumentation. 
The kind of fuzziness that we deal with in 
this paper, however,
is not, again borrowing the expression    
in \cite{Grossi19}, about 
`exogeneously given information 
about the relative strength of arguments' or 
the relative degree of acceptability, 
but about the cardinality of rejected/accepted attackers 
required for may- must- acceptance/rejection of the attacked 
argument. 
\subsection{Summary of the contribution and 
the structure of the paper} 
We develop {\it may-must argumentation} by broadening the labelling in \cite{Caminada06}
with a may-must acceptance scale and 
a may-must rejection scale for each argument, 
as we stated in section 1.1, which helps
localise the nuance of acceptability 
of an argument. 
That those conditions only 
specify the numbers of (accepted and rejected)
attacking arguments and that their 
satisfaction conditions  
are monotonic (Cf. section \ref{section_motivation_for}) help the approach 
retain a level of abstractness that 
facilitates a principled explanation 
as to why an argument becomes accepted, rejected, 
or undecided in response to an increase or decrease 
in the number of rejected/accepted attackers of 
its. 

Technically, we identify that finding 
a labelling that satisfies  
local 
criteria for each argument is not always 
possible due to a circular reasoning. 
We address this 
problem by 
enforcing 
a labelling to maximally respect local 
acceptance criteria, while keeping the rest 
that would necessarily cause non-termination 
labelled undecided. 
Several semantics are presented, and 
comparisons are made. 

In Section 2, we will cover technical preliminaries, 
specifically of Dung abstract argumentation labelling semantics  
\cite{Caminada06} 
(in \mbox{Section \ref{section_technical_preliminaries}}). 

In Section 3, we will present technical vehicles of 
may-must argumentation in detail. 

In Section \ref{section_semantics},  we will 
formulate several types of semantics and 
will identify 
the relation between them. 

In Section 5, we will discuss future research directions including 
a fundamental technical problem and adaptation to multi-agent argumentation.

\section{Technical Preliminaries}\label{section_technical_preliminaries}
Dung abstract argumentation \cite{Dung95} considers an argumentation 
as a graph where a node represents an argument and where 
an edge between arguments represents an attack from the 
source argument 
to the target argument. Technically, 
let $\mathcal{A}$ denote a class of abstract entities that 
we understand as arguments, and let $\mathcal{R}$ denote 
a class of all binary relations over $\mathcal{A}$. 
We denote by $\mathcal{R}^A$ a subclass of $\mathcal{R}$ 
which contains all and only members $R$ of $\mathcal{R}$ 
with: $(a_1, a_2) \in R$ only if $a_1, a_2 \in A$. 
Then a (finite) abstract
argumentation is a tuple $(A, R)$ with $A \subseteq_{\text{fin}}
\mathcal{A}$ and $R \in \mathcal{R}^A$.  
$a_1 \in A$ is said to attack $a_2 \in A$ if and only if, 
or iff, $(a_1, a_2) \in R$ holds. We denote 
the class of all Dung abstract argumentations 
by $\mathcal{F}^{\Dung}$. 

One of the main objectives 
of representing an argumentation formally 
as a graph 
is to infer from it which set(s) of arguments may be 
accepted. Acceptability of a set of arguments is determined 
by whether it satisfies certain criteria. 

In this paper, 
we will uniformly use labelling \cite{Caminada06} for 
characterisation 
of the acceptability semantics; readers 
are referred to Dung's original paper \cite{Dung95}   
for equivalent semantic characterisation
through conflict-freeness and defence.\footnote{Similar 
semantic characterisation with the conflict-freeness 
and the defence is, as with 
some of the ranking-based approaches or with \ADF, not 
actually practical 
in this work, since the nuance of an attack in 
Dung abstract argumentation is only one 
of many that are  expressible in our proposal.}

Let $\mathcal{L}$ denote $\{\inL,\outL,\undecL\}$,
and let $\Lambda$ denote the class of all partial 
functions $\mathcal{A} \rightarrow \mathcal{L}$. 
Let $\Lambda^A$ for $A \subseteq 
\mathcal{A}$ denote a subclass of $\Lambda$ 
that includes all (but nothing else) $\lambda \in \Lambda$ 
that is defined for all (but nothing else) members of $A$.   
For the order among members of $\Lambda$, let $\preceq$ be a binary relation over $\Lambda$  
such that $\lambda_1 \preceq \lambda_2$ 
for $\lambda_1, \lambda_2 \in \Lambda$ iff 
all the following conditions hold. (1) There 
is some $A \subseteq_{\text{fin}}
 \mathcal{A}$ such that $\lambda_1, \lambda_2 \in 
\Lambda^A$. (2) For every $a \in A$, $\lambda_1(a) = \inL$ 
(and respectively $\lambda_1(a) = \outL$) 
materially implies $\lambda_2(a) = \inL$ (and 
respectively $\lambda_2(a) = \outL$). We may write 
$\lambda_1 \prec \lambda_2$ when 
$\lambda_1 \preceq \lambda_2$ but not 
$\lambda_2 \preceq \lambda_1$. 

Then, $\lambda \in \Lambda$ is said to be: 
a complete labelling of $(A, R) \in \mathcal{F}^{\Dung}$ iff all the following 
conditions hold for every $a \in A$ \cite{Caminada06}. 
\begin{enumerate} 
\item $\lambda \in \Lambda^A$. 
\item 
$\lambda(a) = \inL$ iff there exists no $a_x \in A$ 
such that $a_x$ attacks $a$ and that 
$\lambda(a_x) \not= \outL$. 
\item $\lambda(a) = \outL$ iff there exists 
some $a_x \in A$ such that $a_x$ attacks $a$ 
and that $\lambda(a_x) = \inL$. 
\end{enumerate} 
First of all, $(\Lambda^A, \preceq)$ is 
clearly a meet-semilattice (see \cite{Davey02} for all these 
notions around a lattice). 
Denote the set of all complete labellings of 
$(A, R)$ by $\Lambda_{(A, R)}^{com}$, it is well-known 
that 
$(\Lambda_{(A, R)}^{com}, \preceq)$ is also a meet-semilattice. 

A complete labelling $\lambda$ of $(A,R)\ 
(\in \mathcal{F}^{\Dung})$ is said to 
be also a preferred labelling of $(A, R)$ iff, 
for every complete labelling $\lambda_x$ of $(A, R)$, 
$\lambda \prec \lambda_x$ does not hold. A preferred 
labelling $\lambda$ of $(A, R)$ is also a stable labelling of $(A, R)$ 
iff, for every $a \in A$, $\lambda(a) \not= \undecL$ holds. 
Also, $\lambda \in \Lambda$ is called a grounded 
labelling of $(A, R)$ iff $\lambda$ 
is the meet of $\Lambda^{com}_{(A, R)}$ in 
$(\Lambda^A, \preceq)$.\footnote{We make 
it more general here in light of 
some more recent argumentation studies (including 
this work) in which a grounded 
labelling is not necessarily a complete labelling 
\cite{ArisakaSantini19,Bistarelli17}, although, in case 
of Dung argumentation, it is trivial 
that a grounded labelling is the meet 
of $\Lambda_{(A, R)}^{\text{com}}$ in 
$(\Lambda_{(A,R)}^{\text{com}}, \preceq)$.}

For any such labelling $\lambda$ of $(A, R)$, 
we say that $a \in A$ is: accepted iff  
$\lambda(a) = \inL$; rejected iff 
$\lambda(a) = \outL$; and undecided, otherwise.

We call the set of all complete/preferred/stable/grounded 
labellings of $(A, R)$  
complete/preferred/stable/grounded semantics of $(A, R)$.   

Let $a_1 \rightarrow a_2$ or $a_2 \leftarrow a_1$ 
be a graphical representation of $(a_1, a_2) \in R$. 
A small concrete example 
$a_1 \leftrightarrows a_2$ should suffice for 
highlighting the relation among the semantics. 
Assume 
that $[a_1: l_1, \ldots, a_n:l_n]^{\lambda}$  
for $a_1, \ldots, a_n \in A$ and $l_1, \ldots, l_n \in \mathcal{L}$ 
denotes some member of $\Lambda^{\{a_1, \ldots, a_n\}}$ 
	with $[a_1:l_1, \ldots,a_n:l_n]^{\lambda}(a_i) = l_i\ (1 \leq i \leq n)$,
and let $\lambda_1, \lambda_2, \lambda_3 \in \Lambda$
be $\lambda_1 \equiv [a_1:\inL,a_2:\outL]^{\lambda}$ (as shown below, to the left), 
$\lambda_2 \equiv [a_1:\outL,a_2:\inL]^{\lambda}$ (as shown below, at the centre), 
and $\lambda_3 \equiv [a_1:\undecL,a_2:\undecL]^{\lambda}$ (as shown, below, 
to the right). 
\begin{center} 
\begin{tikzcd}  
   \overset{\ensuremath{\inL}}{\ensuremath{a_1}} 
  \rar & \overset{\ensuremath{\outL}}{\ensuremath{a_2}} \lar & 
   \overset{\ensuremath{\outL}}{\ensuremath{a_1}} \rar & 
   \overset{\ensuremath{\inL}}{\ensuremath{a_2}} \lar &   
   \overset{\ensuremath{\undecL}}{\ensuremath{a_1}} \rar & 
   \overset{\ensuremath{\undecL}}{\ensuremath{a_2}} \lar
\end{tikzcd}  
\end{center}   
Then, complete, preferred, stable and grounded semantics 
of this argumentation are 
exactly $\{\lambda_1, \lambda_2, \lambda_3\}$, 
$\{\lambda_1, \lambda_2\}$, $\{\lambda_1, \lambda_2\}$
and $\{\lambda_3\}$.  

\section{Label-based Argumentation Semantics 
with May-Must Scales}\label{section_label_based} 
We present abstract argumentation 
with may-must scales, and 
characterise its labelling-based semantics 
in this section. In the remaining, 
for any tuple $T$ of $n$-components, 
we make the following a rule that  
$(T)^i$ for $1 \leq i \leq n$ refers to 
$T$'s $i$-th component. Since the two may-must scales 
(one for acceptance and one for rejection) define  
a nuance of acceptability of an argument, we 
call the pair a nuance tuple: 
\begin{definition}[Nuance tuple]  
    We define a nuance tuple to be 
   $(\pmb{X}_1, \pmb{X}_2)$ for some\linebreak
    $\pmb{X}_1, \pmb{X}_2 \in \mathbb{N} \times \mathbb{N}$.   
   We denote the class of all nuance tuples by 
   $\mathcal{Q}$. For any $Q \in \mathcal{Q}$, 
   we call $(Q)^1$ its may-must acceptance scale and 
   $(Q)^2$ its may-must rejection scale.  
\end{definition} 
For $Q \in \mathcal{Q}$, $Q$ is some $((n_1, n_2), (m_1, m_2))$ 
with $n_1, n_2, m_1, m_2 \in \mathbb{N}$. Therefore, 
 $(Q)^1$ is $(n_1, n_2)$, $(Q)^2$ is $(m_1, m_2)$. 
 Further, 
 $((Q)^1)^i$ is $n_i$ and $((Q)^2)^i$ is $m_i$ 
 for $i \in \{1,2\}$. 
\begin{definition}[May-must argumentation]\label{def_abstract_argumentation}
    We define a (finite) abstract argumentation with may-must 
	scales, i.e. may-must argumentation, to be a tuple $(A, R, f_Q)$ with: 
    $A \subseteq_{\text{fin}} \mathcal{A}$; 
	$R \in \mathcal{R}^A$; and 
    $f_Q: A \rightarrow \mathcal{Q}$ with 
    $((f_Q(a))^i)^1 \leq ((f_Q(a))^i)^2$ for every
    $a \in A$ and every $i \in \{1,2\}$. 
    
    We denote the class of all 
    (finite) abstract argumentations with
    may-must scales by $\mathcal{F}$, 
    and refer to its member by $F$ with or without a subscript. 
\end{definition} 

\noindent The role of a nuance tuple within 
$(A, R, f_Q) \in \mathcal{F}$ is as was described 
in Section \ref{section_introduction}.  
Each $a \in A$ comes with its own criteria for acceptance 
($(f_Q(a))^1$) and for rejection ($(f_q(a))^2$).  
The independent acceptance/rejection judgements 
are combined for a final decision on which acceptance  
status(es) are expected for $a$. 

\subsection{Independent acceptance/rejection judgement} 
Let us consider the acceptance judgement for $(A, R, f_Q) \in \mathcal{F}$. 
If $a \in A$ 
is such that $((f_Q(a))^1)^1 = 2$ 
and 
$((f_Q(a))^1)^2 = 3$, then 
$a$ can never be judged accepted 
unless there are at least 2 rejected arguments attacking $a$. Once there are at 
least 3 rejected arguments 
attacking $a$, then $a$ is judged to be, in the absence 
of the rejection judgement, certainly accepted.

Given the nature of attack, 
it is not very intuitive to permit   
the value of a may- condition to be strictly larger 
than that of a must- condition of a single may-must scale: 
an accepted attacking argument has a non-favourable 
effect on the acceptance of argument(s) it attacks; 
if, say, 2 arguments attacking $a$ need to be accepted 
in order that $a$ can be rejected, intuitively  
1 accepted argument attacking $a$ does not produce a 
strong enough non-favourable effect on $a$ to reject it; 
also into the other direction, if, say, 3 arguments attacking 
$a$ need to be accepted in order that $a$ must be rejected, 
intuitively 4 accepted arguments attacking $a$ still 
enforce rejection of $a$. 
It is for this reason that we are formally 
precluding the possibility 
in Definition \ref{def_abstract_argumentation}. 

Satisfactions of 
may- and must- conditions of the may-must scales 
are as defined in \mbox{Definition \ref{def_may_and_must}} 
below. In the remaining, for any $F \equiv (A, R, f_Q)\ (\in \mathcal{F})$,
any $a \in A$ and any $\lambda \in \Lambda$, we denote 
by $\pre^F(a)$ the set of all $a_x \in A$ with 
$(a_x, a) \in R$, 
by 
$\pre^F_{\lambda, \inL}(a)$ the set of all $a_x \in \pre^F(a)$ 
such that $\lambda$ is defined for $a_x$ and that $\lambda(a_x) = \inL$, 
and by $\pre^F_{\lambda, \outL}(a)$ the set of all $a_x \in 
\pre^F(a)$ such that $\lambda$ is defined for $a_x$ and that $\lambda(a_x) = 
\outL$.  
\begin{definition}[May- and must- satisfactions]\label{def_may_and_must} 
   For $F \equiv (A, R, f_Q)\ (\in \mathcal{F})$, 
    $\lambda \in \Lambda$, and 
	$a \in A$ for which  $\lambda$ is defined,  
	we say that $a$ satisfies:  
	\begin{itemize} 
	\item may-acceptance condition (resp. may-rejection condition) 
		under $\lambda$ in $\FFMMA$ iff\linebreak 
		$((f_Q(a))^1)^1 \leq |\pre^{\FFMMA}_{\lambda, \outL}(a)|$ 
		(resp. $((f_Q(a))^2)^1 \leq |\pre^{\FFMMA}_{\lambda, \inL}(a)|$).   
	\item must-acceptance condition (resp. must-rejection condition) 
	     under $\lambda$ in $\FFMMA$ iff \linebreak
		$((f_Q(a))^1)^2 \leq |\pre^{\FFMMA}_{\lambda, \outL}(a)|$ 
		(resp. $((f_Q(a))^2)^2 \leq |\pre^{\FFMMA}_{\lambda, \inL}(a)|$).   
	\item may$_s$-acceptance condition (resp. may$_s$-rejection 
		condition) under $\lambda$ in $\FFMMA$ iff \linebreak 
	      $((f_Q(a))^1)^1 \leq |\pre^{\FFMMA}_{\lambda, \outL}(a)| < 
	((f_Q(a))^1)^2$ 
	(resp. $((f_Q(a))^2)^1 \leq |\pre^{\FFMMA}_{\lambda, \inL}(a)| < 
	((f_Q(a))^2)^2$).   
\item not-acceptance condition (resp. not-rejection condition)  
	under $\lambda$ in $\FFMMA$ iff \linebreak 
	      $|\pre^{\FFMMA}_{\lambda, \outL}(a)| < 
	((f_Q(a))^1)^1$ 
	(resp. $|\pre^{\FFMMA}_{\lambda, \inL}(a)| < 
	((f_Q(a))^2)^1$).   
\end{itemize} 
  If obvious from the context, we may omit 
	``under $\lambda$ in $F$''. Moreover,  
  we may shorten ``$a$ satisfies may-acceptance condition'' 
  with ``$a$ satisfies may-a''; and similarly 
  for all the others. 
\end{definition} 
Clearly, the may-must conditions are monotonic over the increase 
in the number of rejected/accepted attacking arguments. 

\subsection{Combinations of acceptance/rejection judgements} 
As we described in 
Section \ref{section_introduction}, acceptance 
and rejection variations give rise to 
several combinations. Here, we cover all possible cases 
exhaustively and precisely for each $F \equiv (A, R, f_Q)\ (\in \mathcal{F})$, 
each $\lambda \in \Lambda$ 
and each $a \in A$ for which $\lambda$ is defined. 
\begin{itemize}   
	\item[(1)]{\makebox[2cm]{\textbf{must-must}:\hfill}  $a$ satisfies must-a and must-r 
		under $\lambda$.} 
	\item[(2)]{\makebox[2cm]{\textbf{must-may$_s$}:\hfill} 
     $a$ satisfies must-a (resp.
     must-r)  and  
		may$_s$-r (resp. may$_s$-a) under $\lambda$.} 
	\item[(3)]{\makebox[2cm]{\textbf{must-not}:\hfill} 
     $a$ satisfies must-a (resp.
     must-r), and 
     not-r 
		(resp. not-a) under $\lambda$.}
	\item[(4)]{\makebox[2cm]{\textbf{may$_s$-may$_s$}:\hfill}  $a$ satisfies 
		may$_s$-a and may$_s$-r under $\lambda$.} 
	\item[(5)]{\makebox[2cm]{\textbf{may$_s$-not}:\hfill} 
     $a$ satisfies may$_s$-a 
    (resp. may$_s$-r), 
   and not-r
		(resp. not-a) under $\lambda$.} 
	\item[(6)]{\makebox[2cm]{\textbf{not-not}:\hfill}  $a$ satisfies 
		not-a and not-r under $\lambda$.}
\end{itemize}
 
\noindent In the classic (i.e. non-intuitionistic) interpretation of modalities 
(see for example \cite{Garson13}), 
a necessary (resp. possible) proposition is true 
iff it is true 
in every (resp. some) possible world 
accessible from the current world, and a not possible proposition 
is false in every accessible possible world.  
To apply this interpretation to may-must argumentation, 
we take acceptance for the truth, rejection for 
the falsehood, must for the necessity and may for the possibility, 
to obtain 
for each $a \in A$ and each $\lambda \in \Lambda$ 
that $a$'s satisfaction of: 
\begin{itemize}
\item must-a (resp. must-r) of $a$ 
implies  $a$'s acceptance (resp. rejection) 
in every accessible possible 
world.
\item may$_s$-a (resp. may$_s$-r) of $a$ implies $a$'s acceptance 
(resp. rejection) in a non-empty 
subset of all accessible possible worlds and
$a$'s rejection (resp. acceptance) in 
		the other accessible possible worlds (if any). 
\item not-a (resp. not-r) 
is equivalent to must-r (resp. must-a).
\end{itemize} 
Note the use of ``may$_s$'' instead of ``may'' here. Once 
a may- condition is also a must- condition, it suffices 
to simply consider the must- condition. 

\indent With no prior
knowledge of the possible worlds, the accessibility relation 
and the current world, 
we infer the potential acceptance status(es) of $a$  
under $\lambda$ as follows: \\

For (1), since $a$ can either be accepted or rejected 
but not both simultaneously in any possible world, this case where 
both acceptance and rejection of $a$ are implied 
in every accessible possible world 
is logically inconsistent. Thus, 
only $\undecL$ is expected as the acceptability status of $a$ under $\lambda$, 
or, synonymously, $\lambda$ designates only $\undecL$ for $a$. 
In the rest, we uniformly make use of the latter expression. \\
 
For (2), in some accessible possible worlds, 
$a$'s acceptance 
and rejection are both implied, leading to 
inconsistency, while in the other accessible possible worlds, 
only $a$'s acceptance (resp.
rejection) is implied. Hence, it is clear 
that $\lambda$ designates any of $\inL$ and $\undecL$ (resp. 
$\outL$ and $\undecL$) for $a$.\\

For (3), it is the case that acceptance (resp. rejection) of $a$ 
is implied in every possible world.  $\lambda$ designates
only $\inL$ (resp. $\outL$) for $a$. \\

For (4), it is possible that only 
$a$'s acceptance is implied in 
some accessible 
possible worlds, only $a$'s rejection is implied 
in some other accessible possible worlds, 
and both $a$'s acceptance and rejection are implied
in the remaining accessible possible worlds. 
Thus, $\lambda$ designates any of $\inL$, $\outL$ and $\undecL$ 
for $a$. \\

For (5), it is analogous to (2). $\lambda$ designates 
either of $\inL$ and $\undecL$ (resp. $\outL$ and $\undecL$) for $a$. \\

For (6), we have logical 
inconsistency, and so $\lambda$ designates $\undecL$ for $a$.  \\
 
\begin{wrapfigure}[15]{r}{6.3cm}  
 \begin{tikzcd}[column sep=tiny,row sep=tiny,
	 /tikz/execute at end picture={
		 \draw (-1.35, -1.5) -- (-1.35, 1.2);
		 \draw (-2.7, 0.77) -- (2.6, 0.77);}]  
	 & \text{must}\text{-r} & \text{may}_s\text{-r} & \text{not}\text{-r} \\
	  \text{must}\text{-a} & \undecL & \inL ? & \inL \\ 
	 \text{may}_s\text{-a}	 &  \outL ? & \textsf{any} & \inL ? \\
	 \text{not-a} & \outL & \outL ? & \undecL
\end{tikzcd}   
\caption{Label designation table for any argument 
which satisfies a combination of 
may- must- conditions. \textsf{any} 
is any of $\inL, \outL, \undecL$, \textsf{in}$?$
is any of $\inL, \undecL$, and \textsf{out}$?$ 
is any of $\outL, \undecL$.}
\label{fig_table} 
\end{wrapfigure} 

\noindent Fig. \ref{fig_table} 
summarises the label designation
for any $a \in A$ which satisfies 
a given combination of may- must- conditions under a given $\lambda 
\in \Lambda$. 
The 
\textsf{any} 
entry in the table abbreviates 
either of $\inL$, $\outL$ and $\undecL$, the \textsf{in}$?$ entry 
either of $\inL$ and $\undecL$, and the \textsf{out}$?$ entry 
either of $\outL$ and $\undecL$. 
 
It should be noted that we obtain this result in general. 
If we restrict out attention to a particular set of 
possible worlds, a particular accessibility relation and 
a particular current world, the \textsf{any}, \textsf{in}?, 
and \textsf{out}? entries in the table may be more precise.  
We do not deal with the specific setting in this particular 
paper. 

We make the concept of label designation formal in the following definition. 
\begin{definition}[Label designation]\label{def_label_designation}
   For any $F \equiv 
   (A, R, f_Q)\ (\in \mathcal{F})$, any $a \in A$, and 
  any $\lambda \in \Lambda$, we say that    
  $\lambda$ designates $l \in \mathcal{L}$  
  for $a$ 
  iff  all the following conditions hold.  
\begin{enumerate} 
      \item $\lambda$ is defined for every 
	      member of $\pre^F(a)$. 
      \item If $l = \inL$, then  
            $a$ satisfies may-a
            but not must-r. 
      \item If $l = \outL$, then 
      $a$ satisfies may-r but 
      not must-a.
     \item If $l = \undecL$, then either of the following 
    holds.
     \begin{itemize} 
         \item $a$ satisfies must-a and 
      must-r. 
         \item $a$ satisfies at least either may$_s$-a 
       or may$_s$-r. 
        \item $a$ satisfies not-a
     and not-r. 
      \end{itemize} 
   \end{enumerate} 
\end{definition}   
As can be easily surmised from Fig. \ref{fig_table}, 
a labelling $\lambda \in \Lambda$ may designate more than one label 
for an argument: 
\begin{proposition}[Non-deterministic label 
designation]\label{prop_non_deterministic} 
   There exist $F \equiv 
   (A, R, f_Q)\linebreak (\in \mathcal{F})$, $a \in A$, 
  $\lambda \in \Lambda$, and $l_1, l_2 \in \mathcal{L}$
   such that 
  $\lambda$ designates $l_1$ and $l_2$ for $a$ with $l_1 \not= l_2$.
\end{proposition}   
 On a closer look, it follows that from every corner of 
 the table with only one label,  
 there will be increasingly more labels  
 as we minimally travel from the cell to the centre cell. 
 We have: 
\begin{proposition}[Label designation subsumption]\label{prop_label_designation_subsumption}       
	Let $x, y$ be a member of\linebreak $\{\text{must}, \text{may}_s, \text{not}\}$. 
	For any $F \equiv (A, R, f_Q)\ (\in \mathcal{F})$, 
	any $a \in A$ and any $\lambda_1, \lambda_2, \lambda_3 \in \Lambda^{A_x}$
	with $\pre^F(a) \subseteq A_x$, 
	if  $a$ satisfies: $x$-a and $y$-r
	under $\lambda_1$; 
	may$_s$-a and $y$-r under $\lambda_2$; 
	and $x$-a and may$_s$-r under $\lambda_3$, 
	and if $\lambda_1$ designates $l \in \mathcal{L}$ 
	for $a$, then 
	both $\lambda_2$ and $\lambda_3$ designate $l$ 
	for $a$.  
\end{proposition} 
\begin{proof} 
	See Fig. \ref{fig_table}. \hfill$\Box$ 
\end{proof}

Now, in general, if $\lambda$ designates $l \in \mathcal{L}$ for $a \in A$, 
$l$ may still not actually 
be $a$'s label under $\lambda$. To connect the two, 
we define the following. 
\begin{definition}[Proper label]\label{def_designated_label}
   For any $F \equiv 
   (A, R, f_Q)\ (\in \mathcal{F})$, any $a \in A$, and 
  any $\lambda \in \Lambda$, 
we say that $a$'s label is proper under $\lambda$ iff  
   all the following conditions hold. 
	(1) $\lambda$ is defined for $a$. 
	(2) 
	$\lambda$ designates $\lambda(l)$ for $a$. 
\end{definition}  
An earlier conference paper \cite{arisaka2020a} 
calls it a `designated label' instead 
of a `proper label' under $\lambda$. We thought better of 
the term. 

 An intuitive understanding 
of the significance of the properness 
is: if every argument's label is proper under  
$\lambda$, then $\lambda$ is a `good' 
labelling in the sense of 
every argument respecting the correspondences 
in Fig. \ref{fig_table}.  
\begin{example}[Labelling]\label{ex_labelling} 
To illustrate these definitions around 
labelling, let us consider the following 
acyclic argumentation graph  
with associated nuance tuples. 
We let $\underset{Q}{a}$ be a graphical 
  representation of an argument $a$ 
  with $f_Q(a) = Q$.
\begin{center} 
\begin{tikzcd}  
   \underset{((0, 0), (1,1))}{\ensuremath{a_1}} 
  \rar & 
  \underset{((0,1), (1,2))}{\ensuremath{a_2}} \rar & 
   \underset{((1,1),(1,1))}{\ensuremath{a_3}} & \underset{((1,1),(1,1) )}{\ensuremath{a_4}} \lar &   
   \underset{((0,0), (1, 1))}{\ensuremath{a_5}} \lar
\end{tikzcd}  
\end{center}  
\noindent Denote this argumentation (with the indicated 
nuance tuples) by $F$. There are 3 `good' labellings in  
$\Lambda^{\{a_1, \ldots, a_5\}}$
satisfying 
any one of the following label assignments. 
Let us call the label with the first 
(, second, third) 
label assignment $\lambda_1$ (, $\lambda_2$, $\lambda_3$).
\begin{center} 
\begin{tikzcd}  
   \overset{\inL}{\underset{((0, 0), (1,1))}{\ensuremath{a_1}}} 
  \rar & 
\overset{\outL}{\underset{((0,1), (1,2))}{\ensuremath{a_2}}} \rar & 
   \overset{\inL}{\underset{((1,1),(1,1))}{\ensuremath{a_3}}} & 
   \overset{\outL}{\underset{((1,1),(1,1) )}{\ensuremath{a_4}}} \lar &   
   \overset{\inL}{\underset{((0,0), (1, 1))}{\ensuremath{a_5}}} \lar
\end{tikzcd}   

\begin{tikzcd}  
   \overset{\inL}{\underset{((0, 0), (1,1))}{\ensuremath{a_1}}} 
  \rar & 
\overset{\inL}{\underset{((0,1), (1,2))}{\ensuremath{a_2}}} \rar & 
   \overset{\undecL}{\underset{((1,1),(1,1))}{\ensuremath{a_3}}} & 
   \overset{\outL}{\underset{((1,1),(1,1) )}{\ensuremath{a_4}}} \lar &   
   \overset{\inL}{\underset{((0,0), (1, 1))}{\ensuremath{a_5}}} \lar
\end{tikzcd}  

\begin{tikzcd}  
   \overset{\inL}{\underset{((0, 0), (1,1))}{\ensuremath{a_1}}} 
  \rar & 
\overset{\undecL}{\underset{((0,1), (1,2))}{\ensuremath{a_2}}} \rar & 
   \overset{\inL}{\underset{((1,1),(1,1))}{\ensuremath{a_3}}} & 
   \overset{\outL}{\underset{((1,1),(1,1) )}{\ensuremath{a_4}}} \lar &   
   \overset{\inL}{\underset{((0,0), (1, 1))}{\ensuremath{a_5}}} \lar
\end{tikzcd}  
\end{center}  
Of the 5 arguments, the labels of $a_1$, $a_5$ and 
$a_4$ are proper under some 
$\lambda \in \Lambda^{\{a_1, \ldots,a_5\}}$ 
iff $\lambda(a_1) = \lambda(a_5) = \inL$ 
and $\lambda(a_4) = \outL$ hold. 
To see that that 
is the case, let us firstly note
that $\pre^F(a_1) = \pre^F(a_5) = \emptyset$. 
Thus, designation of label(s) 
for both $a_1$ and $a_5$ 
are known with no dependency on other 
arguments.\footnote{This does not mean that 
there can be only one label 
to be designated for $a_x$: if $a_x$ satisfies may$_s$-a or 
may$_s$-r, $a_x$ is designated more than one label. 
We simply mean that no other arguments are required 
to know which may- must- conditions 
$a_x$ would satisfy.} 
It follows trivially from the associated 
nuance tuples that both $a_1$ 
and $a_5$ satisfy must-a
but not may-r. 
From \mbox{Fig. \ref{fig_table}}, then, 
$\inL$ is the only one label to be designated for 
these two arguments. Vacuously,  
if any $\lambda \in \Lambda^{\{a_1, \ldots, a_5\}}$  designates 
only $\inL$ for $a_1$ and $a_5$, 
it must deterministically 
hold that $\lambda(a_1) = \lambda(a_5) = \inL$, 
if the two arguments' labels are to be 
proper under $\lambda$. 
Now for $a_4$, assume $\inL$ label for $a_5$, it 
satisfies must-r
(because there is 1 accepted  
attacking argument) and not-a
(because 
there is 0 rejected attacking argument), 
which finds in \mbox{Fig. \ref{fig_table}} the corresponding 
label $\outL$ to be designated for $a_4$. 
This is deterministic 
provided $\inL$ label for $a_5$ is deterministic, which 
happens to be the case in this example.  

A more interesting case is of $a_2$. Assume 
$\inL$ label for $a_1$, then it satisfies may$_s$-a
 (because 
there is 0 rejected attacking argument) 
and may$_s$-r (because 
there is 1 accepted attacking argument), 
which finds $\textsf{any}$ in \mbox{Fig. \ref{fig_table}} 
indicating that any of the 3 labels is designated 
for $a_2$. 

Finally for $a_3$ for which neither may$_s$-a
nor  
may$_s$-r can be satisfied, 
every combination 
of acceptability statuses of $a_2$ and $a_4$ 
leads to at most one of the 3 labels to be designated for $a_3$. 
Consequently, $\lambda_1, \lambda_2$ and $\lambda_3$ 
are indeed the only 3 possible labellings
of $F$ such that 
every argument's label is proper under 
them.  \hfill$\clubsuit$ 
\end{example}  

\section{Semantics with Comparisons}\label{section_semantics}  
\subsection{Exact labellings and semantics} 
The `good' labellings are exact to the labelling conditions
on the arguments as imposed by their may-must scales; thus we call 
them exact labellings: 
\begin{definition}[Exact labellings]\label{def_ideal_labelling}
  For any $F \equiv (A, R, f_Q)\ (\in \mathcal{F})$, 
	we say that $\lambda \in \Lambda$ 
	is an exact labelling of $F$ iff 
	(1) $\lambda \in \Lambda^A$ $\andC$ (2) 
	 every $a \in A$'s label is proper under $\lambda$. 
\end{definition} 
The labelling conditions for Dung abstract argumentation 
(see Section \ref{section_technical_preliminaries})  
are such that any labelling that satisfies 
the acceptance and rejection conditions of every 
argument in $(A, R) \in \mathcal{F}^{\Dung}$ is a complete 
labelling of $(A, R)$. As such, any complete labelling  
of $(A, R)$ is exact. Exact labellings form the following semantics. 
\begin{definition}[Exact semantics] 
        For any $F \equiv (A, R, f_Q)\ (\in \mathcal{F})$,  
	we say that $\Lambda_x \subseteq \Lambda$ 
	is the exact semantics of $F$ iff 
	every $\lambda \in \Lambda_x$ is, but no 
	$\lambda_x \in (\Lambda^A \backslash \Lambda_x)$ is, 
	an exact labelling of $F$. We denote 
	it by $\Lambda^{exact}_F$. 
\end{definition} 

\noindent Unfortunately, an exact labelling may not actually exist in general, 
as the following example shows.  
\begin{example}[Non-termination of choosing 
	an exact labelling]
  \label{ex_undecidability_of} 
Consider 
\begin{tikzpicture}[baseline= (a).base] 
			\node[scale=.94] (a) at (0,0){
\begin{tikzcd} 
   \underset{((0, 0), (1,1))}{\ensuremath{a_1}} 
   \arrow[out=5,in=-5,loop,looseness=10]
\end{tikzcd} 
	};
\end{tikzpicture} 
	Then, \mbox{$[a_1:\inL]^\lambda$} (see Section 2 for the notation) designates 
$\undecL$ for $a_1$. 
	However, \mbox{$[a_1:\undecL]^{\lambda}$} designates $\inL$ for $a_1$.  
	Also, \mbox{$[a_1:\outL]^\lambda$} designates $\inL$ for 
$a_1$.  
\hfill$\clubsuit$ 
\end{example}    

\begin{theorem}[Absence of an exact labelling]\label{thm_nothing} 
  There exists some $(A, R, f_Q) \in \mathcal{F}$ 
  and some $a \in A$ 
such that, for every $\lambda \in \Lambda^A$,  
$\lambda$ designates $l \not= \lambda(a)$ for $a$. 
\end{theorem}   
\begin{proof} 
	See Example \ref{ex_undecidability_of}. \hfill$\Box$ 
\end{proof}
\subsection{Maximally proper labellings and 
semantics} 
As the result of Theorem \ref{thm_nothing}, 
generally with $(A, R, f_Q) \in \mathcal{F}$ with 
an arbitrary $f_Q$, we can only 
hope to 
obtain maximally proper labellings with which   
as many arguments as are feasible are proper while 
the remaining are labelled $\undecL$. 
\begin{definition}[Maximally proper labellings and semantics]\label{def_maximally_proper} 
	For any $F \equiv (A, R, f_Q)\linebreak (\in \mathcal{F})$, 
we say that 
$\lambda \in \Lambda$ is a pre-maximally proper labelling 
	of $F$ iff (1) $\lambda \in \Lambda^A$ $\andC$ 
		(2) for any $a \in A$, either 
			$a$'s label is proper under $\lambda$ 
			or  $\lambda(a) = \undecL$. 
 
	We say that a pre-maximally proper labelling $\lambda \in \Lambda$
	of $F$ is maximally proper 
iff,  for every pre-maximally proper labelling $\lambda_x \in \Lambda$ 
	of $F$ and every $a \in A$, it holds that 
		if $a$'s label is proper under 
			$\lambda_x$, then $a$'s label is proper 
			under $\lambda$.  

    We say that $\Lambda_x \subseteq \Lambda$ is 
	the maximally proper semantics of $F$ 
	iff every $\lambda \in \Lambda_x$ is, but no 
	$\lambda_x \in (\Lambda^A \backslash \Lambda_x)$ is, 
	a maximally proper labelling of $F$. 
	We denote the maximally proper semantics of $F$ 
	by $\Lambda^{maxi}_F$. 
\end{definition}  
It should be noted that the order on the set of 
all pre-maximally proper labellings is not $\preceq$. 
For example, If $\lambda_1 \equiv [a_1: \inL, a_2: \undecL, a_3: \outL]$ 
and $\lambda_2 \equiv [a_1: \outL, a_2: \inL, a_3: \inL]$ are 
both pre-maximally proper labellings of $F$, 
$\lambda_2$ is, but not $\lambda_1$ is, 
a maximally proper labelling of $F$.

For the relation to the exact semantics, we have: 
\begin{theorem}[Conservation and existence]\label{thm_conservation} 
    For any $F \equiv (A, R, f_Q)\ (\in \mathcal{F})$,  
	it holds that $\Lambda_F^{exact} \subseteq \Lambda_F^{maxi}$.  
	If $\Lambda^{exact}_F \not= \emptyset$, then 
	it also holds that $\Lambda^{maxi}_F \subseteq \Lambda^{exact}_F$. 
	Moreover, if $A \not= \emptyset$, then 
	$\Lambda_F^{maxi} \not= \emptyset$. 
\end{theorem} 
\begin{proof} 
	Straightforward if we have $\Lambda_F^{maxi}, 
	\Lambda^{exact}_F \subseteq \Lambda^A$, which holds 
	to be the case by the definitions of a labelling 
	exact or maximally proper. 
	\hfill$\Box$ 
\end{proof} 
\noindent For the example 
in Example \ref{ex_undecidability_of}, 
we have one maximally proper labelling\linebreak
$[a_1:\undecL]^\lambda$.  

We can also 
obtain the maximally proper semantics of $F \in \mathcal{F}$  
in more step-by-step a manner, 
since the dependency of arguments' acceptability 
statuses is, as with \cite{Caminada06,Brewka13}, 
strictly from source argument(s)' to their target argument's. 
Without loss of generality, we may thus consider 
any maximality per strongly connected 
component (see below), from ones that 
depend on 
a fewer number of other strongly connected 
components to those 
with a larger number of them to depend on. 
 
Let us first recall 
the definition of a strongly connected component. 

\begin{definition}[SCC and SCC depth]\label{def_scc_and}    
	   For any $F \equiv (A, R, f_Q)\ (\in \mathcal{F})$, 
	we say that $(A_1, R_1) \in \mathcal{F}^{\Dung}$ with $A_1 \subseteq A$ 
	and $R_1 \in \mathcal{R}^{A_1}$  
is a 
	strongly connected component (SCC) of $F$ iff, for every 
$a_x \in A$ and every $a_y \in A_1$, we have: 
  $\{(a_x, a_y), (a_y, a_x)\} \subseteq R^*$ iff 
$a_x \in A_1$. Here, $R^*$ is the reflexive and 
transitive closure of $R$.  

Let $\Delta: \mathcal{F} \times \mathcal{A} 
    \rightarrow 2^{\mathcal{A}}$ be such that,   
   for any $F \equiv (A, R, f_Q)\ (\in \mathcal{F})$  
    and any $a \in A$, 
     $\Delta(F, a)$ is the set of all arguments 
     of a SCC of $F$ 
     that includes $a$, and let $\delta: \mathcal{F} \times \mathcal{A} \rightarrow
   \mathbb{N}$ be such that, 
    for any $F \equiv (A, R, f_Q)\ (\in \mathcal{F})$  
    and any $a \in A$, 
   $\delta(F, a)$ is:  
   \begin{itemize} 
      \item  $0$ if 
     there is no $a_x \in \Delta(F, a)$   
     and $a_y \in (A \backslash \Delta(F, a))$ 
     such that $(a_y, a_x) \in R$. 
      \item $1+ \max_{a_z \in A'}\delta(F, a_z)$
       with: $A' = \{a_w \in (A \backslash 
    \Delta(F, a)) \mid \exists a_u \in \Delta(F, a).(a_w,
     a_u) \in R\}$, otherwise.
   \end{itemize} 
   
    \noindent We say that 
    $\Delta(F, a)$ has SCC-depth $n$ iff 
    $\delta(F, a) = n$. 
\end{definition} 

%
\begin{example}[SCC and SCC depth]  
If we have $F$: 
	\begin{tikzpicture}[baseline= (a).base] 
			\node[scale=.94] (a) at (0,0){
\begin{tikzcd} 
   \underset{Q_1}{\ensuremath{a_1}} 
   \rar & \underset{Q_2}{a_2} \lar \rar & 
   \underset{Q_3}{a_3} \rar & 
   \underset{Q_4}{a_4} 
\end{tikzcd} 
};
	\end{tikzpicture} 
we have 
3 SCCs of $F$:  
\begin{tikzpicture}[baseline= (a).base] 
			\node[scale=.94] (a) at (0,0){
\begin{tikzcd} 
   \underset{}{\ensuremath{a_1}} 
   \rar & \underset{}{a_2} \lar\ ; 
\end{tikzcd} 
};
	\end{tikzpicture}  
	\begin{tikzpicture}[baseline= (a).base] 
			\node[scale=.94] (a) at (0,0){
\begin{tikzcd} 
   \underset{}{a_3}\ ; 
\end{tikzcd} 
};
	\end{tikzpicture} 
and 
\begin{tikzpicture}[baseline= (a).base] 
			\node[scale=.94] (a) at (0,0){
\begin{tikzcd} 
   \underset{}{a_4} . 
\end{tikzcd} 
};
	\end{tikzpicture} 
	Therefore,  $\delta(F, a_1) = \delta(F, a_2) = 0$, 
	$\delta(F, a_3) = 1$ and $\delta(F, a_4) = 2$. 
\hfill$\clubsuit$ 
\end{example} 

\noindent We also assume the following definitions for labelling manipulation. 
\begin{definition}[Labelling restriction] 
   We define $\downarrow: \Lambda \times 2^{\mathcal{A}} \rightarrow 
	  \Lambda$ to be such that, for any $A_1 \in 2^{\mathcal{A}}$ 
	  and for any $\lambda \in \Lambda$, 
	  $\downarrow\!(\lambda, A_1)$, 
	  alternatively $\lambda_{\downarrow A_1}$, satisfies all the following.    
	  \begin{enumerate}  
  		\item   If $\lambda \in \Lambda^{A_x}$ for some $A_x \subseteq \mathcal{A}$, then $\lambda_{\downarrow A_1} \in \Lambda^{A_x \cap A_1}$. 
		\item   For every $a_x \in A_x \cap A_1$, 
			  it holds that $\lambda(a_x) = 
			  \lambda_{\downarrow A_1}(a_x)$.  
	  \end{enumerate} 
\end{definition} 
\begin{definition}[Labelling composition] 
    We define $\oplus: \Lambda \times \Lambda \rightarrow 
	  \Lambda$ to be such that, for any $\lambda_1 \in \Lambda^{A_1}$ 
	  and any $\lambda_2 \in \Lambda^{A_2}$, 
          $\oplus(\lambda_1, \lambda_2)$ is such that 
	  all the following conditions hold. 
	  \begin{enumerate} 
		  \item $\oplus(\lambda_1, \lambda_2) \in \Lambda^{(A_1 \cup 
			  A_2) \backslash (A_1 \cap A_2)}$. 
		  \item For every $a \in (A_1 \cup A_2) \backslash 
			  (A_1 \cap A_2)$,  
			  $\oplus(\lambda_1, \lambda_2)(a)$ is: $\lambda_1(a)$ 
			  if $\lambda_1$ is defined for $a$; 
			  $\lambda_2(a)$, otherwise. 
	  \end{enumerate}
\end{definition}
\begin{example}[Labelling manipulation] 
	Assume $(\{a_1, a_2\}, R, f_Q) \in \mathcal{F}$. 
	Assume 3 labellings $\lambda_1 \equiv [a_1:l_x]^{\lambda}, 
	\lambda_2 \equiv [a_2:l_y]^{\lambda}$, 
	and $\lambda_3 \equiv [a_1:l_1, a_2:l_2]^{\lambda}$.  
	Then for any $i \in \{1,2\}$, ${\lambda_3}_{\downarrow \{a_i\}}$ is 
	a member of $\Lambda^{\{a_i\}}$ and  
	${\lambda_3}_{\downarrow \{a_i\}}(a_i) = l_i$.  
	For the composition, $\oplus(\lambda_1, \lambda_2)(a_1) = l_x$ 
	and $\oplus(\lambda_1, \lambda_2)(a_2) = l_y$.  
	\hfill$\clubsuit$ 
\end{example} 
As the final preparation, we introduce the following notion. 
\begin{definition}[Designation conservative sub-argumentation]\label{def_labelling_conservative}  
	We define $\Downarrow: \mathcal{F} \times 2^{\mathcal{A}} \times 
	\Lambda \rightarrow 
	\mathcal{F}$ to be such that, 
        for any $F \equiv (A, R, f_Q)\ (\in \mathcal{F})$, 
	any $A_1 \subseteq A$ and any $\lambda \in \Lambda$, 
	\mbox{$\Downarrow((A, R, f_Q), A_1, \lambda)$}, 
	alternatively \mbox{$(A, R, f_Q)_{\Downarrow A_1, \lambda}$},  
	is some $(A', R', f_Q')$ satisfying all the following. 
	 \begin{enumerate} 
		 \item $A' = A \cap A_1\ (= A_1)$.  
		\item $R' = R \cap (A' \times A')$.   
		\item $f_Q'$ is a function $(A \cap A_1) \rightarrow \mathcal{Q}$  
			such that, for every $a_x \in (A \cap A_1)$ 
			 and every $\lambda_x \in \Lambda^{A \cap A_1}$, 
			 $l \in \mathcal{L}$ 
			 is designated for $a_x$ under $\lambda_x$ 
			 in $(A',R',f_Q')$ 
			 iff 
			 $l$ is designated for $a_x$ under 
			 $\oplus(\lambda_{\downarrow 
			 A \backslash A_1}, \lambda_x)$ in $F$. 
	 \end{enumerate}  
	 We say that $(A, R, f_Q)_{\Downarrow A_1, \lambda}$ 
	 is a designation conservative sub-argumentation of 
	 $(A, R, f_Q)$ with respect to $A_1$ and $\lambda$.
\end{definition} 
The point of Definition \ref{def_labelling_conservative}  
is, if we can find a designation conservative sub-argumentation 
for every SCC of $F \in \mathcal{F}$, then we can show derivation 
of any maximally proper labelling of $F$ as a sequential composition 
of maximally proper labellings for the designation conservative 
SCCs. However, note that if we for instance have: 
\begin{center} 
	\begin{tikzcd}[/tikz/execute at end picture={ 
    \node (largeA) [draw, rectangle, 
		minimum width=1cm,fit=(a2),opacity=0,label={[label distance=0.4cm]0:$\cdots$}]{};}] 
	& \underset{((n_1, n_2), (m_1, m_2))}{\ensuremath{a}}\\ 
		\overset{}{\underset{Q_1}{a_1}} \arrow[ur] &  
		|[alias=a2]| \overset{}{\underset{Q_2}{a_2}} \arrow[u] & 
		\overset{}{\underset{Q_n}{a_n}} \arrow[ul] 
\end{tikzcd}   
\end{center} 
with $n$ arguments attacking $a$, and if $\lambda \in \Lambda$ is 
such that $|\pre^F_{\lambda,\outL}(a)| = i$ and that 
$|\pre^F_{\lambda,\inL}(a)| = j$, 
then  
\begin{tikzpicture}[baseline= (a).base] 
			\node[scale=.94] (a) at (0,0){
\begin{tikzcd}
	\underset{((\max(0, n_1 - i), \max(0, n_2 - i)), (\max(0, m_1 - j), 
	\max(0, m_2 - j)))}{\ensuremath{a}} 
\end{tikzcd}   
};
\end{tikzpicture} 
ensures that $\lambda$ designates the same label(s) for $a$ as before 
the transformation. Thus, we have: 
\begin{proposition}[Existence of a designation conservative 
	sub-argumentation]  
 For any $F \equiv (A, R, f_Q)\ (\in \mathcal{F})$, 
	any $A_1 \subseteq A$ and any $\lambda \in \Lambda$, 
        there exists some designation conservative sub-argumentation
	of $F$ with respect to $A_1$ and $\lambda$. 
\end{proposition} 
\begin{proof}  
     For every $a \in A$, let $((n_1^a, n_2^a), (m_1^a, m_2^a))$  
	 denote $f_Q(a)$. 
	Let $A_x$ denote 
	$\{a_x \in (A \cap A_1) \mid \exists a_y \in (A \backslash A_1).(a_y, a_x) 
	\in R\}$. Let $A_y$ 
	denote $\pre^F(a) \cap (A \backslash A_1)$. 
	Let $A_{\inL, y}$ denote $\{a_y \in A_y \mid \lambda(a_y) = \inL\}$, 
	and let $A_{\outL, y}$ denote $\{a_y \in A_y \mid \lambda(a_y) = \outL\}$. 
		Let $f'_Q$ be such that, 
	for any $a_x \in A_x$, 
	$f'_Q(a_x) = ((\max(0, n_1^a - |A_{\outL, y}|), 
	\max(0, n_2^a - |A_{\outL, y}|)), 
	(\max(0, m_1^a - |A_{\inL, y}|), \max(0, m_2^a - |A_{\inL, y}|)))$.  
	Then $(A \cap A_1, R \cap ((A \cap A_1) \times 
	(A \cap A_1)), f'_Q)$ is a designation conservative 
	sub-argumentation of $F$ with respect to $A_1$ and $\lambda$.  
	\hfill$\Box$
\end{proof} 

\noindent We obtain our result that we can obtain  
any maximally proper labelling restricted to 
SCCs of any SCC-depth incrementally: 
\begin{proposition}[Bottom-up identification]    
	Let $\bundle: \mathcal{F} \times \mathbb{N}$ 
	be such that, for any $(A, R, f_Q) \in \mathcal{F}$, 
	$\bundle(F, n) = \{a \in A \mid  \delta(F, a) = n\}$.  
	Let $\Theta$ be the class of all functions 
	$\vartheta: \mathcal{F} \times \mathbb{N} \rightarrow  
	2^{\Lambda}$, each of which is such that, for any 
	$F \equiv (A, R, f_Q) \in \mathcal{F}$ and 
	any $n \in \mathbb{N}$, 
	$\vartheta(F, n)$ satisfies all the following. 
	\begin{enumerate} 
	        \item $\vartheta(F, 0) = \Lambda^{maxi}_{F_{\Downarrow 
			\bundle(F, 0), \lambda}}$ for any $\lambda \in \Lambda$. 
		\item For any $n$ with $0 < n$, there exist 
			some $\lambda_0 \in \vartheta(F, 0)$, 
			$\ldots$, $\lambda_{n-1} \in \vartheta(F, n-1)$ 
			such that, for any 
			$1 \leq i \leq n$, 
			$\vartheta(F, i) = \Lambda^{maxi}_{F_{\Downarrow
			\bundle(F, i), \lambda_{i-1}}}$. 
	\end{enumerate}
	Then, let $n$ be at most as large as the maximum SCC depth 
	of all SCCs of $F$, 
	it holds that 
		   $\{\lambda_{\downarrow \bundle(F, n)} \in 
		   \Lambda^{\bundle(F, n)} \mid 
			\lambda \in \Lambda^{maxi}_F\}  
			= \{\lambda \in \Lambda \mid \exists \vartheta \in 
			\Theta.\lambda \in \vartheta(F, n)\}$.  
\end{proposition} 
\begin{proof}  
	By induction on the value of $n$.  
	If $n = 0$, then no $a_x \in (A \backslash \bundle(F, 0))$ 
	attacks a member of $\bundle(F, 0)$. Thus,  
	$F_{\Downarrow \bundle(F, 0), \lambda} 
	\equiv (\bundle(F, 0), R \cap (\bundle(F,0) \times
	\bundle(F,0)), f_Q')$ with 
	$f_Q'(a_x) = f_Q(a_x)$ for every $a_x \in \bundle(F, 0)$  
	is designation conservative. The choice of $\lambda$ 
	makes no difference. 
        It is obvious that 
	$\{\lambda_{\downarrow \bundle(F, 0)} \in \Lambda^{\bundle(F, 0)} 
	\mid \lambda \in \Lambda^{maxi}_F\} = 
	\{\lambda \in \Lambda \mid \exists \vartheta \in \Theta.\lambda 
	\in \vartheta(F, 0)\} = 
	\{\lambda \in \Lambda \mid \exists \vartheta \in \Theta.\lambda
	\in \Lambda^{maxi}_{(\bundle(F, 0), R \cap 
	(\bundle(F, 0) \times \bundle(F,0)), f_Q')}\}$.  

	For any $n$ with $0 < n \leq (\text{maximum SCC-depth of } F)$, assume that the result  
	holds up to $n-1$. Then for any $\lambda \in \Lambda^{maxi}_F$,  
	there is some $\vartheta \in \Theta$ 
	and some 
	$\lambda_1 \in \vartheta(F, n-1)$ 
	such that $\lambda_1(a_x) = \lambda(a_x)$  
	for every $a_x \in \bundle(F, n-1)$.  
	By definition, if $a_y \in (A \backslash \bundle(F, n))$ 
	attacks a member of $\bundle(F, n)$, then 
	it must be that $a_y \in \bundle(F,  n-1)$. Thus, 
	we can derive a designation conservative sub-argumentation of $F$ 
	with respect to $\bundle(F, n)$ and $\lambda_1$. 
	The rest is straightforward. \hfill$\Box$   
\end{proof}

\subsection{Maximally proper complete/preferred/stable/grounded semantics} 
Based on the maximally proper semantics with a guaranteed 
existence, we may define its variants in the manner 
similar to classic variations (Cf. Section 2) as follows. 
\begin{definition}[Maxi.x labellings]\label{def_labelling} 
  For any $F \equiv (A, R, f_Q)\ (\in \mathcal{F})$ 
   and 
  any $\lambda \in \Lambda$,
   we say that $\lambda$ is:
   \begin{itemize}
     \item  a maxi.complete labelling 
   of $F$ 
		   iff $\lambda \in \Lambda^{maxi}_F$. 
      We denote the set of all maxi.complete labellings 
   of $F$ 
		    by $\Lambda^{maxi.com}_F\ (= \Lambda^{maxi}_F)$. 
    
\item a maxi.preferred labelling of $F$ 
	 iff $\lambda \in \Lambda^{maxi.com}_F$ $\andC$ 
	 for every
		   $\lambda' \in \Lambda^{maxi.com}_F$, 
   it does not hold that $\lambda \prec \lambda'$.  
    
    \item a maxi.stable labelling of $F$ 
	    iff $\lambda \in \Lambda^{maxi.com}_F$ $\andC$ 
      for every
    $a \in A$, $\lambda(a) \not= \undecL$ holds. 
    
    \item a maxi.grounded labelling of $F$ iff 
      $\lambda$ is the meet of $\Lambda^{maxi.com}_{F}$
     in $(\Lambda^A, \preceq)$. 
    \end{itemize}  
\end{definition}

\begin{definition}[Maxi.x semantics]\label{def_acceptability_semantics}
For any $F \equiv (A, R, f_Q)\ (\in \mathcal{F})$, we say that 
   $\Lambda' \subseteq \Lambda$ is
   the maxi.complete (, maxi.preferred, maxi.stable, maxi.grounded)    
   semantics of $F$ iff every maxi.complete (, maxi.preferred, 
   maxi.stable, maxi.grounded) labelling of $F$, but no other, 
   is in $\Lambda'$. 
\end{definition}

\begin{example}[Maxi.x semantics]\label{ex_semantics} 
  Consider the example in Example 
  \ref{ex_labelling}, with the 3 
  maximally-designating labellings 
  $\lambda_1$ (the first), $\lambda_2$ 
  (the second), and $\lambda_3$ (the last).  
  \begin{center} 
\begin{tikzcd}  
   \overset{\inL}{\underset{((0, 0), (1,1))}{\ensuremath{a_1}}} 
  \rar & 
\overset{\outL}{\underset{((0,1), (1,2))}{\ensuremath{a_2}}} \rar & 
   \overset{\inL}{\underset{((1,1),(1,1))}{\ensuremath{a_3}}} & 
   \overset{\outL}{\underset{((1,1),(1,1) )}{\ensuremath{a_4}}} \lar &   
   \overset{\inL}{\underset{((0,0), (1, 1))}{\ensuremath{a_5}}} \lar
\end{tikzcd}   

\begin{tikzcd}  
   \overset{\inL}{\underset{((0, 0), (1,1))}{\ensuremath{a_1}}} 
  \rar & 
\overset{\inL}{\underset{((0,1), (1,2))}{\ensuremath{a_2}}} \rar & 
   \overset{\undecL}{\underset{((1,1),(1,1))}{\ensuremath{a_3}}} & 
   \overset{\outL}{\underset{((1,1),(1,1) )}{\ensuremath{a_4}}} \lar &   
   \overset{\inL}{\underset{((0,0), (1, 1))}{\ensuremath{a_5}}} \lar
\end{tikzcd}  

\begin{tikzcd}  
   \overset{\inL}{\underset{((0, 0), (1,1))}{\ensuremath{a_1}}} 
  \rar & 
\overset{\undecL}{\underset{((0,1), (1,2))}{\ensuremath{a_2}}} \rar & 
   \overset{\inL}{\underset{((1,1),(1,1))}{\ensuremath{a_3}}} & 
   \overset{\outL}{\underset{((1,1),(1,1) )}{\ensuremath{a_4}}} \lar &   
   \overset{\inL}{\underset{((0,0), (1, 1))}{\ensuremath{a_5}}} \lar
\end{tikzcd}  
\end{center}  
 
$\lambda_1, \lambda_2$ and $\lambda_3$ 
are the only 3 possible maxi.complete labellings
of $F$, i.e. $\{\lambda_1, \lambda_2, \lambda_3\}$ 
is its maxi.complete semantics. 

The relation among $\lambda_1, \lambda_2$ and $\lambda_3$ 
is such that $\lambda_3 \prec \lambda_1$, 
$\lambda_3 \not\preceq \lambda_2$, $\lambda_2 \not\preceq
\lambda_3$, \linebreak $\lambda_1 \not\preceq \lambda_2$, 
and $\lambda_2 \not\preceq \lambda_1$ all hold. Hence, 
$\{\lambda_1, \lambda_2\}$ is the maxi.preferred semantics 
of $F$, a subset of the maxi.complete semantics. Moreover, 
$\{\lambda_1\}$ is the maxi.stable semantics of $F$. 
On the other hand,
none of $\lambda_1, \lambda_2, \lambda_3$ 
are the least in $(\{\lambda_1, \lambda_2, \lambda_3\}, \preceq)$, and thus 
they cannot be a member of the maxi.grounded semantics; 
instead, it is $\{\lambda_4\}$ with: 
$\lambda_4(a_1) = \lambda_4(a_5) = \inL$; 
$\lambda_4(a_2) = \lambda_4(a_3) = \undecL$; 
and $\lambda_4(a_4) = \outL$. Clearly, $\lambda_4$ 
is the meet of $\{\lambda_1, \lambda_2, \lambda_3\}$ 
in $(\Lambda^{\{a_1, \ldots, a_5\}}, \preceq)$.  
\hfill$\clubsuit$ 
\end{example}  

\noindent The relation among the semantics below follows 
from Definition \ref{def_labelling} immediately, 
and is almost as expected. 
\begin{theorem}[Subsumption]\label{thm_subsumption}{\ }\\
  All the following hold for any $F 
   \in \mathcal{F}$.
   \begin{enumerate}  
     \item The maxi.complete, the maxi.preferred, and the maxi.grounded 
semantics of $F$ exist. 
     \item The maxi.preferred semantics of $F$ 
      is a subset of the maxi.complete semantics of $F$. 
     \item If the maxi.stable semantics of $F$ exists, 
	     then it consists of all (but nothing else) members $\lambda$ of 
        the maxi.preferred semantics of $F$ such that, 
        for every $a \in A$, $\lambda(a) \not= \undecL$
        holds.  
\item It is not necessary that 
  the maxi.grounded semantics be a subset of 
  the maxi.complete semantics. 
    \end{enumerate} 
   \end{theorem} 
\begin{proof} 
	Mostly straightforward. See Example \ref{ex_semantics} 
	for the 4th claim. \hfill$\Box$ 
\end{proof}
\noindent There is an easy connection to 
Dung abstract argumentation 
labelling 
(see Section \ref{section_technical_preliminaries}). 
\begin{theorem}[Correspondences to acceptability 
semantics in Dung
	argumentation]\label{thm_correspondence_acceptability}{\ }\\ 
   For any $F \equiv (A, R, f_Q)\
   (\in \mathcal{F})$, if $(f_Q(a))^1 = (|\pre^F(a)|,|\pre^F(a)|)$ and 
    $(f_Q(a))^2 = (1, 1)$ for every $a \in A$, then:  
     $\Lambda_x \subseteq \Lambda$ is 
   the maxi.complete (, maxi.preferred, 
maxi.stable, and resp. maxi.grounded) semantics 
	of $F$ iff $\Lambda_x$ is complete (, preferred, 
	stable, and resp. grounded) semantics of $(A, R)$. 
	
	Moreover, the maxi.complete semantics of $F$ is the exact 
	semantics of $(A, R)$. 
\end{theorem}    
\begin{proof}   
	For the first part of the claim, 
	we show for the correspondence for maxi.complete and complete 
	semantics, from which the others follow straightforwardly, 
	By definition, it holds that $\Lambda^{maxi.com}_{F}, 
	\Lambda^{com}_{(A, R)} \subseteq \Lambda^A$. 

	Now, assume $a \in A$. 
	Firstly, $\lambda$ designates $\inL$ for $a$ if, for 
	every $a_x \in \pre^F(a)$, $\lambda(a_x) = \outL$.  
	If there is some $a_x \in \pre^F(a)$ with $\lambda(a_x) \not= \outL$, 
	then it is not the case that $\lambda$ designates $\inL$ for $a$. 
        Secondly, $\lambda$ designates $\outL$ for $a$ if 
	there is some $a_x \in \pre^F(a)$ with $\lambda(a_x) = \inL$. 
	If there is no $a_x \in \pre^F(a)$ with $\lambda(a_x) = \inL$, 
	then it is not the case that $\lambda$ designates $\outL$ for $a$. 
         
	These establish $\Lambda^{maxi.com}_F = \Lambda^{com}_{(A, R)} 
	 = \Lambda^{exact}_F$, 
	as is also required in the second part of the claim.  \hfill$\Box$ 
\end{proof} 
This result should underscore an advantage of the level of 
abstractness 
of may-must argumentation, in that it is very 
easy to determine nuance tuples globally 
(and also locally) 
with just 4 specific natural numbers or expressions 
that are evaluated into natural numbers.  
For example, we can specify  
the requirement for: possible acceptance of 
an argument $a$ to be rejection of 80\% 
of attacking arguments;
acceptance of $a$ to be rejection of 90\%
of attacking arguments; 
possible rejection of $a$ to be 
acceptance of at least 1 but otherwise 40\% of attacking arguments; 
and rejection of $a$ to be 
acceptance of at least 1 but otherwise 50\% of 
attacking arguments, all rounded up 
to the nearest natural numbers. 
We have: 
$F \equiv (A, R, f_Q)\ (\in \mathcal{F})$  
with: $(f_Q(a))^1 = (\lceil 0.8*|\pre^F(a)| \rceil, \lceil 0.9*|\pre^F(a)| \rceil)$ and 
$(f_Q(a))^2 = (
  \max(1, \lceil 0.4*|\pre^F(a)| \rceil), 
\max(1, \lceil 0.5*|\pre^F(a)| \rceil))$ for every $a \in A$, 
to satisfy it.

\subsection{{\ADF} semantics}  
In this subsection we show what relations hold 
between exact and maximally proper semantics and 
{\ADF} semantics (with 3 values \cite{Brewka13}, for its closest 
connection to $\mathcal{F}$). We also define {\ADF} semantics 
for $\mathcal{F}$. 

Let us first state 
the formal definition of an {\ADF} tuple with its notations kept 
consistent with those used in this paper. 
Let $\mathcal{F}^{\ADF}$ be the class of all tuples 
$F^{\ADF} \equiv (A, R, C)$ with: 
$A \subseteq_{\text{fin}} \mathcal{A}$; 
$R \in \mathcal{R}^A$; and 
$C = \bigcup_{a \in A}\{C_a\}$ where each $C_a$ 
is a function: $\Lambda^{\pre^{F^{\ADF}}(a)} \rightarrow 
\mathcal{L}$. Here, $\pre^{(A, R, C)}(a)$ for $a \in A$ 
denotes $\{a_x \in A \mid (a_x, a) \in R\}$. 
We say that $F^{\ADF} \in \mathcal{F}^{\ADF}$ is an {\ADF} tuple, 
which we may refer to with  a subscript. 

To ease the juxtaposition with $\mathcal{F}$, we talk 
of label designation for $\mathcal{F}^{\ADF}$ as well.  
For any $\FFADF \equiv (A, R, C) \ (\in \FADF)$, any $a \in A$ 
and any $\lambda \in \Lambda$,  
we say that $\lambda$ designates $l \in \mathcal{L}$ for $a \in A$ iff (1)   
$\lambda$ is defined at least for each member of 
$\pre^{\FFADF}(a)$, and 
	  (2) 
	$C_a(\lambda_{\downarrow \pre^{F^{\ADF}}(a)})(a) = l$. 
 We say that $a \in A$'s label is proper under $\lambda \in \Lambda$ 
 in $\FFADF$
iff (1) $\lambda$ is defined for $a$, and (2) 
$\lambda$ designates 
$\lambda(a)$ for $a$. 

Also for the purpose of easing the juxtaposition, 
we define an exact labelling  
of \FADF: if every $a \in A$'s label is proper 
under $\lambda \in \Lambda^A$, then we say $\lambda$ is an 
{\it exact labelling} 
of $\FFADF$. 


\begin{example}[Illustration of $\FFADF$ exact labellings]\label{ex_illustration_of}  
In the following $\FFADF$ 
\begin{tikzcd}[
  column sep=small, row sep=small,inner sep=0pt]  
  \underset{C_p}{a_p}  
	\arrow[r] & 
	\underset{C_q}{a_q} \arrow[l,shift left]  
\end{tikzcd}  
with associated conditions $C_p$ and $C_q$, 
	assume $C_p({[a_q:\inL]^{\lambda}}) = \undecL$ and \linebreak
	$C_p([a_q:l]^{\lambda}) = \inL$ for $l \in \{\undecL, \outL\}$. 
	Assume also $C_q([a_p:\inL]^{\lambda}) = \undecL$ 
	and $C_q([a_p:l]^{\lambda}) = \inL$ for $l \in \{\undecL,\outL\}$. 
	Then there are two exact labellings of $\FFADF$, namely 
	${[a_p:\inL,a_q:\undecL]^{\lambda}}$ and 
	${[a_p:\undecL,a_q:\inL]^{\lambda}}$.  \hfill$\clubsuit$ 
\end{example}

\noindent Instead of the exact semantics as 
the set of all exact labellings, however,  
{\ADF} semantics essentially pre-interpret 
the labelling conditions $C$.  
We introduce the semantics below, and, in so doing, 
also adapt them to \FMMA. 

\subsubsection{{\ADF} semantics for \FMMA.}  

Let $\maxi: \mathcal{A} \times  
\Lambda \rightarrow 2^{\Lambda}$, 
alternatively $\maxi_{\mathcal{A}}: 
\Lambda \rightarrow 2^{\Lambda}$, 
be such that, 
for any $F \equiv (A, R, X)\ (\in \FADF \cup \FMMA)$, 
any $A \subseteq_{\textsf{fin}} 
\mathcal{A}$ and any 
$\lambda \in \Lambda^A$,  \\\\
\indent $\maxi_A(\lambda) = 
\{\lambda_x \in \Lambda^A \ | \ 
  \lambda \preceq \lambda_x \text{ and } 
\lambda_x \text{ is maximal in } 
(\Lambda^A, \preceq)\}$.\\

\noindent Every member $\lambda_x$ of $\maxi_A(\lambda)$ 
is such that $\lambda_x(a) \in \{\inL, \outL\}$ 
for every $a \in A$. Now, 
let $\Gamma 
: (\mathcal{F}^{\ADF} \cup \FMMA) \times \Lambda 
\rightarrow \Lambda$, 
alternatively 
$\Gamma^{(\mathcal{F}^{\ADF} \cup \FMMA)}: 
\Lambda \rightarrow 
\Lambda$, 
be such that, 
for any $F^{\ADF} \equiv (A, R, C)\  (\in \mathcal{F}^{\ADF})$, 
any $F \equiv (A, R, f_Q)\ (\in \FMMA)$ 
and any $\lambda \in \Lambda^A$,  all the following hold, 
with $Y$ denoting either of $\FFADF$ and $\FFMMA$. 
\begin{enumerate} 
  \item 
	  $\Gamma^{Y}(\lambda)
     \in \Lambda^A$.
   \item For every $a \in A$ and every $l \in \{\inL, \outL\}$, 
	   $\Gamma^{Y}(\lambda)(a) = l$ iff, 
		for every $\lambda_x \in \maxi_A(\lambda)$,    
		$\lambda_x$ designates only $l$ for $a$ in $Y$. 
\end{enumerate}
\noindent In a nutshell \cite{Brewka13}, $\Gamma^{Y}(\lambda) 
$ 
gets a consensus of every $\lambda_x \in \maxi_A(\lambda)$ 
on the label of each $a \in A$: if each one of them 
says only $\inL$ for $a$, then $\Gamma^{Y}(\lambda)(a) = 
\inL$, 
if each one of them says only $\outL$ for $a$, then 
$\Gamma^{Y}(\lambda)(a) = \outL$,
and for the other cases $\Gamma^{Y}(\lambda)(a) = \undecL$.  

\begin{definition}[{\ADF} semantics] 
	For any $Y \equiv (A, R, X)\ (\in (\FADF \cup \FMMA))$, we say
	$\lambda \in \Lambda^A$ is: 
	\begin{itemize} 
		\item adf.complete iff $\lambda$ 
			is a fixpoint of $\Gamma^Y$.  
			We denote the set of all adf.complete labellings 
			of $Y$ by $\Lambda^{adf.com}_Y$. 
		\item adf.preferred iff $\lambda$ 
			is maximal in $(\Lambda^{adf.com}_Y, \preceq)$.   
			We denote the set of all adf.preferred labellings
			of $Y$ by $\Lambda^{adf.prf}_Y$.  
		\item adf.grounded iff $\lambda$  
			is the least fixpoint of $\Gamma^Y$.    
			We denote the set of all adf.grounded labellings 
			of $Y$ by $\Lambda^{adf.grd}_Y$.  
        \end{itemize} 
	We call $\Lambda^{adf.com}_Y$ (, $\Lambda^{adf.prf}_Y$, 
	$\Lambda^{adf.grd}_Y$) 
	the adf.complete (, adf.preferred, adf.grounded) semantics of 
	$Y$.\footnote{It seems that formulation of the adf.stable 
	semantics has undergone 
	and may potentially undergo further change. We leave it out.} 
\end{definition} 

\noindent Unlike the maximally proper semantics, 
the {\ADF} semantics do not necessarily equate to the exact semantics 
regardless of existence of the latter. 

\begin{example}[Counter-example] \label{ex_counter_example} 
 We considered the following $\FFADF$ in 
	Example \ref{ex_illustration_of}:\linebreak 
	\begin{tikzpicture}[baseline= (a).base] 
			\node[scale=.94] (a) at (0,0){
	\begin{tikzcd}[
  column sep=small, row sep=small,inner sep=0pt]  
  \underset{C_p}{a_p}  
	\arrow[r] & 
	\underset{C_q}{a_q} \arrow[l,shift left]  
\end{tikzcd}  
};
		\end{tikzpicture} 
	with:  $C_p({[a_q:\inL]^{\lambda}}) = \undecL$; 
	$C_p([a_q:l]^{\lambda}) = \inL$ for $l \in \{\undecL, \outL\}$; 
	$C_q([a_p:\inL]^{\lambda}) = \undecL$;  
	and $C_q([a_p:l]^{\lambda}) = \inL$ for $l \in \{\undecL,\outL\}$.  
	There are two exact labellings of $\FFADF$, namely  
	${[a_p:\inL,a_q:\undecL]^{\lambda}}$ and 
	${[a_p:\undecL,a_q:\inL]^{\lambda}}$. 
        
	For the {\ADF} semantics, we have:  
	\begin{align*} 
		\Gamma^{\FFADF}([a_p:\inL,a_q:\outL]^{\lambda}) &=  
			 [a_p:\inL,a_q:\undecL]^{\lambda}. \\
		  \Gamma^{\FFADF}([a_p:\outL,a_q:\inL]^{\lambda}) &= 
			 [a_p:\undecL,a_q:\inL]^{\lambda}. \\
		  \Gamma^{\FFADF}([a_p:\inL,a_q:\inL]^{\lambda}) &= 
			 [a_p:\undecL,a_q:\undecL]^{\lambda}.\\ 
		 \Gamma^{\FFADF}([a_p:\outL,a_q:\outL]^{\lambda}) &= 
			 [a_p:\inL,a_q:\inL]^{\lambda}. \\
		 \Gamma^{\FFADF}([a_p:\undecL,a_q:\undecL]^{\lambda}) &= 
			 [a_p:\undecL,a_q:\undecL]^{\lambda}.\\ 
		 \Gamma^{\FFADF}([a_p:\inL, a_q:\undecL]^{\lambda}) &= 
			 [a_p:\undecL,a_q:\undecL]^{\lambda}.\\ 
		 \Gamma^{\FFADF}([a_p:\undecL,a_q:\inL]^{\lambda}) &= 
			 [a_p:\undecL,a_q:\undecL]^{\lambda}.\\
		 \Gamma^{\FFADF}([a_p:\outL, a_q:\undecL]^{\lambda}) &= 
			 [a_p:\undecL,a_q:\inL]^{\lambda}.\\ 
		 \Gamma^{\FFADF}([a_p:\undecL,a_q:\outL]^{\lambda}) &= 
			 [a_p:\inL,a_q:\undecL]^{\lambda}. 
	\end{align*}
	Hence, $\{[a_p:\undecL,a_q:\undecL]^{\lambda}\}$ 
	is the adf.complete, adf.preferred, and adf.grounded 
	semantics of $\FFADF$. Evidently, there is no intersection 
	in the set of all the exact labellings and 
	any one of 
	the 3 {\ADF} semantics for this example. 
	The same result is obtained for {\FMMA} with $\FFMMA$: 
	\begin{tikzpicture}[baseline= (a).base]
		\node[scale=.94] (a) at (0,0){
	  \begin{tikzcd}[
  column sep=small, row sep=small,inner sep=0pt]  
  \underset{((0,0), (1,1))}{a_p}  
	\arrow[r] & 
		  \underset{((0,0),(1,1))}{a_q} \arrow[l,shift left]  
\end{tikzcd}  
};
	\end{tikzpicture} 
	\hfill$\clubsuit$ 
\end{example} 

\begin{theorem}[Non-conservation]\label{thm_non_conservation}   
       Assume $F \equiv (A, R, f_Q)\ (\in \mathcal{F})$ 
	and $x \in \{\text{com}, \text{prf}, \text{gr}\}$, then 
	neither $\Lambda^{exact}_F \subseteq \Lambda^{adf.x}_F$ 
	nor $\Lambda^{adf.x}_F \subseteq \Lambda^{exact}_F$ 
          generally holds. 
\end{theorem}
\begin{proof} 
    See Example \ref{ex_counter_example}. \hfill$\Box$ 
\end{proof} 
\begin{corollary}[Relation between maximally proper semantics 
	and {\ADF} semantics]\label{cor_relation} 
     Assume $F \equiv (A, R, f_Q)\ (\in \mathcal{F})$ 
	and $x \in \{\text{com}, \text{prf}, \text{gr}\}$, then 
	neither $\Lambda^{maxi.x}_F \subseteq \Lambda^{adf.x}_F$ 
	nor $\Lambda^{adf.x}_F \subseteq \Lambda^{maxi.x}_F$ 
          generally holds. 
\end{corollary} 
\begin{proof} 
 By Theorem \ref{thm_conservation} and Theorem \ref{thm_non_conservation}. 
	\hfill$\Box$ 
\end{proof} 

\section{Discussion on Future Directions}  
We conclude this paper by describing  
research directions. 
\subsection{Identification of general conditions 
		to guarantee existence of the exact semantics} 
		By Theorem \ref{thm_nothing}, there is no possibility 
		of knowing in advance of time whether an arbitrary 
		may-must argumentation has a non-empty 
		exact semantics. As we discussed, this is also the case 
		for \ADF, and, with a simple extrapolation, 
		in fact for any labelling-based 
		argumentation theory accommodating 
		sufficiently unrestricted assignment of 
		labelling conditions. On the other hand,  
		we know for a fact that some labelling conditions 
		do guarantee the universal existence 
		of the exact semantics. The classical deterministic labelling 
		\cite{Caminada06} in Section 2 
		is one example, which, by Theorem \ref{thm_correspondence_acceptability}, is equated to the choice of 
		 $((|\pre^F(a)|, |\pre^F(a)|), (1, 1))$ 
		for every argument $a$ in a given may-must 
		argumentation $F$; and 
		choice of $((0, |\pre^F(a)|+1),(0, |\pre^F(a)|+1))$  
		for every argument $a$ in $F$ is another example 
		with non-deterministic labellings. 
		How much freedom, then, can we allow in the labelling conditions 
		of a labelling-based argumentation theory 
		before the universal existence of the exact semantics 
		is lost for good?  
		Where is the boundary of existence and non-existence? 
		It is somehow striking, all the more so given the publication 
		of the first {\ADF} paper dated a decade ago, that these questions 
		have not amassed much interest in the argumentation 
		community so far; however, if the question 
		of an inference of a property of an argumentation 
		graph being decidable or undecidable is an interesting 
		question, these are equally interesting questions worth 
		an investigation. One approach is to start by identifying 
		general enough conditions to ensure the existence, 
		to peruse further refinement from the result. Our current 
		conjecture is that, so long as no argument satisfies 
		both must-a and must-r simultaneously, there exists 
		the exact semantics. This, however, 
		needs to be made formally certain. \\
\subsection{Family of may-must argumentations} 
		 We covered the classic interpretation 
		 of modalities, which 
		 led to the label designation in the table of 
		 Fig. \ref{fig_table}.  
		 There are other interpretations of modalities 
		 in the literature. Thus, it is plausible that 
		 employing different interpretations 
		 we obtain other types of label designation 
		 resulting in 
		   a family of may-must argumentations. It is of interest 
		   to detail similarities 
		   and differences among them.  
\subsection{May-must multi-agent argumentation} 
		  In the context of multi-agent argumentation, 
		  may- must- conditions of an argument can be thought of 
		  as being obtained through the synthesis of  
		  agents' perspectives on the acceptance statuses 
		  of the argument in the possible worlds 
		  accessible to them. We may hoist 
		  argumentation graphs, possible worlds and 
		  agents as 3 key conceptual pillars  
		  of may-must multi-agent argumentation (Cf. Fig.  
		  \ref{fig_pillars}). We describe 
		  how they interrelate below. \\
		  \begin{figure}[!t]  
\centering
   \includegraphics[scale=0.19]{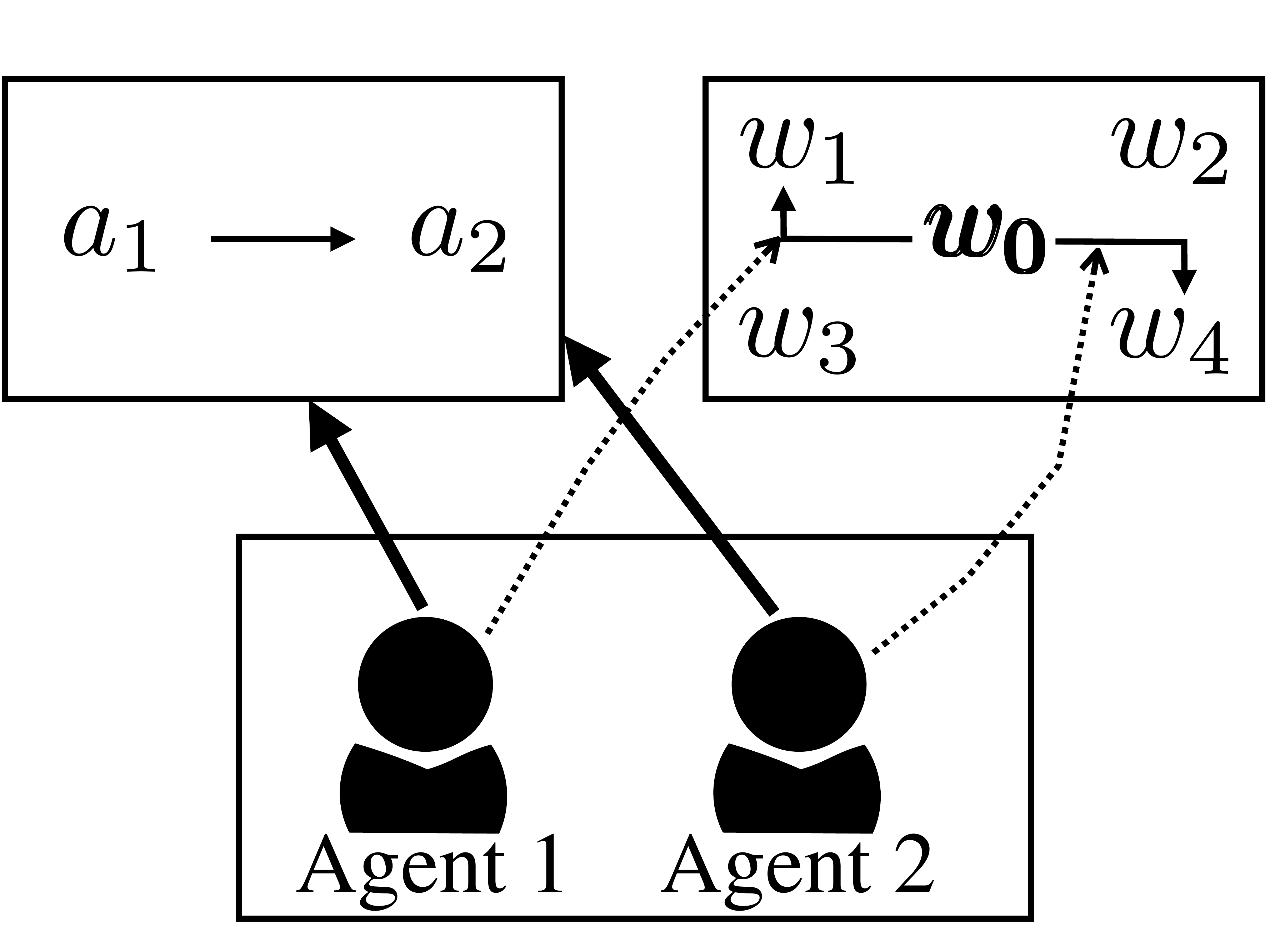} 
\caption{Three key components for may-must 
   multi-agent argumentation (\textbf{top left}: argumentation graphs, 
      \textbf{top right}: possible worlds (accessibility 
     of a possible world from itself is implicitly assumed and 
			  not explicitly shown), \textbf{bottom}:  
       agents). An argumentation graph is viewed 
       by an agent. Each agent then considers possible 
       situations (in this illustration,  
        Agent 1 thinks $w_1$ is accessible from $w_0$, 
      while Agent 2 thinks $w_4$ is accessible from $w_0$). 
   }
\label{fig_pillars}  
\end{figure} 

\noindent \textbf{Argumentation graphs as objects to be observed by agents } 
Given argumentation graphs, 
we can assume there are agents 
observing them. 
For simplification,  
let us assume that they observe them entirely and accurately. 
Fig. \ref{fig_pillars} shows one example where Agent 1 and Agent 2
are seeing 
\begin{tikzcd}[column sep=small,row sep=small] 
   a_1 
   \rar &  a_2 . 
\end{tikzcd}   

{\ }\\
\noindent \textbf{Possible worlds to be thought of 
by agents }    
Every agent, as in epistemic logic (e.g. \cite{Rendsvig19}), 
has its own understanding of which 
situation (world) is possible. 
Fig. \ref{fig_pillars} illustrates an example where 
Agent 1 sees both $w_0$ (the current world) and $w_1$, 
while Agent 2 sees both $w_0$ and $w_4$. Each agent, 
in each possible world it sees, considers  
label assignments to $a_1$ and $a_2$, one out of 
$\inL$, $\outL$ and $\undecL$. 
Note it is possible but not necessary 
that an agent assigns the same labels as others in a possible world 
they commonly see. \\

\noindent \textbf{Synthesis of the accessibility relation }  
As for where the possible worlds 
as are accessible from the current world come from 
in the multi-agent setting, 
we can understand that they are the collection of 
all the possible worlds observed accessible by at least one agent. 
In Fig. \ref{fig_pillars}, Agent 1 sees $w_1$ accessible, and Agent 2 sees 
$w_4$ accessible, which, once put together, 
becomes the one 
shown there. \\

\noindent \textbf{Synthesis of label assignments into may-must 
conditions } 
To understand from where the may-must conditions arise,  
we can similarly allude to the synthesis of  
label assignments by agents. Since 
labelling dependency only goes from 
a source argument to a target argument, 
for Fig. \ref{fig_pillars}, 
may-must conditions of $a_1$ can be first determined prior to $a_2$'s. 

 If Agent 1 and Agent 2 in 
all the possible worlds accessible to them assign $\inL$ 
(resp. $\outL$) 
to $a_1$, then only $\inL$ (resp. $\outL$) should 
be designated 
for $a_1$. 
There is only one entry in Fig. \ref{fig_table} where that is possible - 
when $a_1$ satisfies \mbox{must-a} and \mbox{not-r} 
(resp. \mbox{must-r}  and \mbox{not-a}). For that, 
$a_1$ must have $((0, 0), (m_1, m_2))$ with 
$1 \leq m_1, m_2$ (resp. $((n_1, n_2), (0, 0))$ with 
$1 \leq n_1, n_2$). 

If (1) Agent 1 assigns $\undecL$ to $a_1$ in $w_1$, 
(2) Agent 2 assigns $\undecL$ to $a_1$ in $w_4$, 
and (3) it is not the case that both 
of them assign either $\inL$ or $\outL$ 
to $w_0$, then that is tantamount to them assigning $\undecL$ 
to $a_1$ in all $w_0, w_1, w_4$, where 
only $\undecL$ should be designated for 
$a_1$. There are only two entries   
in Fig. \ref{fig_table} where that is the case  (must-a 
and must-r, or else not-a and not-r). 
Thus, we derive $((0,0), (0,0 ))$ or 
$((n_1, n_2), (m_1, m_2))$ with 
$1 \leq n_1, n_2, m_1, m_2$ as the may-must conditions 
of $a_1$.  

For the other cases, let us just consider 
one case where $a_1$ is assigned: $\inL$ by \mbox{Agent 1} and Agent 2 
in $w_0$; $\outL$ by Agent 1 in $w_1$; 
and $\undecL$ by Agent 2 in $w_4$. 
The entry in Fig. \ref{fig_table} in which 
any of the three labels is possible 
is when $a_1$ satisfies may$_s$-a and 
\mbox{may$_s$-r}, which results in 
$a_1$'s may-must conditions  
to be $((0, n_2), (0, m_2))$ with $1 \leq n_2, m_2$.   
The other cases are similarly reasoned. \\

Let us now move onto $a_2$. If the agents' label assignments 
for $a_1$ are as in the case we have just covered 
(Agent 1 assigns $\inL$ to $a_1$ in $w_0$ and $\outL$ in $w_1$, 
Agent 2 assigns $\inL$ to $a_1$ in $w_0$ and $\undecL$ in $w_4$), 
then $a_1$ may take any of the 3 labels 
on a non-deterministic basis, resulting in the following cases.

\begin{itemize} 
  \item  $a_1$ is assigned $\inL$: 
then for $a_2$'s label, only $w_0$ is relevant.  
       Say both Agent 1 and Agent 2 assign $\outL$ to $a_2$ in $w_0$, 
	then only $\outL$ is designated for $a_2$. 
		Similarly for any other cases.
 \item $a_1$ is assigned $\outL$: then 
	     only $w_1$, the possible world 
		  accessible to Agent 1 but not to Agent 2, 
		is relevant. The label Agent 1 assigns 
                  to $a_2$ in $w_1$ is the label to be designated for it. 
 \item $a_1$ is assigned $\undecL$: then only $w_4$, the possible
	     world accessible to Agent 2 but not to Agent 1,
		 is relevant. The label Agent 2 assigns 
                  to $a_2$ in $w_4$ is the label to be designated for it. \\
\end{itemize} 

\noindent \textbf{Static synthesis } 
As such, suppose that Agent 1 assigns $\outL$ to $a_2$ in $w_0$ and 
$\inL$ to $a_2$ in $w_1$, and that Agent 2 assigns $\outL$ to $a_2$ in 
$w_0$ and $\inL$ to $a_2$ in $w_4$, 
then we could understand the origin of the may-must conditions 
of $a_2$ to be the synthesis 
of every one of the 3 possibilities above, which leads to 
$a_2$ satisfying either: 
(\textbf{A}) must-r   %
and not-a when $a_1$ is assigned $\inL$; or  
(\textbf{B}) must-a and not-r, %
otherwise. 
As it turns out, however, there are no two pairs 
of natural numbers that satisfy them all.

To see that, suppose the may-must conditions of $a_2$ 
form $((n^{a_2}_1, n_2^{a_2}), (m_1^{a_2}, m_2^{a_2}))$.  
The condition (\textbf{A}) then gives us two 
constraints: $1 \leq n_1^{a_2}, n_2^{a_2}$; 
and $m_2^{a_2} \in \{0,1\}$. 
The condition (\textbf{B}) meanwhile gives us $n^{a_2}_1 = n^{a_2}_2 = 0$ and 
$1 \leq m^{a_2}_1, m^{a_2}_2$.  
One nearby approximation of all these constraints 
gives us $((n_1^{a_2}, n_2^{a_2}), 
(m_1^{a_2}, m_2^{a_2})) = ((0, 0), (1,1))$, which 
does not satisfy (\textbf{A}).  With another:  
$((n_1^{a_2}, n_2^{a_2}), 
(m_1^{a_2}, m_2^{a_2})) = ((1, 1), (1,1))$, 
(\textbf{B}) is not satisfied when $a_1$ is assigned 
$\undecL$. Hence, with the perspective of may-must 
multi-agent argumentation spelled out so far, we may
understand the origin 
of the may-must conditions of $a_2$ 
to be at best some non-deterministically 
correct synthesis of 
every one of the 3 possibilities above.  \\

\noindent \textbf{Dynamic synthesis } 
However, our perspective provides us also with an idea 
of more dynamic synthesis. 
Assume the setting so far (Agent 1: $a_1$ gets $\inL$ in $w_0$ 
and $\outL$ in $w_1$, and $a_2$ gets $\outL$ in $w_0$ and 
$\inL$ in $w_1$. Agent2: $a_1$ gets $\inL$ in $w_0$ 
and $\undecL$ in $w_4$, and $a_2$ gets $\outL$ in $w_0$ 
and $\inL$ in $w_4$),  then since 
$a_1$'s may-must conditions can be determined beforehand, 
we can determine $a_2$'s may-must conditions conditionally 
to each 
label assignable to $a_1$. Then, with $a_1$'s label $\inL$, 
we only need to synthesise 
those possible worlds in which $a_1$ gets $\inL$, only 
$w_0$ in the assumed setting, where $a_2$ is assigned only $\outL$. 
From this, we obtain $a_2$'s may-must conditions 
conditional to $a_1$'s label being $\inL$ 
to be $((n^{a_2}_1, n^{a_2}_2), (m^{a_2}_1, m^{a_2}_2))$ with $1 \leq n^{a_2}_1, n^{a_2}_2$ and 
$m^{a_2}_2 \in \{0,1\}$. Conditional to $a_1$'s label being $\outL$, 
similarly, we only synthesise in $w_1$ to obtain 
$n^{a_2}_2 \in \{0,1\}$ and $1 \leq m^{a_2}_1, m^{a_2}_2$. 
Conditional to $a_1$'s label being $\undecL$, 
we only synthesise in $w_4$ to obtain 
$n^{a_2}_1 = n^{a_2}_2 = 0$ and $1 \leq m^{a_2}_1, m^{a_2}_2$. \\\\

\noindent We have illustrated a perspective which we believe 
will help adapt may-must argumentation into a multi-agent argumentation 
theory. It should be interesting to expand further on the idea 
to a fuller-fledged mutli-agent argumentation theory.

\section{Conclusion} 
We have presented may-must argumentation, a novel labelling-based argumentation 
theory with may-must scales. The following characteristics 
are salient. 
\begin{description} 
	\item[1. Principled generalisation of classic 
		acceptability conditions]{\ }\\ Must- conditions 
		generalise the classical 
		labelling conditions proposed in \cite{Caminada06} 
		(Cf. Section 2) in a principled way. Instead of designating 
		$\outL$ with one $\inL$-labelled attacker 
		or $\inL$ when every attacker is labelled $\outL$,
		 may-must argumentation permits us to specify 
		 any number of accepted or rejected 
		  attackers as a requirement 
		 for the attacked to be designated $\outL$ or $\inL$.  \\   
	 \item[2. Non-deterministic labelling] {\ }\\
		 May- conditions accommodate non-deterministic 
		 label designations. \\ 

	 \item[3. Locality]{\ }\\
		 May- must- conditions are given locally to each argument.\\ 
		 
	 \item[4. Two-way evaluation]{\ }\\
		 Every argument is evaluated of its acceptability 
		and of its rejectability independently first, and 
		then the two assessments are combined 
		into a final decision for label designation. 
		From online reviews and recommendations 
		to legal trials with a prosecutor 
		and a defence lawyer, 
		the two-way evaluation 
		is a common practice in real-life situations.  \\
	\item[5. Monotonicity]{\ }\\
		May- must- conditions accommodate a level of 
	        monotonicity.  
		Once any may- must- condition is 
		satisfied with $n$ accepted/rejected attacking arguments, 
		$m\ (\leq n)$ accepted/rejected attacking arguments 
		also satisfies the same may- must- condition.  
\end{description} 
As we discussed in section 1.3, previous lines of research  
in the literature 
have some of these themes in common:  
graded argumentation \cite{Grossi19}, 
while not a labelling-approach, 
shares \textbf{1.} and \textbf{5.}; 
{\ADF} shares \textbf{3.}; 
a classic labelling-approach \cite{Jakobovits99} \textbf{2.};   
and 
another classic labelling-approach \cite{Caminada06} \textbf{4.} 
and \textbf{5.}. However, to the best of our knowledge, 
may-must argumentation is the first labelling-based argumentation theory 
respecting all of them. 

As for the semantics, we presented 3 types: 
exact semantics, maximally proper semantics and  
{\ADF} semantics, with technical comparisons revealing 
the relation to hold among them (Cf. Theorem \ref{thm_conservation}, 
Theorem \ref{thm_non_conservation}, and Corollary 
\ref{cor_relation}). We also discussed some interesting future research
directions.

Modalities in formal argumentation
have been conceptually 
around since the  beginning of formal argumentation. 
For instance, since acceptance judgement 
in the classical labelling \cite{Caminada06} is 
all $\outL$ attackers or some $\inL$ attacker, 
it is perfectly natural to rephrase the universal and the existential 
conditions with modal operators, e.g. \cite{Caminada09}.   
It is also possible to regard an argumentation graph itself 
as a possible world and to define an accessibility relation 
among them \cite{Barringer12} (2012 publication; 2005 is the year 
of publication of its conference paper version) to link 
the behaviour of different argumentation graphs. However,  
as far as we are aware, previous studies 
on formal argumentation and modalities 
do not consider with possible worlds what
it may mean in labelling argumentation for an argument 
to be judged possibly or necessarily accepted/rejected  
or how the two independent judgements may be combined into 
non-deterministic labelling.

\section*{Acknowledgements} 
This work was partially supported by 
JST CREST Grant Number JPMJCR15E1, 
and the first author is also additionally supported 
by AIP challenge program, Japan.

\hide{ 
\section{Variations with Preference and Intuitionism}   
In this section, we understand $\mathcal{F}$ in a broader 
perspective with preference and intuitionism.  

\subsection{May-must argumentation with preference} 
For the modal interpretation for $\mathcal{F}$,  
we assumed that any accessible possible world in which 
both acceptance and rejection of an argument are implied 
implies, due to the conflict between the two, $\undecL$ for the argument. 
If, however, there is a reason to prefer acceptance to rejection, 
or rejection to acceptance, it can be $\inL$, or $\outL$, instead.   
\ryuta{Actually, it would be better to be very rigorous.} 

Hence, it is possible to extend may-must argumentation 
with a parameter to control the preference.  

\begin{definition}[Nuance tuple with preference]  
    We define a nuance tuple with preference to be 
   $(\pmb{X}_1,  \pmb{X}_2, l)$ for some 
    $\pmb{X}_1, \pmb{X}_2 \in \mathbb{N} \times \mathbb{N}$, 
	and $l \in \mathcal{L}$. 
   We denote the class of all nuance tuples with preference by 
   $\mathcal{Q}^{\pmb{o}}$. For any $Q
   \in \mathcal{Q}^{\pmb{o}}$, 
   we call $(Q)^1$ its may-must acceptance scale, 
   $(Q)^2$ its may-must rejection scale, and  
   $(Q)^3$ its preference. 
\end{definition}  
The exact role of preference will be defined shortly. 
For now, it suffices to note: $(\pmb{X}_1, \pmb{X}_2, \inL)$ 
prioritises acceptance, $(\pmb{X}_1, \pmb{X}_2, \outL)$
rejection, and $(\pmb{X}_1, \pmb{X}_2, \undecL)$ neither. 
\begin{definition}[May-must argumentation with preference] 
    We define a (finite) may-must argumentation with preference 
	to be a tuple $(A, R, f_{Q^{\pmb{o}}})$ with: 
    $A \subseteq_{\text{fin}} \mathcal{A}$; 
	$R \in \mathcal{R}^A$; and 
	$f_{Q^{\pmb{o}}}: A \rightarrow \mathcal{Q}^{\pmb{o}}$, such that
    $((f_Q(a))^i)^1 \leq ((f_Q(a))^i)^2$ for every
    $a \in A$ and every $i \in \{1,2\}$. 
    
    We denote the class of all 
    (finite) may-must argumentations with
	preference by $\mathcal{F}^{\pmb{o}}$, 
	and refer to its member by $F^{\pmb{o}}$ with or without a 
	subscript. 
\end{definition} 

\begin{wrapfigure}[25]{r}{6.3cm}  
	\vspace{-\intextsep}
 \begin{tikzcd}[column sep=tiny,row sep=tiny,
	 /tikz/execute at end picture={
		 \draw (-1.35, 1.6) -- (-1.35,4.3); 
		 \draw (-2.7, 3.87) -- (2.6, 3.87);
		 \draw (-1.35, -1.5) -- (-1.35, 1.2);
		 \draw (-2.7, 0.77) -- (2.6, 0.77); 
                 \draw (-1.35, -4.6) -- (-1.35, -1.9);
		 \draw (-2.7, -2.35) -- (2.6, -2.35); 
		 }]  
	 \dot{>} & \text{must}\text{-r} & \text{may}_s\text{-r} & \text{not}\text{-r} \\
	  \text{must}\text{-a} & \inL & \inL & \inL \\ 
	 \text{may}_s\text{-a}	 &  \outL/\inL  & \outL/\inL & \inL \\
	 \text{not-a} & \outL & \outL/\inL  & \inL\\\\
	  \dot{=} & \text{must}\text{-r} & \text{may}_s\text{-r} & \text{not}\text{-r} \\
	  \text{must}\text{-a} & \undecL & \inL ? & \inL \\ 
	 \text{may}_s\text{-a}	 &  \outL ? & \textsf{any} & \inL ? \\
	 \text{not-a} & \outL & \outL ? & \undecL\\\\
	 \dot{<} & \text{must}\text{-r} & \text{may}_s\text{-r} & \text{not}\text{-r} \\
	  \text{must}\text{-a} & \outL & \outL/\inL  & \inL \\ 
	 \text{may}_s\text{-a}	 &  \outL & \outL/\inL & \outL/\inL\\
	 \text{not-a} & \outL & \outL & \outL
\end{tikzcd}   
\caption{Label designation for may-must argumentation with preference. 
	$\outL/\inL$ means either of $\outL$ and $\inL$.}
\label{fig_table_preference} 
\end{wrapfigure} 

Almost all definitions given earlier apply 
to $\mathcal{F}^{\pmb{o}}$, except label designation.   
\begin{itemize} 
	\item must-acceptance (resp. must- rejeection) of $a$ 
		implies $a$'s acceptance (resp. rejection) 
		in every accessible possible world. 
	\item may$_s$- acceptance (resp. may$_s$- rejection) of $a$ 
		implies $a$'s acceptance (resp. 
		rejection) in a non-empty subset of all accessible 
		possible worlds and $a$'s rejection (resp. acceptance) 
		in the other accessible possible worlds (if any).
	\item not- acceptance (resp. not- rejection) is equivalent 
		to must- rejection (resp. must- acceptance). 
\end{itemize}

The difference is therefore in the aggregation of the individual 
assessments. 
For any argument and for any accessible possible world 
in which both acceptance and rejection of the argument 
are implied, we obtain, with $\dot{=}$, 
designation of $\undecL$ for the argument as before. However, with 
$\dot{>}$, we obtain designation of $\inL$ for the argument, 
and, with $\dot{<}$, designation of $\outL$ for it. 
Fig. \ref{fig_table_preference} summarises this.

\subsection{Intuitionistic may-must argumentation with preference} 
Along with preference, we can also consider an intuitionistic 
interpretation on the modal judgements. Recall for the interpretation 
of not- acceptance and not- rejection that the classical interpretation 
assumed so far has flipped not- acceptance into must- reject 
and not- rejection into must- accept. That is, 
in any possible world, an argument whose acceptance 
(resp. rejection) 
is not implied is implied of rejection (resp. acceptance). 
An intuitionistic point of view rejects this 
as too strong a condition, to obtain for each $a \in $ and each $\lambda \in 
\Lambda$ that $a$'s satisfaction of:
\begin{itemize} 
     \item must-acceptance (resp. must- rejeection) of $a$ 
		implies $a$'s acceptance (resp. rejection) 
		in every accessible possible world. 
	\item may$_s$- acceptance (resp. may$_s$- rejection) of $a$ 
		implies $a$'s acceptance (resp. 
		rejection) in a non-empty subset of all accessible 
		possible worlds and $a$'s rejection (resp. acceptance) 
		in the other accessible possible worlds (if any).
	\item not- acceptance (resp. not- rejection) does not 
		imply $a$'s acceptance or $a$'s rejection 
		in every accessible possible world. 
\end{itemize}
\begin{wrapfigure}[25]{r}{6.3cm}  
	\vspace{-\intextsep}
 \begin{tikzcd}[column sep=tiny,row sep=tiny,
	 /tikz/execute at end picture={
		 \draw (-1.35, 1.6) -- (-1.35,4.3); 
		 \draw (-2.7, 3.87) -- (2.6, 3.87);
		 \draw (-1.35, -1.5) -- (-1.35, 1.2);
		 \draw (-2.7, 0.77) -- (2.6, 0.77); 
                 \draw (-1.35, -4.6) -- (-1.35, -1.9);
		 \draw (-2.7, -2.35) -- (2.6, -2.35); 
		 }]  
	 \dot{>} & \text{must}\text{-r} & \text{may}_s\text{-r} & \text{not}\text{-r} \\
	  \text{must}\text{-a} & \inL & \inL & \inL \\ 
	 \text{may}_s\text{-a}	 &  \outL/\inL  & \textsf{any} & \inL? \\
	 \text{not-a} & \outL & \outL?  & \undecL\\\\
	  \dot{=} & \text{must}\text{-r} & \text{may}_s\text{-r} & \text{not}\text{-r} \\
	  \text{must}\text{-a} & \undecL & \inL ? & \inL \\ 
	 \text{may}_s\text{-a}	 &  \outL ? & \textsf{any} & \inL ? \\
	 \text{not-a} & \outL & \outL ? & \undecL\\\\
	 \dot{<} & \text{must}\text{-r} & \text{may}_s\text{-r} & \text{not}\text{-r} \\
	  \text{must}\text{-a} & \outL & \outL/\inL  & \inL \\ 
	 \text{may}_s\text{-a}	 &  \outL & \textsf{any} & \inL?\\
	 \text{not-a} & \outL & \outL? & \undecL
\end{tikzcd}   
\caption{Label designation for intuitionistic may-must argumentation 
	with preference.}
\label{fig_table_intuitionism} 
\end{wrapfigure} 
Hence, in intuitionistic may-must argumentation, a possible world 
may imply for an argument any one of: (1) only acceptance; 
(2) only rejection; (3) both acceptance and rejection; 
and (4) neither acceptance nor rejection, 
giving rise to (4) not present in classic may-must argumentation.   
Label designations are summarised in Fig. \ref{fig_table_intuitionism}. 

Interestingly, for both (3) and (4), $\undecL$ 
is the designated label for the argument, for (3) because 
at most one of acceptance and rejection can be implied
for any argument in any possible world, and for (4) because 
either of acceptance and rejection should be implied 
for any argument in any possible world. Hence, despite 
the distinction between (3) and (4), we still have the same label 
designation as with the classic may-must argumentation, only, 
however, so long as no preference comes into play. 

With preference, more interestingly, we

intuitionistic perspective 
briefly mentioned in Section 3, and how that affects 
may-must argumentation labelling designation and semantics.  
Then, we identify several subclasses of the generalised may-must argumentation 
that may be interesting because of properties 
presented in them.

\subsection{Concretisation and abstraction}  
We first consider the mapping of $F \in \mathcal{F}$ into $F^{\ADF} \in 
\mathcal{F}^{\ADF}$. Recall that $\mathcal{F}$ admits non-deterministic labellings 
where potentially multiple labels are designated, while
$\mathcal{F}^{\ADF}$ admits only one. The mapping of $F$'s $f_Q$ 
into $F^{\ADF}$'s $C$
hence requires choosing one particular label for each argument $a$ 
that is designated by $\lambda \in \Lambda^{\pre{F}(a)}$, as more rigorously 
defined below. See also Example \ref{ex_concretisation} 
to follow for a concrete example. 
\begin{definition}[Concretisation: from $F$ to $F^{\ADF}$]\label{def_concretisation_from} 
	Let $\Gamma$ be a class of 
	all functions $\gamma: \mathcal{F} \rightarrow \mathcal{F}^{\ADF}$, 
	each of which is 
	such that, for any $F \equiv (A, R, f_Q)\ (\in \mathcal{F})$, 
	$\gamma(F) = (A, R, C)$ where $C$ satisfies the following for any 
	$a \in A$. 
	\begin{itemize} 
		\item If $\Lambda^{\pre^F(a)} = \emptyset$, 
			then if $L \subseteq \mathcal{L}$ 
			is such that some $\lambda \in \Lambda^A$ 
			designates $l \in L$ but does not designate 
			any $l \in (\mathcal{L} \backslash L)$ 
			for $a$, then $C_a(\emptyset) \in L$. 
		\item Otherwise, for any $\lambda 
	\in \Lambda^{\pre^F(a)}$, if $L \subseteq \mathcal{L}$ 
	is such that $\lambda$ designates each $l \in L$ but does not 
	designate any $l \in (\mathcal{L} \backslash L)$ for $a$, then 
	$C_a(\lambda) \in L$.  
       \end{itemize}
	We say that $F^{\ADF} \in \mathcal{F}^{\ADF}$ is a concretisation 
	of $F \in \mathcal{F}$ iff there is some $\gamma \in \Gamma$ 
	with $F^{\ADF} = \gamma(F)$. By $\pmb{\Gamma}(F)$ we denote  
	the set of all concretisations of $F \in \mathcal{F}$.

	
\end{definition}  
\begin{example}[Concretisation]\label{ex_concretisation} 
	Consider the example $F \equiv (A, R, f_Q) \ (\in \mathcal{F})$ in 
	\mbox{Example 
	\ref{ex_labelling}} again. We show some examples of concretisations of it. 
	Since $\gamma \in \Gamma$ does not modify 
	$A$ and $R$, the problem at hand is identification of 
	$C_{a_x}$ to correspond to $f_Q(a_x)$ for each $x \in \{1, \ldots, 5\}$. 
	See the diagram below. 
  \begin{center} 
	  \begin{tikzcd}[column sep=small]
		  F: \arrow[d,"\gamma"] & 
   \overset{}{\underset{((0, 0), (1,1))}{\ensuremath{a_1}}} 
  \rar & 
\overset{}{\underset{((0,1), (1,2))}{\ensuremath{a_2}}} \rar & 
   \overset{}{\underset{((1,1),(1,1))}{\ensuremath{a_3}}} &  
   \overset{}{\underset{((0,1),(1,1) )}{\ensuremath{a_4}}} \lar &   
   \overset{}{\underset{((0,0), (1, 1))}{\ensuremath{a_5}}} \lar \\
		  F^{\ADF}: & 
		    \overset{}{\underset{C_{a_1}}{\ensuremath{a_1}}} 
  \rar & 
		    \overset{}{\underset{C_{a_2}}{\ensuremath{a_2}}} \rar & 
		    \overset{}{\underset{C_{a_3}}{\ensuremath{a_3}}} &  
		    \overset{}{\underset{C_{a_4}}{\ensuremath{a_4}}} \lar &   
		    \overset{}{\underset{C_{a_5}}{\ensuremath{a_5}}} \lar
\end{tikzcd}   
  \end{center}

	For 
	both $a_1$ and $a_5$, recall (Cf. Example \ref{ex_labelling})
	that 
	  every $\lambda \in \Lambda^{A}$ 
	  designates just $\inL$ for $a_1$ and $a_5$. 
	  As per \mbox{Definition \ref{def_concretisation_from}},  
	  $C_{a_1}(\emptyset) = C_{a_5}(\emptyset) = \inL$ irrespective 
	  of which $\gamma \in \Gamma$. 
          
	  For $a_4$, there are 3 distinct labellings $\lambda_1, 
	   \lambda_2, \lambda_3 \in \Lambda^{a_4}$ 
	   with: $\lambda_1(a_5) = \inL$; 
	   $\lambda_2(a_5) = \outL$; and 
	    $\lambda_3(a_5) = \undecL$. 
	    We have: $\lambda_1$ designates only $\outL$ for $a_4$; 
	    $\lambda_2$ designates only $\inL$ for $a_4$; 
	    and $\lambda_3$ designates $\inL$ and $\undecL$ for $a_4$.  
	   Thus, $C_{a_4}$ is either of the following. 
	   \begin{itemize} 
		   \item $C_{a_4}(\lambda_1) = \outL$, 
			   $C_{a_4}(\lambda_2) = \inL$, 
			   $C_{a_4}(\lambda_3) = \inL$. 
		   \item $C_{a_4}(\lambda_1) = \outL$, 
			   $C_{a_4}(\lambda_2) = \inL$, 
			   $C_{a_4}(\lambda_3) = \undecL$. 
	   \end{itemize}
	   That is, some $\gamma \in \Gamma$ has the first $C_{a_4}$, 
	   while some other $\gamma' \in \Gamma$ has the second $C_{a_4}$. 
           Similarly for $a_2$ and $a_3$, we consider 
	   every $\lambda \in \Lambda^{\pre^F(a_2)}$ and 
	   $\lambda' \in \Lambda^{\pre^F(a_3)}$. 
	    \hfill$\clubsuit$       
\end{example} 
\noindent When we either decrease a may- condition or 
increase a must- condition in $F \in \mathcal{F}$, then we obtain 
at least as large a set of concretisations as before the change (Lemma 
\ref{lem_monotonicity}). 
\begin{definition}[Abstract order] 
  Let $\unlhd: \mathcal{Q} \times \mathcal{Q}$ 
	be such that $(Q_1, Q_2) \in \unlhd$, alternatively 
	 $Q_1 \unlhd Q_2$, holds iff, 
	 for any $i \in \{1,2\}$, 
	  $((Q_2)^i)^1 \leq ((Q_1)^i)^1$ and 
	  $((Q_1)^i)^2 \leq ((Q_2)^i)^2$ both hold. 
	  Then we define $\sqsubseteq$ 
	  to be a binary relation over 
	  $\mathcal{F}$ satisfying that, 
	  for any $F_1 \equiv (A, R, f_Q)$ 
	  and any $F_2 \equiv (A, R, f_Q')$, 
	  $F_1 \sqsubseteq F_2$ holds iff, 
	  for any $a \in A$, 
	  $f_Q(a) \unlhd f_Q'(a)$ holds.  
	  
	  We also extend $\sqsubseteq$ for $2^{\mathcal{F}}$ 
	  in the following manner. 
	  For $\mathcal{F}_x, \mathcal{F}_y \subseteq \mathcal{F}$,  
	  we define 
	  $\mathcal{F}_x \sqsubseteq \mathcal{F}_y$ 
	  iff, for any $F_x \in \mathcal{F}_x$, 
	  there exists some $F_y \in \mathcal{F}_y$ 
	  such that $F_x \sqsubseteq F_y$. 
\end{definition}  

\begin{lemma}[Monotonicity 1]\label{lem_monotonicity_1} 
	  For any $F \equiv (A, R, f_Q)\ (\in \mathcal{F})$ 
	  and any $F' \equiv (A, R, f_Q')\ (\in \mathcal{F})$,  
	  if $F \sqsubseteq F'$ holds, then 
	  $\pmb{\Gamma}(F) \subseteq \pmb{\Gamma}(F')$ holds.  
\end{lemma} 
\begin{proof} 
	By Definition \ref{def_concretisation_from}, it 
	 suffices to show that, for any $a \in A$ and 
	any $\lambda \in \Lambda^{\pre^F(a)}$, if $\lambda$ 
	designates 
	$l \in \mathcal{L}$ for $a$ in $F$, then 
	$\lambda$ designates $l$ for $a$ in $F'$. Now, the differences 
	between $f_Q$ and 
	$f'_Q$ are such that, firstly, if $a$ satisfies 
	must-a condition (resp. must-r condition) in $F$, $a$ 
	satisfies either must-a 
	or may$_s$-a condition (resp. must-r or may$_s$-r condition) in 
	$F'$, and, secondly, if $a$ satisfies not-a condition 
	(resp. not-r condition) in $F$, $a$ 
	satisfies either not-a or may$_s$-a condition 
	(resp. not-r or may$_s$-r condition). Apply 
	Lemma \ref{lem_label_designation_subsumption}. \hfill$\Box$ 
\end{proof}

\noindent Into the other direction of mapping $F^{\ADF} \in \mathcal{F}^{\ADF}$ 
into $F \in \mathcal{F}$, recall (see also Example \ref{ex_concretisation}) that 
$F^{\ADF}$'s $C_a$ determines label designation for each 
$\lambda \in \Lambda^{\pre^{F^{\ADF}}(a)}$, while 
$F$ does the determination based on $|\pre^{F}_{\lambda, \inL}(a)|$ 
and $|\pre^{F}_{\lambda, \outL}(a)|$. This difference 
for instance causes the situation 
where, for distinct $\lambda_1, \lambda_2 \in \Lambda^{\pre^{F^{\ADF}}(a)}$ 
with $|\pre^{F}_{\lambda_1, \inL}(a)| = |\pre^{F}_{\lambda_2, \inL}(a)|$ 
and $|\pre^{F}_{\lambda_1, \outL}(a)| = |\pre^{F}_{\lambda_2, \outL}(a)|$,  
$C_a(\lambda_1) = l_1$ and $C_a(\lambda_2) \not= l_1$, 
while with $F$, $\lambda_1$ and $\lambda_2$ always designate 
the same label(s). As such, the specificity of $F^{\ADF}$'s $C$ 
must be abstracted. Instead of functions from 
$\mathcal{F}^{\ADF}$ to $\mathcal{F}$, however, 
we generally consider functions from 
a subclass $X$ of $2^{\mathcal{F}^{\ADF}}$ to $\mathcal{F}$ 
where any $(A_1, R_1, C_1), (A_2, R_2, C_2) \in X$ 
are such that $A_1 = A_2$ and $R_1 = R_2$. This generalisation 
is for simulating the non-deterministic labellings 
in $\mathcal{F}$ with multiple members of $\mathcal{F}^{\ADF}$ (which 
provide multiple ``$C$''s). Formally, we consider the following class 
of functions.  
See also Example \ref{ex_abstraction} to follow 
for a concrete example. 
\begin{definition}[Abstraction: from $\mathcal{F}^{\ADF}$ to $F$] \label{def_abstraction_from} 
	  Let $\Delta$ be a class of all functions $\alpha: 
	 \mathcal{F}^{\ADF} \rightarrow \mathcal{F}$, 
	 each of which is such that, for any 
	  $F^{\ADF} \equiv (A, R, C)\ (\in \mathcal{F}^{\ADF})$, 
	 $\alpha(F^{\ADF}) = (A, R, f_Q)$ holds 
         where, for every $a \in A$, 
	 $f_Q(a) \equiv ((n_1^a, n_2^a), (m_1^a, m_2^a))$ 
	  satisfies all the following 
	  conditions for any $\lambda \in (\Lambda^{\pre^F(a)} \cup \{\emptyset\})$.  
                       \begin{enumerate}  
			       \item $0 \leq n^a_1 \leq n^a_2 \leq |\pre^F(a)|+1$. 
				       Also 
			       $0 \leq m^a_1 \leq m^a_2 \leq |\pre^F(a)| + 1$.\\                             
			       \item $\lambda = \emptyset$ iff  
				       $\pre^{F}(a) = \emptyset$. Also 
			        $|\pre^{F}_{\emptyset, \outL}{(a)}| 
				       = 
				       |\pre^{F}_{\emptyset, \inL}{(a)}| = 0$. \\
			       \item If $|\pre^{F}_{\lambda, \outL}{(a)}| 
				       < n_1^{a}$ and 
				       $|\pre^{F}_{\lambda, \inL}{(a)}| 
				       < m_1^{a}$, then  
				       $C_{a}(\lambda) = \undecL$.\\
				       (This corresponds to \textbf{not-a, not-r} 
				       satisfaction.)\\
			       \item
				       If $n_1^{a} \leq 
				       |\pre^{F}_{\lambda, \outL}{(a)}| 
				       < n_2^{a}$ and 
				       $|\pre^{F}_{\lambda, \inL}{(a)}| 
				       < m_1^{a}$, then  
				       $C_{a}(\lambda) \in \{\inL, \undecL\}$. \\
				       (\textbf{may$_s$-a, not-r})  \\
			       \item
				       If $n_2^{a} \leq 
				       |\pre^{F}_{\lambda, \outL}{(a)}|$
				       and 
				       $|\pre^{F}_{\lambda, \inL}{(a)}| 
				       < m_1^{a}$, then  
				       $C_{a}(\lambda) = \inL$. \\
				        (\textbf{must-a, not-r}) \\ 
			       \item If $|\pre^{F}_{\lambda, \outL}{(a)}| 
				       < n_1^{a}$ and 
				       $m_1^{a} \leq 
				       |\pre^{F}_{\lambda, \inL}{(a)}| 
				       < m_2^{a}$, then 
				       $C_{a}(\lambda) \in 
				       \{\outL, \undecL\}$.\\
				        (\textbf{not-a, may$_s$-r})\\
			       \item
				       If $n_1^{a} \leq 
				       |\pre^{F}_{\lambda, \outL}{(a)}| 
				       < n_2^{a}$ and  
				       $m_1^{a} \leq 
				       |\pre^{F}_{\lambda, \inL}{(a)}| 
				       < m_2^{a}$,
				         then  
				       $C_{a}(\lambda) 
				       \in \{\inL,\outL,\undecL\}$.\\
				        (\textbf{may$_s$-a, may$_s$-r})\\ 
			       \item
				       If $n_2^{a} \leq 
				       |\pre^{F}_{\lambda, \outL}{(a)}|$
				       and  
				       $m_1^{a} \leq 
				       |\pre^{F}_{\lambda, \inL}{(a)}| 
				       < m_2^{a}$,
				        then  
				       $C_{a}(\lambda) \in \{\inL, \undecL\}$. \\
				       (\textbf{must-a, may$_s$-r}) \\ 
			       \item
				      If $|\pre^{F}_{\lambda, \outL}{(a)}| 
				       < n_1^{a}$ and 
				       $m_2^{a} \leq 
				       |\pre^{F}_{\lambda, \inL}{(a)}| 
				       $, then 
				       $C_{a}(\lambda) = \outL$. \\
				       (\textbf{not-a, must-r})\\
			       \item
				       If $n_1^{a} \leq 
				       |\pre^{F}_{\lambda, \outL}{(a)}| 
				       < n_2^{a}$ and  
				       $m_2^{a} \leq 
				       |\pre^{F}_{\lambda, \inL}{(a)}| 
				       $,
				         then  
				       $C_{a}(\lambda) \in \{\outL,\undecL\}$. \\
				        (\textbf{may$_s$-a, must-r})\\ 
			       \item
				       If $n_2^{a} \leq 
				       |\pre^{F}_{\lambda, \outL}{(a)}|$
				       and  
				       $m_2^{a} \leq 
				       |\pre^{F}_{\lambda, \inL}{(a)}| 
				       $, 
				        then  
				       $C_{a}(\lambda) = \undecL$.\\
				        (\textbf{must-a, must-r}) 
		       \end{enumerate} 
	For any $F^{\ADF} \in \mathcal{F}^{\ADF}$,  
	we say that $F \in \mathcal{F}$ is an abstraction of $F^{\ADF}
	$ iff there is some $\alpha \in \Delta$ 
	with $F = \alpha(F^{\ADF})$. We 
        denote the set of all abstractions of $F^{\ADF}$ 
	by $\pmb{\Delta}(F^{\ADF})$. 
	\end{definition}

\begin{example}[Abstraction] \label{ex_abstraction}  
	Suppose $F^{\ADF}:$ \begin{tikzcd}[column sep=small]
		{\underset{C_{a_1}}{\ensuremath{a_1}}} 
  \rar & 
		{\underset{C_{a_2}}{\ensuremath{a_2}}}  & 
		\overset{}{\underset{C_{a_3}}{\ensuremath{a_3}}} \lar . 
   \end{tikzcd}   
	Assume $\lambda_{l_1l_3} \in \Lambda^{\pre^{F^{\ADF}}(a_2)}$ 
	for $l_1, l_3 \in \mathcal{L}$ 
	is such that $\lambda_{l_1l_3}(a_1) = l_1$ 
	and that $\lambda_{l_1l_3}(a_3) = l_3$. Then 
	$\Lambda^{\pre^{F^{\ADF}}(a_2)} = 
	\bigcup_{l_1, l_3 \in \mathcal{L}}\{\lambda_{l_1l_3}\}$. 
	Assume: 
	\begin{enumerate} 
          \item $C_{a_2}(\lambda_{\undecL\ \undecL}) = \undecL$.  
		  \quad\qquad
		   {\normalfont 2.} \ 
	  $C_{a_2}(\lambda_{l_1l_3}) = \outL$ if $\{\inL\} 
		  \subseteq \{l_1\} \cup \{l_3\}$.  
			\setcounter{enumi}{2}
	  \item $C_{a_2}(\lambda_{\undecL\ \outL}) = \inL$. \qquad\qquad\qquad
		   {\normalfont 4.} \ 
		$C_{a_2}(\lambda_{\outL\ \undecL}) = \undecL$. 
			\setcounter{enumi}{4}
	\item	$C_{a_2}(\lambda_{\outL\ \outL}) = \undecL$. 
	\end{enumerate} 	
	Let us consider which $f_Q(a_2) \equiv ((n_1^{a_2}, n_2^{a_2}), 
	(m_1^{a_2}, m_2^{a_2}))$ 
	can be 
	in an abstraction of $F^{\ADF}$.
	Firstly, due to 1. and 2., we can set $(m_1^{a_2}, m_2^{a_2})$ 
	to $(1, 1)$. On the other hand for $(n_1^{a_2}, n_2^{a_2})$, 
	since we have 3., 4., it cannot be that $n_2^{a_2} \leq 1$, 
	but also we have 5., and $n_2^{a_2} \not=2$. Thus, 
	we must have $n_2^{a_2} = 3$. However, because of 3., 
	$m_1^{a_2}$ cannot be greater than or equal to 2. 
	Hence we can set $m_1^{a_2}$ to be 1, resulting 
	in $(n_1^{a_2}, n_2^{a_2}) = (1, 3)$. 
	It is trivial to see that any $F \in \pmb{\Delta}(F^{\ADF})$ 
	with this $f_Q(a_2)$ least abstracts $C_{a_2}$ of $F^{\ADF}$ 
	among all possible $f_Q(a_2)$ in members of $\pmb{\Delta}(F^{\ADF})$. 
	Other $f_Q(a_2)$ also appear in other members of 
	$\pmb{\Delta}(F^{\ADF})$. 
	Specifically, due to Lemma \ref{lem_label_designation_subsumption} 
	and Definition \ref{def_abstraction_from}, 
	$F \in \pmb{\Delta}(F^{\ADF})$ 
        can come with 
	any $f_Q(a_2) = ((n_1^{'a_2}, 3), (m_1^{'a_2}, m_2^{'a_2}))$ 
	with $0 \leq n_1^{'a_2} \leq 1$,  $0 \leq m_1^{'a_2} \leq 1$, 
	$1 \leq m_2^{'a_2} \leq 3$. 
	Analogously for $f_Q(a_1)$ and $f_Q(a_3)$.  
	\hfill$\clubsuit$ 
\end{example} 
Until this point, we left the cardinality of $\pmb{\Gamma}(F)$ and 
$\pmb{\Delta}(F^{\ADF})$  all up to intuition. The following 
theorem establishes the bounds, with an implication of 
existence of the two sets (Corollary \ref{cor_existence}).  
\begin{theorem}[Cardinality of the maps]  
	For any $F \equiv (A, R, f_Q)\ (\in \mathcal{F})$,  
	we have 
	$1 \leq \pmb{\Gamma}(F) \leq |\mathcal{L}|^{|A|(|\mathcal{L}|^{|A|})}$, 
	and for 
	any $F^{\ADF} \equiv (A, R, C)\ (\in {\mathcal{F}^{\ADF}})$,  
	we have $1 \leq \pmb{\Delta}(F^{\ADF})
	\leq (\frac{(|A|+2)(|A|+3)}{2})^{2|A|}$. 
	It holds that 
	$(\frac{(|A|+2)(|A|+3)}{2})^{2|A|} \ll 
	|\mathcal{L}|^{|A|(|\mathcal{L}|^{|A|})}$ for any $2 \leq |A|$ 
	and $|\mathcal{L}|=3$. 
\end{theorem}  
\begin{proof} 
	For the former, the lower bound is due to the fact that any 
	$\lambda \in \Lambda^{\pre^F(a)}$ for $a \in A$ 
	designates at least one $l \in \mathcal{L}$ for $a$ in $F$.  
	For the upper bound, $|\pre^F(a)| \leq |A|$, and thus 
	$|\Lambda^{\pre^{F}(a)}| \leq |\mathcal{L}|^{|A|}$. 
	Any $\lambda \in \Lambda^{\pre^F(a)}$  
	may designate as many labels as there are in $\mathcal{L}$  
	for $a$, hence there may be up to $|\mathcal{L}|^{(|\mathcal{L}|^{|A|})}$
	alternatives for the third component of 
	$F^{\ADF} \in \pmb{\Gamma}(F)$. There are $|A|$ arguments. 
	Put together, we obtain the result.  
	For the latter,  
	the lower bound is due to the fact that 
	 $f_Q(a) = ((0, |\pre^{F}(a)|+1),\linebreak (0, |\pre^{F}(a)|+1))$  
	 designates each of $\inL, \outL$ and $\undecL$. 
        For the upper bound, 
	for each $a \in A$,  
	$C_a$ 
	 may map to 
	member(s) of 
	$X \equiv \{((n_1,n_2), (m_1,m_2)) \mid 0 \leq n_1, n_2, m_1, m_2\leq |\pre^{F}(a)| + 1\}$. 
Since we have $|\pre^{F}(a)| + 1 \leq |A|+1$, it holds that 
	$|X| \leq (\frac{(|A|+2)(|A|+3)}{2})^2$. There are $|A|$ arguments.  
		 \hfill$\Box$ 
\end{proof}  
\begin{corollary}[Existence]\label{cor_existence} 
	For any $F \in \mathcal{F}$, 
	there exists some concretisation 
	of $F$, and 
	for any $F^{\ADF} \in \mathcal{F}^{\ADF}$, 
	there exists some abstraction of 
	$F^{\ADF}$. 
\end{corollary} 
While $\pmb{\Delta}(F^{\ADF})$ 
is still rather large, note that it covers every possible abstraction
of $F^{\ADF}$. In practice, 
there should be certain criteria for abstraction  
e.g. minimal abstraction (see $f_{\alpha}$ in Theorem \ref{thm_galois_connection}), 
which leaves only a handful of the set, or just one 
in case the minimum abstraction exists. 
 
Any abstraction of $F^{\ADF} \in {\mathcal{F}^{\ADF}}$ 
correctly overapproximates $C_a$'s label 
designation for every $a \in A$, trivially from 
Definition \ref{def_label_designation} and 
Definition \ref{def_abstraction_from}. 
\begin{theorem}[Abstraction soundness]\label{thm_abstraction_soundness} 
	For any $F^{\ADF} \equiv (A, R, C)\ (\in {\mathcal{F}^{\ADF}})$ and 
	any $F \in \mathcal{F}$, 
	if $F$ is an abstraction of $F^{\ADF}$, 
	then for any $a \in A$, any $l \in \mathcal{L}$
	and any $\lambda \in \Lambda^{\pre^{F}(a)}$,  
		if $l \in 
		\bigcup_{(A, R, C') \in \mathcal{F}_x^{\ADF}}\{C'_a(\lambda)\}
		$, then 
			 $\lambda$ designates $l$ for $a$ in $F$. 
\end{theorem} 
\subsubsection{Galois connection}  
When there are two systems related in concrete-abstract relation, 
it is of interest to establish Galois connection (see 
any standard text, e.g. \cite{Davey02}). 
Briefly, let $S_1$ and $S_2$ (each) be an ordered set, 
partially ordered in $\leq_1$ and $\leq_2$ resp.. 
Let $f_{1\rightarrow 2}$ be a function that maps 
each element of $S_1$ onto an element of $S_2$, 
and let $f_{2 \rightarrow 1}$ be one that maps 
each element of $S_2$ onto an element of $S_1$.  
If $f_{1 \rightarrow 2}(s_1) \leq_2 s_2$ implies 
$s_1 \leq_1 f_{2\rightarrow 1}(s_2)$ and vice versa, then 
the pair of $f_{1 \rightarrow 2}$ and $f_{2 \rightarrow 1}$ is said to be a 
Galois connection. The following properties hold good. A Galois connection is: 
contractive, i.e. $(f_{1 \rightarrow 2} \circ f_{2\rightarrow 1})(s_2) \leq_2 s_2$ for every $s_2 \in S_2$; extensive, i.e. $s_1 \leq_1 (f_{2 \rightarrow 1}
\circ f_{1 \rightarrow 2})(s_1)$ for every $s_1 \in S_1$;  and 
monotone for both 
$f_{1 \rightarrow 2}$ and $f_{2 \rightarrow 1}$ (to follow from the contractiveness and the extensiveness). Further, it holds that 
$f_{1 \rightarrow 2} \circ f_{2 \rightarrow 1} \circ f_{1 \rightarrow 2} = 
f_{1 \rightarrow 2}$ and that 
$f_{2 \rightarrow 1} \circ f_{1 \rightarrow 2} \circ f_{2 \rightarrow 1} = 
f_{2 \rightarrow 1}$. 
Galois connection is widely used for 
abstraction interpretation 
in program analysis \cite{Cousot77,Cousot79,Nielson99} 
for verification of properties of 
large-scale programs as the verification in concrete space
is often undecidable or very costly. 
In our view, it is no different with formal argumentation; 
reasoning  
about a large-scale argumentation, with a large cost, 
will benefit from utilising the technique. 
Here, we identify a Galois connection 
between $\mathcal{F}$ and $\mathcal{F}^{\ADF}$ 
based on $\Gamma$ and $\Delta$ that we introduced. 
\begin{theorem}[Galois connection] \label{thm_galois_connection}  
	Let $2^{\mathcal{F}}_{(A, R)}$  
	be a subclass of $2^{\mathcal{F}}$ 
	which contains every $F \equiv (A, R, f_Q)\ (\in \mathcal{F})$ 
	for some $f_Q$ but nothing else.
	Let $2^{\mathcal{F}^{\ADF}}_{(A, R)}$  
	be a subclass of $2^{\mathcal{F}^{\ADF}}$ 
	which contains every $F^{\ADF} \equiv (A, R, C)\ (\in \mathcal{F}^{\ADF})$ 
	for some $C$ but nothing else.

	Let $f_{\gamma}: 2^{\mathcal{F}} \rightarrow 2^{\mathcal{F}^{\ADF}}$ 
	and  $f_{\alpha}: 2^{\mathcal{F}^{\ADF}} \rightarrow 2^{\mathcal{F}}$ 
        be the following.
	\begin{itemize} 
		\item for any $\mathcal{F}_x \in 2^{\mathcal{F}}_{(A. R)}$, 
	$f_{\gamma}(\mathcal{F}_x) = 
	\bigcup_{F_x \in \mathcal{F}_x}\pmb{\Gamma}(F_x)$.  
\item  for any  $\mathcal{F}^{\ADF}_x \in 2^{\mathcal{F}^{\ADF}}_{(A, R)}$,
	 $f_{\alpha}(\mathcal{F}^{\ADF}_x) =  
			\bigcup_{F_x^{\ADF} \in \mathcal{F}_x^{\ADF}}\{F \in \pmb{\Delta}(F_x^{\ADF}) 
			\mid \forall F' \in \pmb{\Delta}(F_x^{\ADF}).
			F' \not\sqsubset F\}$.  
	\end{itemize} 
	 Then 
	 $(f_{\alpha}, f_{\gamma})$ is a Galois connection for 
	 any 
	 $(2^{\mathcal{F}}_{(A, R)}, \sqsubseteq)$ and 
	 $(2^{\mathcal{F}^{\ADF}}_{(A,R)}, \subseteq)$. 
\end{theorem}
\begin{proof} 
	Suppose $\mathcal{F}_x \in 2^{\mathcal{F}}_{(A, R)}$ and 
	$\mathcal{F}^{\ADF}_x
	 \in 2^{\mathcal{F}^{\ADF}}_{(A, R)}$. Suppose  
	$f_{\alpha}(\mathcal{F}^{\ADF}_x) \sqsubseteq \mathcal{F}_x$, then 
	we have to show that $\mathcal{F}^{\ADF}_x \subseteq  
	f_{\gamma}(\mathcal{F}_x)$. By Lemma \ref{lem_monotonicity_1}, 
	we have $\bigcup_{F_y \in f_{\alpha}(\mathcal{F}^{\ADF}_x)}\pmb{\Gamma}(F_y) 
	\subseteq 
	\bigcup_{F_x \in \mathcal{F}_x}\pmb{\Gamma}(F_x)$. 
	By the definition of $f_{\gamma}$, 
	we have $\bigcup_{F_x \in \mathcal{F}_x}\pmb{\Gamma}(F_x) \subseteq  
	f_{\gamma}(\mathcal{F}_x)$. By \mbox{Theorem \ref{thm_abstraction_soundness}} 
	and \mbox{Definition \ref{def_concretisation_from}}, 
	we have $\mathcal{F}_x^{\ADF} \subseteq 
	\bigcup_{F_y \in f_{\alpha}(\mathcal{F}^{\ADF}_x)}\pmb{\Gamma}(F_y)$. 
	Hence $\mathcal{F}_x^{\ADF} \subseteq f_{\gamma}(\mathcal{F}_x)$, as required.

	Suppose on the other hand that 
	$\mathcal{F}^{\ADF}_x \subseteq f_{\gamma}(\mathcal{F}_x)$, then 
	we have to show that 
	$f_{\alpha}(\mathcal{F}^{\ADF}_x) \sqsubseteq 
	\mathcal{F}_x$. First, trivially,  
	   we have $f_{\alpha}(\mathcal{F}_x^{\ADF}) \sqsubseteq 
	  f_{\alpha}(f_{\gamma}(\mathcal{F}_x))$, since 
	  every $F^{\ADF} \in \mathcal{F}_x^{\ADF}$ 
	  is contained in $f_{\gamma}(\mathcal{F}_x)$ by the present assumption. 
	  By Lemma \ref{lem_label_designation_subsumption} (together with 
	  the reasoning in the proof of 
	  Lemma \ref{lem_monotonicity_1}), we have $f_{\alpha}(f_{\gamma}(\mathcal{F}_x)) \sqsubseteq \mathcal{F}_x$, thus 
	  we have $f_{\alpha}(\mathcal{F}_x^{\ADF}) \sqsubseteq \mathcal{F}_x$, 
	  as required. \hfill$\Box$ 
\end{proof} 
While, in fact, the $f_{\alpha}$ in this theorem 
does not realy have to select any particular member(s) of 
$\pmb{\Delta}(F_x^{\ADF})$, 
it has, as we stated earlier, a benefit of minimal abstraction. 

\subsection{Semantics}  

For the semantics, 
let $\maxi: \mathcal{A} \times  
\Lambda \rightarrow 2^{\Lambda}$, 
which we alternatively 
state $\maxi^{\mathcal{A}}: 
\Lambda \rightarrow 2^{\Lambda}$, 
be such that, 
for any $A \subseteq_{\textsf{fin}} 
\mathcal{A}$ and any 
$\lambda \in \Lambda^A$, we have: \\
\indent $\maxi^A(\lambda) = 
\{\lambda_x \in \Lambda^A \ | \ 
  \lambda \preceq \lambda_x \text{ and } 
\lambda_x \text{ is maximal in } 
(\Lambda^A, \preceq)\}$.\\ 
Every member of $\maxi^A(\lambda)$ 
is such that $\lambda(a) \in \{\inL, \outL\}$ 
for every $a \in A$. \\ 

Also,   
let $\Gamma 
: \mathcal{F}^{\ADF} \times \Lambda 
\rightarrow \Lambda$, 
which we alternatively 
state $\Gamma^{\mathcal{F}^{\ADF}}: 
\Lambda \rightarrow 
\Lambda$, 
be such that, 
for any $F^{\ADF} \equiv (A, R, C)\  (\in \mathcal{F}^{\ADF})$
and any $\lambda \in \Lambda^A$, 
$\Gamma^{F^{\ADF}}(\lambda)$ satisfies all
the following.
\begin{enumerate} 
  \item 
	  $\Gamma^{F^{\ADF}}(\lambda) 
     \in \Lambda^A$. 
   \item For every $a \in A$, 
	   $\Gamma^{F^{\ADF}}(\lambda)(a) = \inL$ iff, 
     for every 
     $\lambda_x \in \maxi^A(\lambda)$, 
		$C_{a}({\lambda_x}_{\downarrow \pre^{F^{\ADF}}(a)}) = \inL$.  
\end{enumerate}

In a nutshell \cite{Brewka13}, $\Gamma^{F^{\ADF}}(\lambda) 
$ 
gets a consensus of every $\lambda_x \in \maxi^A(\lambda)$ 
on the label of each $a \in A$: if each one of them 
says $\inL$ for $a$, then $\Gamma^{F^{\ADF}}(\lambda)(a) = 
\inL$, 
if each one of them says $\outL$ for $a$, then $\Gamma^{F^{\ADF}}(\lambda)(a) = \outL$,
and for the other cases $\Gamma^{F^{\ADF}}(\lambda)(a) = \undecL$. 

Then the grounded semantics of $F^{\ADF} \equiv (A, R, C)\ (\in 
\mathcal{F}^{\ADF})$ contains just the least fixpoint of $\Gamma^{F^{\ADF}}$
(the order is $\preceq$); 
the complete semantics of $F^{\ADF}$ 
contains all and only fixpoints of $\Gamma^{F^{\ADF}}$; and 
the others are defined in a usual way from the complete 
semantics. \\

\noindent Some differences are therefore easy to identify. 
For instance, the grounded semantics 
of $F \in \mathcal{F}$ may not be a subset of 
its complete semantics (Cf. Theorem \ref{thm_subsumption}), 
while $\ADF$ enforces the property via 
$\Gamma$. In view of recent 
studies \cite{ArisakaSantini19,Bistarelli17} where 
the subsumption does not hold, we believe 
our characterisation can be accommodating. 
For where a difference occurs, 
we may consider the following member $F$ of $\mathcal{F}$ 
\begin{tikzcd}  
   \underset{((0,0), (1,1))}{\ensuremath{a_1}} 
  \rar & \underset{((0,0), (1,1))}{\ensuremath{a_2}} \lar
\end{tikzcd}  
 for which there are the following 4 labellings 
to constitute the complete semantics of $F$:  
\begin{center} 
\begin{tikzcd}  
   \overset{\inL}{\underset{((0,0), (1,1))}{\ensuremath{a_1}}} 
  \rar & \overset{\undecL}{\underset{((0,0), (1,1))}{\ensuremath{a_2}}} \lar
\end{tikzcd}  
\begin{tikzcd}  
   \overset{\undecL}{\underset{((0,0), (1,1))}{\ensuremath{a_1}}} 
  \rar & \overset{\inL}{\underset{((0,0), (1,1))}{\ensuremath{a_2}}} \lar
\end{tikzcd}\\   
\begin{tikzcd}  
   \overset{\inL}{\underset{((0,0), (1,1))}{\ensuremath{a_1}}} 
  \rar & \overset{\outL}{\underset{((0,0), (1,1))}{\ensuremath{a_2}}} \lar
\end{tikzcd}  
\begin{tikzcd}  
   \overset{\outL}{\underset{((0,0), (1,1))}{\ensuremath{a_1}}} 
  \rar & \overset{\inL}{\underset{((0,0), (1,1))}{\ensuremath{a_2}}} \lar
\end{tikzcd} 
\end{center} 
Any of these labellings are maximally proper, 
but in fact both of the arguments' labels are 
designated under them (that is, they are ideal; Cf. Definition \ref{def_ideal_labelling}), which is a condition 
consistent with that of 
a complete labelling for a Dung abstract argumentation 
(see Section \ref{section_technical_preliminaries}).  
Meanwhile, $\lambda \in \Lambda^{\{a_1, a_2\}}$ 
with $\lambda(a_1) = \lambda(a_2) = \undecL$  
is not maximally proper, since we have
two converging update sequences: $(\lambda, \lambda_x)$,
and $(\lambda, \lambda_y)$, 
with: $\lambda_x(a_1) = \lambda_y(a_2) = \undecL$ 
and $\lambda_x(a_2) = \lambda_y(a_1) = \inL$. 

A corresponding {\ADF} tuple in the meantime has 
the following 
$C$. For $\lambda_1, \lambda_2, \lambda_3 
\in \Lambda^{\{a_2\}}$ with:  
$\lambda_1(a_2) = \inL$; 
$\lambda_2(a_2) = \outL$; and 
$\lambda_3(a_2) = \undecL$, 
and $\lambda_4, \lambda_5, \lambda_6 \in 
\Lambda^{\{a_1\}}$ with:   
$\lambda_4(a_1) = \inL$; 
$\lambda_5(a_1) = \outL$; 
and $\lambda_6(a_1) = \undecL$, all the following hold. 
\begin{itemize} 
\item 
$C_{a_1}(\lambda_1) = \undecL$, 
$C_{a_1}(\lambda_2) = \inL$, and 
$C_{a_1}(\lambda_3) = \inL$. 
\item 
$C_{a_2}(\lambda_4) = \undecL$, 
$C_{a_2}(\lambda_5) = \inL$, 
and $C_{a_2}(\lambda_6) = \inL$.  
\end{itemize}  

With $\Gamma$, this $\ADF$ tuple obtains 
the following 3 labellings for its 
complete semantics.   
\begin{center} 
\begin{tikzcd}  
   \overset{\undecL}{\ensuremath{a_1}} 
  \rar & \overset{\undecL}{\ensuremath{a_2}} \lar
\end{tikzcd}  
\begin{tikzcd}  
\overset{\outL}{\ensuremath{a_1}} 
  \rar & \overset{\inL}{\ensuremath{a_2}} \lar
\end{tikzcd} 
\begin{tikzcd}  
   \overset{\inL}{\ensuremath{a_1}} 
  \rar & \overset{\outL}{\ensuremath{a_2}} \lar
\end{tikzcd}  
\end{center} 
Thus, {\ADF} weighs more the semi-lattice formation 
in its complete semantics than the labelling specification 
given by $C$. In addition to this difference, any $F \equiv (A, R, f_Q) 
\ (\in \mathcal{F})$, 
which we also refer to by $F$, 
if a semantics $\Lambda \subseteq \Lambda^A$ 
of $F$ is such that, for two distinct 
$\lambda_1, \lambda_2 \in \Lambda$, 
there exists some $a \in A$ such that (1) 
$\lambda_1(a) \not= \lambda_2(a)$, and that 
(2) $\lambda_1(a_x) = \lambda_2(a_x)$ for every $a_x \in 
\pre(a)$, then $\Lambda$ does not belong
to a semantics of $(A, R, C)$ (or $(A, R^a, C)$ 
with $R \equiv R^a$) above. This is rather clear, 
because, for any  $a \in A$, $C_a$ is a function 
with its range of $\{\inL, \outL, \undecL\}$.  
An example of such $\Lambda$ has been covered; 
see Example \ref{ex_semantics}.  \\

\subsubsection{Abstract interpretation (semantics)} 
Whatever the semantics, it is of interest to 
conduct abstract interpretation (reasoning 
of properties in concrete space from within abstract space) 
\cite{Cousot77,Cousot79,Nielson99}. For that purpose, 
one semantics for $\mathcal{F}^{\ADF}$ and another 
semantics for $\mathcal{F}$ is not convenient. Seeing 
perhaps {\ADF} community may be more intersted in 
the {\ADF} semantics, we first define the semantics 
for $\mathcal{F}$, and then show what can be said  
concerning abstract interpretation.  

Let $\Gamma: \mathcal{F} \times \Lambda \rightarrow \Lambda$, 
alternatively $\Gamma^{\mathcal{F}}: \Lambda \rightarrow \Lambda$, 
be such that, for any $F \equiv (A, R, f_Q)\ (\in \mathcal{F})$ 
and any $\lambda \in \Lambda^A$, $\Gamma^F(\lambda)$ 
satisfies all the following. 
\begin{enumerate}
	\item $\Gamma^F(\lambda) \in \Lambda^A$. 
	\item For every $a \in A$, $\Gamma^{F^{\ADF}}(\lambda)(a) 
		\not= \undecL$ iff  
		$\lambda$ designates at most one $l \in \{\inL,\outL\}$ 
		$\andC$ $\Gamma^{F^{\ADF}}(\lambda)(a) = l$. 
\end{enumerate} 

\begin{theorem}[Abstraction soundness (semantics)] 
	For any $F^{\ADF} \equiv (A, R, C) \ (\in \mathcal{F}^{\ADF})$, any 
	$F \in f_{\alpha}(\{F^{\ADF}\})$ and 
	any $\lambda \in \Lambda^A$,  
	if $\lambda = \Gamma^F(\lambda)$, 
	then there is some $\lambda' \in \Lambda^A$  
	such that $\lambda' = \Gamma^{F^{\ADF}}(\lambda')$  
	and that,   
	for any $a \in A$, 
	if $\lambda(a) \not= \undecL$, then 
	$\lambda(a) = \lambda'(a)$. 
\end{theorem} 
\begin{proof} 
   
\end{proof} 
\begin{corollary}[Acceptance and rejection] 
       Any argument that is credulously accepted / rejected 
	in abstract space is credulously accepted / rejected 
	in concrete space. Any argument that is skeptically 
	accepted / rejected in abstract space is 
	credulously accepted / rejected in concrete space. 
\end{corollary}

\noindent Straightforwardly, this result can be interpolated to other semantics 
including the maximally proper semantics of this paper.   

\hide{ 
\section{Existence Theorem}\label{section_existence_theorem} 
There is one question of considerable formal interest which 
has been postponed till now, as regards 
a condition to ensure
the existence of a labelling under which 
every argument's label is designated.    
In this section, we show such a condition, 
{\it in fact the exact condition} 
that sharply 
divides the existence and non-existence 
of such a labelling 
for any arbitrary member of $\mathcal{F}$.  
Specifically, we prove the following: 
for $(A, R, f_Q) \in \mathcal{F}$, 
only so long as $a \in A$ either cannot 
simultaneously satisfy must- acceptance condition 
or else cannot simultaneously 
satisfy must- rejection condition 
for every $\lambda \in \Lambda^A$, 
is it perforce the case that there exists some 
$\lambda_x \in \Lambda^A$ satisfying the claimed 
condition. 
 
We assume the following which we make use of 
in the proof. 
\begin{definition}[Transformations]
    Let $\beta: \mathcal{F} \times \mathcal{A} 
   \rightarrow \mathcal{F}$ be such that, 
    for any $F \equiv (A, R, f_Q)\ (\in \mathcal{F})$ 
    and any $a \in A$, 
    $\beta(F, a) = 
     (A_x, R_x, {f_{Q}}_x)$ with: 
     \begin{itemize} 
        \item $A_x = A \backslash \{a\}$; 
        \item $R_x = \{(a_x, a_y) \in R \ | \  
                         a \not\in \{a_x, a_y\}\}$; and 
        \item ${f_Q}_x$ being a function to satisfy all the following. 
            \begin{enumerate} 
                \item For every $a_x \not\in (\suc(a) \cup \{a\})$, 
                    it holds that ${f_Q}_x(a_x) =   
       f_Q(a_x)$.   
                \item For every $a_x \in (\suc(a) \backslash 
                   \{a\})$, 
                  \begin{itemize}
                      \item If $a$ satisfies 
    must- acceptance condition and not-may rejection condition,
     then both of the following conditions hold.
      \begin{enumerate}
       \item $({f_Q}_x(a_x))^1 = (f_Q(a_x))^1$. 
       \item 
          ${({f_Q}_x(a_x))^2 = (\max(0, ((f_Q(a_x))^2)^1 - 1),
     \max(0, ((f_Q(a_x))^2)^2 - 1))}$. 
      \end{enumerate} 
  \item Else if $a$ satisfies
    must- rejection condition and not-may acceptance condition,
     then both of the following conditions hold. 
     \begin{enumerate}
                             \item  
${({f_Q}_x(a_x))^1 = (\max(0, ((f_Q(a_x))^1)^1 - 1),
     \max(0, ((f_Q(a_x))^1)^2 - 1))}$. 
     \item $({f_Q}_x(a_x))^2 = (f_Q(a_x))^2$. 
    \end{enumerate} 
   \item Otherwise, 
                        ${f_Q}_x(a_x) = f_Q(a_x)$.
                  \end{itemize}
           \end{enumerate} 
     \end{itemize}
   
\end{definition}

\begin{definition}[Transformations2] 
    Let $\beta: \mathcal{F} \times \mathcal{A}  
      \times \Lambda 
   \rightarrow \mathcal{F}$ be such that, 
    for any $F \equiv (A, R, f_Q)\ (\in \mathcal{F})$, 
     any $a \in A$ and any $\lambda \in \Lambda^A$,   
    denoting the set of all $a_x \in A$ 
    such that $(a, a_x) \in R$ by 
    $\suc(a)$, 
    we have  
    $\beta(F, a, \lambda) = 
     (A_x, R_x, {f_{Q}}_x)$ with: 
     \begin{itemize} 
        \item $A_x = A \backslash \{a\}$; 
        \item $R_x = \{(a_x, a_y) \in R \ | \  
                         a \not\in \{a_x, a_y\}\}$; and 
        \item ${f_Q}_x$ being a function to 
             satisfy all the following. 
            \begin{enumerate} 
                \item For every $a_x \not\in (\suc(a) \cup \{a\})$, 
                    it holds that ${f_Q}_x(a_x) =   
       f_Q(a_x)$.   
                \item For every $a_x \in (\suc(a) \backslash 
                   \{a\})$, 
                  \begin{itemize}
                      \item If $\lambda(a) = \inL$,  
     then both of the following hold.
      \begin{enumerate}
       \item $({f_Q}_x(a_x))^1 = (f_Q(a_x))^1$. 
       \item 
          ${({f_Q}_x(a_x))^2 = (\max(0, ((f_Q(a_x))^2)^1 - 1),
     \max(0, ((f_Q(a_x))^2)^2 - 1))}$. 
      \end{enumerate} 
  \item Else if $\lambda(a) = \outL$, 
     then both of the following conditions hold. 
     \begin{enumerate}
                             \item  
${({f_Q}_x(a_x))^1 = (\max(0, ((f_Q(a_x))^1)^1 - 1),
     \max(0, ((f_Q(a_x))^1)^2 - 1))}$. 
     \item $({f_Q}_x(a_x))^2 = (f_Q(a_x))^2$. 
    \end{enumerate} 
   \item Otherwise, 
                        ${f_Q}_x(a_x) = f_Q(a_x)$.
                  \end{itemize}
           \end{enumerate} 
     \end{itemize}
\end{definition}

\ryuta{I should  
illustrate the impact of label change with examples.}  

\begin{lemma}[Reduction lemma]\label{lem_reduction_lemma} 
    Let $\nu: \Lambda \times 2^{\mathcal{A}} \rightarrow \Lambda$ 
    be such that $\nu(\lambda, A) \in \Lambda^A$
    with $(\nu(\lambda, A))(a_y) = \lambda(a_y)$ 
    for every $a_y \in A$. 
    
    For any $F \equiv (A, R, f_Q)\ (\in \mathcal{F})$, 
    any $a \in A$ and any $\lambda \in \Lambda^A$,  
    let $F_x \equiv (A_x, R_x, {f_Q}_x)$ denote 
    $\beta(F, a, \lambda)$. Then, 
    for every $a_x \in F_x$, both of the following hold. 
    \begin{enumerate} 
     \item 
     $a_x$ satisfies must-, may$_s$-, or 
     not-may- acceptance condition under 
     $\lambda$ (in $F$) iff it satisfies the same 
     acceptance condition under $\nu(\lambda, A_x)$ 
    (in $F_x$).   
     \item  $a_x$ satisfies must-, may$_s$-, or 
     not-may- rejection condition under 
     $\lambda$ (in $F$) iff it satisfies the same 
     rejection condition under $\nu(\lambda, A_x)$ 
    (in $F_x$). 
    \end{enumerate} 
\end{lemma}  
\begin{proof}   
    It suffices to show for 1., and the proof 
    for 2. is analogous. 
  
    \textbf{If}: by assumption, 
          $a_x$ satisfies must-, may$_s$-, or not-may- 
    acceptance condition
      under $\nu(\lambda,A_x)$ (in $F_x$).  
      We have to show 
    that $a_x \in A_x$ satisfies the same 
    acceptance condition 
     under $\lambda$ (in $F$).       
    We consider cases. 
      \begin{enumerate} 
        \item $a_x \not\in (\suc(a) \cup \{a\})$: 
             then $\pre^F(a_x) \subseteq 
             A_x \cap A\ (= A_x)$. Also, by the definition of $\nu$, 
            it holds that 
           $(\nu(\lambda, A_x))(a_y) = \lambda(a_y)$ 
           for every $a_y \in \pre^F(a_x)$.  
            Meanwhile, by the definition of $\beta$, 
           we have: ${f_Q}_x(a_x) = f_Q(a_x)$. 
           Consequently, $a_x$ satisfies the same 
           acceptance condition under $\lambda$ 
           (in $F$).  
        \item $a_x \in (\suc(a) \backslash \{a\})$: 
          then  $\pre^F(a_x) = (\pre^{F_x}(a_x) \cup \{a\})$.  
           We consider all possibilities of the value of 
     $\lambda(a)$. 
     \begin{enumerate} 
       \item $\lambda(a) = \inL$: by the definition 
       of $\beta$, we have that 
       $({f_Q}_x(a_x))^1 = (f_Q(a_x))^1$. 
       Suppose there are $0 \leq n \leq |\pre^F(a_x)|$ 
       distinct members 
       in $\pre^F(a_x)$, $a_1, \ldots, a_n$ (if $n \not= 0$), 
       with $\lambda(a_i) = \outL$ for every $1 \leq i \leq n$ 
       (if $n \not= 0$). 
       Since $\lambda$ is a function with the range 
       of $\mathcal{L}$, it follows that   
       $a \not= a_i$ for each $1 \leq i \leq n$. 
       Thus, it follows that $a_x$ satisfies
       the same acceptance condition 
       in $F$, as required. 
       
       \item $\lambda(a) = \outL$:         if there are  
     $0 \leq n \leq |\pre^{F_x}(a_x)|$ 
       distinct members 
       in $\pre^{F_x}(a_x)$, $a_1, \ldots, a_n$ (if $n \not= 0$), 
       with $\lambda(a_i) = \outL$ for every $1 \leq i \leq n$ 
       (if $n \not= 0$), then 
      there exists $n+1$ distinct members 
      in $\pre^F(a_x)$ with $\lambda(a_i) = \outL$ for 
     every $1 \leq i \leq n$ (if $n \not= 0$).

      {\ }\indent {\ }\indent {\ }\indent {\ }\indent However, 
       by the definition 
       of $\beta$, 
       $({f_Q}_x(a_x))^1 = ({\max(0, ((f_Q(a_x))^1)^1 - 
1)},\linebreak
     {\max(0, ((f_Q(a_x))^1)^2 - 1)})$. That is,
      if it should be that $(({f_Q}_x(a_x))^1)^1 \leq n$,
      it is also the case that $(({f_Q}(a_x))^1)^1 + 1 \leq n+1$,
      and, similarly, if 
      it should be that $(({f_Q}_x(a_x))^1)^2 \leq n$,
      it is also the case that $(({f_Q}(a_x))^1)^1 + 1 \leq n+1$.
      Consequently, $a_x$ satisfies the same 
      acceptance condition under $\lambda$ in $F$, as required. 
    
     \item $\lambda(a) = \undecL$: Similar to (a). 
      
     \end{enumerate}
     \item By the definition of $\beta$, $a_x = a$ 
       is not possible.  
              
      \end{enumerate}  
     
     \textbf{Only if}: by assumption,
      $a_x$ satisfies must-, may$_s$-, or not- acceptance 
     condition under $\lambda$ (in $F$). We have 
     to show that $a_x \in A_x$, if $a_x \not= a$,
     satisfies the same 
     acceptance condition under $\nu(\lambda, A_x)$
     (in $F_x$). Proof is again by cases on $a_x$, 
     which is analogous to the \textbf{If} direction 
     we have covered.  \hfill$\Box$ 
\end{proof} 

\ryuta{Maybe this is the only one that I need in 
the proof below.} 
\begin{theorem}[Designation preservation by addition]
    For any $F_x \equiv (A_x, R_x, {f_Q}_x) \ (\in \mathcal{F})$,
    any $F \equiv (A, R, f_Q)\ (\in \mathcal{F})$ and 
    any $\lambda_x \in \Lambda^{A_x}$,   
    if all the following conditions hold, then 
    there exists some $\lambda \in \Lambda^A$ 
    under which every $a \in A$ is designated. 
    
    \begin{enumerate} 
        \item There exists some $a_x \in A$
         such that $F_x = \beta(F, a_x, \lambda_x)$.  
        \item For every $a \in A_x$,  
         $a$'s label is designated under 
        $\lambda_x$. 
        \item  
    \end{enumerate}
    if every $a \in A$ is designated under $\lambda$, 
    then we have:   

Satisfiability of each condition 
    is preserved. 
\end{theorem} 
\begin{proof} 
  \ryuta{Do it later.} 
\end{proof} 
\begin{definition}[Regular arguments]\label{def_regular_arguments}
     For any $F \equiv (A, R, f_Q)\ (\in \mathcal{F})$ 
    and any $a \in A$, we say that
    $a$ is  regular iff, 
    for every $\lambda \in \Lambda^A$, 
    $a$ does not satisfy both must- acceptance 
    condition and must- rejection condition simultaneously. 
\end{definition} 

\begin{theorem}[Labelling existence theorem] 
    For any $F \equiv (A, R, f_Q)\ (\in \mathcal{F})$,   
    if every $a \in A$ is regular, then 
    there exists some $\lambda \in \Lambda^A$ under which, 
    for every $a \in A$, $a$'s label is designated. 
\end{theorem} 
\begin{proof}  
    By induction on $\max_{a \in A}|\Delta(F, a)|$ 
    (the maximum number of arguments in a strongly
    connected component) and a sub-induction
    on $\max_{a \in A}\delta(F, \Delta(F, a))$ (the maximum
    SCC depth).
    
    For the base case of the main induction,
    assume $\Delta(F, a_1) = \{a_1\}$ 
    for every $a_1 \in A$.  
    
    For the base case of the sub-induction
    for the base case of the main induction,
    assume that $\max_{a \in A}\delta(F, \Delta(F, a)) = 0$. 
    Then, for every $a_1, a_2 \in A$,
    there is no attack between them. Hence, 
    it suffices to consider cases for an arbitrary 
    $a \in A$. 
     Assume $(f_Q(a))^1 = (l_1, l_2)$, 
     $(f_Q(a))^2 = (l_3, l_4)$, and 
     some $\lambda_x \in \Lambda^A$ under which 
     every $a_y \in (A \backslash \{a\})$'s 
     label is designated, 
    and consider 
      all possibilities of what acceptance and rejection
     conditions $a$ satisfy under $\lambda_x$. 
    \begin{enumerate} 
       \item $a$ satisfies must- acceptance condition
      and not- rejection condition under $\lambda_x$.  
       Consider all possibilities of $\lambda_x(a)$. 
           \begin{enumerate} 
               \item $\lambda_x(a) = \inL$: then we have 
                nothing to show. $\lambda_x$ is such 
                 $\lambda$.   
               \item $\lambda_x(a) = \outL$: 
                  Consider all possible $0 \leq m \leq 
                  |\pre(a)|$. 
               \begin{enumerate} 
                   \item $m = 0$: then $(a, a) \not\in R$.   
             $\lambda' \in \Lambda^A$ with: 
          $\lambda'(a) = \inL$ and 
           $\lambda'(a_y) = \lambda_x(a_y)$ for 
           every $a_y \in (A \backslash \{a\})$ 
is such $\lambda$.  
                   
              \item $m = 1$: then $(a, a) \in R$ 
              and $l_2 \in \{0, 1\}$. 
                     Consider all possibilities of $l_2$: 
                     \begin{enumerate} 

                         \item $l_2 = 1$:  
                    Then, we must have $1 \leq l_3, l_4$.  
                    Assume $\lambda' \in \Lambda^A$ 
                with $\lambda'(a) = \undecL$ and 
                $\lambda'(a_y) = \lambda_x(a_y)$ for 
                every $a_y \in (A \backslash \{a\})$. 
                Then, 
                 $a$ either satisfies 
                 may$_s$- acceptance condition (when 
                $l_1 = 0$) or else not- acceptance 
               condition (when $l_1 = 1$). Meanwhile,
            $a$ satisfies not- rejection condition. 
                 Hence, regardless of $l_1$'s value 
              (Cf. Fig. \ref{fig_table}), it follows 
                  that $a$'s label is designated under
          $\lambda'$, which is such $\lambda$. 
                       
                       \item $l_1 = l_2 = 0$: then 
                we must have $1 \leq l_3, l_4$.  
                 This forms a sub-case of the previous case.
                    \end{enumerate}
                                 \end{enumerate}

\item $\lambda_x(a) = \undecL$:   
           Consider all possible $0 \leq m \leq |\pre(a)|$. 
                \begin{enumerate} 
                   \item $m = 0$: then $(a,a) \not\in R$.
                $\lambda' \in \Lambda^A$ with $\lambda'(a) 
    =  \inL$ is such $\lambda$. 
                 
                   \item $m = 1$: then $(a, a) \in R$.  
                   It holds that $l_1 = l_2 = 0$. 
                    We must have 
            $1 \leq l_3$ and $2 \leq l_4$ due to 
          the given conditions. This case has been covered 
         above.  $\lambda' \in \Lambda^A$ with 
    $\lambda'(a) = \inL$ is such $\lambda$. 
                \end{enumerate} 
                            
         \end{enumerate} 
    \item $a$ satisfies must- acceptance condition
      and may$_s$- rejection condition. Consider 
     all possibilities of $\lambda_x(a)$. 
      \begin{enumerate} 
         \item $\lambda_x(a) = \inL$: then $\lambda_x$
       is such $\lambda$. There is nothing to show.  

          \item $\lambda_x(a) = \outL$: Consider all 
         possible $0 \leq m \leq |\pre(a)|$. 

        \begin{enumerate}  

           \item $m = 0$: then $(a, a) \not\in R$. 
          $\lambda' \in \Lambda^A$ with: $\lambda'(a) = 
        \inL$; and $\lambda'(a_y) = \lambda_x(a_y)$
       for every $a_y \in (A \backslash \{a\})$
       is such $\lambda$.  
          \item $m = 1$: then $(a, a) \in R$, and 
          $l_2 \in \{0, 1\}$. 
           Consider all possibilities 
         of $l_2$:  
        \begin{enumerate}  
            \item $l_2 = 1$: Assume $\lambda' \in \Lambda^A$
         with $\lambda'(a) = \undecL$.  
         Then $a$ satisfies either may$_s$- acceptance 
         condition (when $l_1 = 0$) or else 
          not may- acceptance condition (when $l_1 = 1$). 
          Meanwhile, $a$ satisfies may$_s$- rejection
         condition. Thus, $a$'s label is designated under 
          $\lambda'$ (see Fig. \ref{fig_table}), which 
        is then such $\lambda$. 

            \item $l_1 = l_2 = 0$: Forms a sub-case of 
        the previous case. 
        \end{enumerate}
     \end{enumerate}  

        \item $\lambda_x(a) = \undecL$: then
         we have nothing to show. $\lambda_x$ is such $\lambda$.
      \end{enumerate} 
      
\item $a$ satisfies may$_s$- acceptance condition 
          and not- rejection condition. Consider 
          all possibilities of $\lambda_x(a)$. 
      \begin{enumerate} 
         \item $\lambda_x(a) = \inL$: then $\lambda_x$ 
       is such $\lambda$. There is nothing to show.  
  
         \item $\lambda_x(a) = \outL$: Consider 
           all possible $0 \leq m \leq |\pre(a)|$.
         \begin{enumerate}  
            \item $m = 0$: Then $(a, a) \not\in R$.   
                  By the so far given conditions, 
                 we have $l_1 = 0$ and $1 \leq l_2, l_3, l_4$. 
                 $\lambda' \in \Lambda^A$ with: 
      $\lambda'(a) = \undecL$; and $\lambda'(a_y) =
     \lambda_x(a_y)$ for every $a_y \in (A \backslash\{a\})$
     is such $\lambda$. 
            \item $m = 1$: Then $(a, a) \in R$. 
          By the so far given conditions, 
          we have $l_1 \in \{0,1\}$, $2 \leq l_2$, 
          and $1 \leq l_3, l_4$. 
          Consider all possibilities of $l_1$. 
       \begin{enumerate} 
          \item $l_1 = 1$: 
     Assume $\lambda' \in \Lambda^A$ with:
             $\lambda'(a) = \undecL$; and $\lambda'(a_y) 
           = \lambda_x(a_y)$ for every $a_y \in (A \backslash
        \{a\})$. Then $a$ satisfies not- acceptance 
            condition and not- acceptance condition
              under $\lambda'$. Thus,
              $a$'s label is designated under $\lambda'$,
           which is such $\lambda$.  
  
          \item $l_1 = 0$:  
      Assume $\lambda' \in \Lambda^A$ with: 
         $\lambda'(a) = \undecL$; and $\lambda'(a_y) 
           = \lambda_x(a_y)$ for every $a_y \in (A \backslash
        \{a\})$. Then $a$ satisfies may$_s$- acceptance 
      condition and not- rejection condition.
     Hence, $a$'s label is designated under $\lambda'$, 
     which is such $\lambda$. 
       \end{enumerate}
         \end{enumerate} 
        \item $\lambda_x(a) = \undecL$: Then
        $\lambda_x$ is such $\lambda$. There is nothing to show.
      \end{enumerate} 

     \item $a$ satisfies may$_s$- acceptance condition and 
        may$_s$- rejection condition. For each 
       value of $\lambda_x(a)$, there is nothing to show. 
     $\lambda_x$ is such $\lambda$.  

     \item $a$ satisfies not- acceptance condition and
     not- rejection condition. Consider all possible 
    values of $\lambda_x(a)$.
     
\begin{enumerate}

        \item $\lambda_x(a) = \inL$: Consider all possible 
         $0 \leq m \leq |\pre(a)|$. 
  
        \begin{enumerate}  

            \item $m = 0$: then $(a, a) \not\in R$.  
           Assume $\lambda' \in \Lambda^A$ with   
            $\lambda'(a) = \undecL$. Then $\lambda'$ is such 
           $\lambda$. 

            \item $m = 1$: then $(a, a) \in R$. 
         By the given constraints, we have 
        $1 \leq l_1, l_2$ and $2 \leq l_3, l_4$. 
           Assume $\lambda' \in \Lambda^A$ with
          $\lambda'(a) = \undecL$. Then 
          $a$ satisfies not- acceptance condition
       and not- rejection condition under $\lambda'$, which   
       is such $\lambda$. 
        \end{enumerate}  

       \item $\lambda_x(a) = \outL$: Similarly,   
   $\lambda' \in \Lambda^A$ with $\lambda'(a) = \undecL$
   is such $\lambda$.   

       \item $\lambda_x(a) = \undecL$: then $\lambda_x$
      is such $\lambda$. There is nothing to show. 
     \end{enumerate}
     \item All the other cases are symmetric to one of these
     cases. 
    \end{enumerate}
    This concludes the base case of the sub-induction
    for the base case of the main induction. 
    
     For the inductive case of the sub-induction for 
     the base case of the main induction,
    assume that the expected result holds
    for any $\max_{a \in A}\delta(F, \Delta(F, a)) \leq k$. 
    We show that it still holds for 
     $\max_{a \in A}\delta(F, \Delta(F, a)) = k+1$.    
    
    By induction hypothesis,
    we have some $\lambda_x \in \Lambda^A$ 
    under which, for every $a_z \in A$ with 
    $\delta(F, \Delta(F, a_z)) \leq k$, 
    $a_z$'s label is designated under $\lambda_x$. 
    Now, since we have that, for every $a_1, a_2 \in
    A$ with $\delta(F, \Delta(F, a_1)) = 
    \delta(F, \Delta(F, a_2)) = k+1$, 
    $(a_1, a_2) \not\in R$, 
it suffices to consider an arbitrary $a \in A$ 
    with $\max_{a \in A} \delta(F, \Delta(F, a)) = k+1$, 
    while assuming that every $a_4 \in (A \backslash \{a\})$'s 
    label is designated under $\lambda_x$. 
    We consider all possibilities of 
    what acceptance and rejection conditions 
   $a$ satisfies under $\lambda_x$.  
   \begin{enumerate} 
     \item $a$ satisfies must- acceptance 
      condition and not- rejection condition
      under $\lambda_x$. Consider cases by 
      whether $(a, a) \in R$. 
     \begin{enumerate} 
        \item $(a, a) \not\in R$: it is trivial 
          that $\lambda' \in \Lambda^A$ with: 
     $\lambda'(a) = \inL$; and $\lambda'(a_y) = 
    \lambda_x(a_y)$ for every $a_y \in (A \backslash \{a\})$ 
   is such $\lambda$. 
        \item $(a, a) \in R$:           
       Consider all possibilities of $\lambda_x(a)$.
       \begin{enumerate} 
          \item $\lambda_x(a) = \inL$: there is nothing to 
        show.
          \item $\lambda_x(a) = \outL$: 
        Suppose $n \leq |\pre(a)|$ distinct members
     of $\pre(a)$, $a_1, \ldots, a_n$, are such that 
    $\lambda_x(a_i) = \outL$ for every $ 1 \leq i \leq n$. 
          If $l_2 < n$, then assume $\lambda' \in \Lambda^A$
     with: $\lambda'(a) = \inL$; and $\lambda'(a_y) = \lambda_x(a_y)$ for every $a_y \in (A \backslash \{a\})$. 
        There is now $n-1$ distinct members of 
     rejected arguments in $\pre(a)$; however,
      it holds that $l_2 \leq n-1$, i.e. 
      $a$ satisfies must-acceptance condition under $\lambda'$.
      Now, by the given condition, it is not possible 
     for $a$ to also satisfy must- rejection condition
     under $\lambda'$. But this means (see 
  Fig. \ref{fig_table}) that 
    $a$'s label is designated under $\lambda'$, which 
    is such $\lambda$, as required.

    \indent{\ } \indent{\ } \indent{\ } \indent{\ } If $l_2 = n$ on the other hand, then assume 
      $\lambda'' \in \Lambda^A$ with: $\lambda''(a) = \undecL$;
    and $\lambda''(a_y) = \lambda_x(a_y)$ for every 
  $a_y \in (A \backslash \{a\})$. 
      There is now $n-1$ distinct members of 
    rejected arguments in $\pre(a)$, and therefore, 
      $a$ satisfies either may$_s$- acceptance condition
     (when $l_2 - 1 \leq l_1$) or else
        not- acceptance condition under $\lambda''$. 
        Meanwhile, $a$ 
      satisfies not- rejection condition
   under $\lambda''$. 
      Hence, $a$'s label is designated under 
  $\lambda''$, which is such $\lambda$, as required. 
         
        \item $\lambda_x(a) = \undecL$: 
          Assume $\lambda' \in \Lambda^A$ with 
   $\lambda'(a) = \inL$; and $\lambda'(a_y) = 
  \lambda_x(a_y)$ for every $a_y \in (A \backslash \{a\})$.  
        Then, $a$ satisfies must- acceptance condition, 
       and, by the given constraint, 
      does not satisfy must- rejection condition at the 
   same time. Hence, $\lambda'$ is such $\lambda$, as required.
       \end{enumerate}
        
     \end{enumerate}
   \item $a$ satisfies must- acceptance condition 
     and may$_s$- rejection condition under $\lambda_x$.  
      Similar to the previous case. 
\item $a$ satisfies may$_s$- acceptance condition
     and not- rejection condition under $\lambda_x$. 
     Consider cases by whether $(a, a) \in R$. 
     \begin{enumerate} 
       \item $(a, a) \not\in R$: it is trivial 
    that $\lambda' \in \Lambda^A$ with: $\lambda'(a) = 
      \undecL$; and $\lambda'(a_y) = \lambda_x(a_y)$ 
     for every $a_y \in (A \backslash \{a\})$ is such $\lambda$.
       \item $(a, a) \in R$:  Consider all possibilities
      of $\lambda_x(a)$.
        
       \begin{enumerate} 
          \item $\lambda_x(a) = \inL$:  there is nothing to show. 
          \item $\lambda_x(a) = \outL$:  
             Suppose $n \leq |\pre(a)|$ distinct members
     of $\pre(a)$, $a_1, \ldots, a_n$, are such that 
    $\lambda_x(a_i) = \outL$ for every $ 1 \leq i \leq n$. 
          If\linebreak $l_1 < n < l_2$, then assume $\lambda' \in \Lambda^A$
     with: $\lambda'(a) = \undecL$
    (or $\lambda'(a) = \inL$; but the proof is shorter 
    with $\lambda'(a) = \undecL$); and $\lambda'(a_y) = \lambda_x(a_y)$ for every $a_y \in (A \backslash \{a\})$. 
        There is now $n-1$ distinct members of 
     rejected arguments in $\pre(a)$; however,
      it holds that $l_1 \leq n-1$, i.e. 
      $a$ satisfies may$_s$- acceptance condition 
      under $\lambda'$. Meanwhile, $a$ satisfies 
      not- rejection condition.  
      Thus, $a$'s label is designated under $\lambda'$, 
   wich is such $\lambda$, as required.

       \indent{\ } \indent{\ } \indent{\ } \indent{\ } 
      If $l_1 = n$ on the other hand, then 
     assume $\lambda'' \in \Lambda^A$ with: 
    $\lambda''(a) = \undecL$; and $\lambda''(a_y) = 
  \lambda_x(a)$ for every $a_y \in (A \backslash \{a\})$. 
     Then $a$ satisfies not- acceptance condition 
   and not- rejection condition. Thus,
   $a$'s label is designated under $\lambda''$, which 
is such $\lambda$, as required. 
        
          \item $\lambda_x(a) = \undecL$:  
         there is nothing to show. 
       \end{enumerate}
     
     \end{enumerate}

   \item $a$ satisfies may$_s$- acceptance condition
    and may$_s$- rejection condition under $\lambda_x$.    
       Regardless of the value of $\lambda_x(a)$, 
       there is nothing to show.

  \item $a$ satisfies not- acceptance condition
    and not- rejection condition. Assume 
   $\lambda' \in \Lambda^A$ with: $\lambda'(a) = \undecL$; 
   and $\lambda'(a_y) = \lambda_x(a_y)$ for every
    $a_y \in (A \backslash \{a\})$. 
    Then, $a$ satisfies not- acceptance condition
   and not- rejection condition under $\lambda'$, 
   which is such $\lambda$, as required. 
 
  \item All the other cases are symmetric to one of these
   cases. 
  \end{enumerate} 
   This concludes the inductive case of the sub-induction 
   for the base case of the main induction.  \\

   \indent For the inductive case of the main induction,
    assume that the expected result holds 
    up to $\max_{a \in A}|\Delta(F, a)| = k\ (\in \mathbb{N})$. 
    We show that it still holds when 
    $\max_{a \in A}|\Delta(F, a)|\linebreak = k + 1$. 
    
   \indent For the base case of the sub-induction
     for the inductive case of the main induction, 
    assume that $\max_{a \in A}\delta(F, \Delta(F, a)) = 0$.
    Since there is no inter-strongly-connected-component attack 
    in $F$, let us assume some $a \in A$, then 
  without loss of generality, we assume 
    $\lambda_x \in \Lambda^A$ which is such that
    every $a_c \in (A \backslash \Delta(F, a))$'s label 
    is designated under $\lambda_x$ (Assumption \textsf{X}).  
    For the arguments in $\Delta(F, a)$, we consider 
    the following cases. 
    \begin{enumerate} 
     \item There exists some $a_x \in \Delta(A, a)$ such that
       $a$ satisfies either must- acceptance condition
      and not- rejection condition or else 
      must- rejection condition and not- acceptance 
      condition under every $\lambda_y \in \Lambda^A$: 
         then, let $(A_x, R_x, {f_Q}_x)$ denote 
       $\beta(F, a_x, \lambda_x)$, 
      we only need to consider 
        $\lambda' \in \Lambda^{A_x}$ under which 
        the label of every $a_u \in (A_x \cap \Delta(F, a))$ 
        is designated. 
        By induction hypothesis on 
       every strongly connected
        component of $(A_x \cap \Delta(F, a))$, 
       we obtain that there exists 
       some $\lambda'' \in \Lambda^{A_x}$ 
       such that every $a_z \in (A_x \cap \Delta(F, a))$'s
       label is designated under $\lambda''$. 
       By Lemma \ref{lem_reduction_lemma}, \ryuta{But 
       I must amend the lemma.} it follows that 
       there exists some $\lambda_q \in \Lambda^{A}$
       such that the label of every argument 
       in $\Delta(F, a)$ is designated under $\lambda_q$.    
       Thus, by Assumption \textsf{X}, 
        there exists some $\lambda_p \in \Lambda^A$
       under which every $a \in A$'s label is designated,  
       which is such $\lambda$, as required. 

     \item Otherwise, we show that 
        $\lambda' \in \Lambda^A$ with:  
      $\lambda'(a) = \undecL$ for every $a_x \in \Delta(F, a)$
      such that 
      every $a_y \in (A \backslash \Delta(F, a))$ is designated
     under $\lambda'$ is such $\lambda$.  
      Suppose  $a_y \in \Delta(F, a)$ 
       for which we  
     assume the following. 
   \begin{itemize} 
    \item 
     $(f_Q(a_y))^1 = (l_1^{a_y}, l_2^{a_y})$. \qquad 
      $(f_Q(a_y))^2 = (l_3^{a_y}, l_4^{a_y})$. 
  \end{itemize} 
     
    If neither $l_2^{a_y} = 0$ nor $l_4^{a_y} = 0$,
     then it holds trivially that $a_y$'s label 
     is designated under $\lambda'$. But suppose 
      $l_2^{a_y} = 0$, then $a_y$ satisfies 
   must-acceptance condition under any
    $\lambda_y \in \Lambda^A$; thus, 
     in order that $a_y$ not satisfy  
       not- rejection condition under every 
    $\lambda_w \in \Lambda^A$, it must be that 
        $0 = l_3^{a_y}$. However, this means
       that $a_y$'s label is designated under $\lambda'$ 
     (see Fig. \ref{fig_table}). 
     Similarly when $l_4^{a_y} = 0$.  
     As $a_y$ is chosen arbitrarily, 
     this applies to every argument in $\Delta(F, a)$.  
     Hence, indeed $\lambda'$ is such $\lambda$, as required. 
   \end{enumerate}
   This concludes the base case of the sub-induction 
   for the inductive case of the main induction. 

   \indent For the inductive case of the sub-induction
    for the inductive case of the main induction, 
    assume that the expected result holds for any
   $\max_{a \in A}\delta(F, \Delta(F, a)) \leq k$. We 
   show that it still holds for $\max_{a \in A}\delta(F, 
    \Delta(F, a)) = k+1$.

    Suppose some $a \in A$ with $\delta(F, \Delta(F, a))
    = k + 1$, and suppose some $\lambda_x \in \Lambda^A$ 
    under which the label of every 
    $a_d \in A$ with $\delta(F, \Delta(F, a_d)) \leq k$  
    is designated. By induction hypothesis, 
    such $\lambda_x$ exists.   
    Let $\Lambda^A_{\lambda_x}$ denote the set of 
    all $\lambda \in \Lambda^A$ satisfying 
   $\lambda(a) = \lambda_x(a)$ for every 
     $a \in A$ so long as $\delta(F,\Delta(F, a)) \leq k$
    holds, then 
    we show existence of some 
    $\lambda' \in \Lambda^A_{\lambda_x}$ under which 
     the label of 
    every $a_x \in \Delta(F, a)$ is designated. 
    We  consider the following cases.  
   \begin{enumerate} 
       \item There exists some $a_x \in \Delta(F, a)$
       such that $a_x$ satisfies either must- acceptance 
      condition and not- rejection condition or else
     must- rejection condition and not- acceptance 
      condition under every $\lambda_y \in 
     \Lambda^A_{\lambda_x}$: 
    then, let $(A_x, R_x, {f_Q}_x)$ denote 
   $\beta(F, a_x, \lambda_x)$ (it does not matter 
    which member of $\Lambda^A_{\lambda_x}$ to choose as 
    the third parameter of $\beta$), we 
      show the existence of 
     $\lambda'' \in \Lambda^A_{\lambda_x}$
      under which the label of every $a_u \in 
     (A_x \cap \Delta(F, a))$ is designated. But 
     this follows from induction hypothesis on 
   every strongly connected 
   component of $(A_x \cap \Delta(F, a))$. 
   Then, by Lemma \ref{lem_reduction_lemma}, it follows 
   that there exists some $\lambda_q \in \Lambda^A_{\lambda_x}$
under which the label of every $a_x \in \Delta(F, a)$   
    is designated, which is such $\lambda'$. Hence, 
  by Assumption \textsf{X}, there exists some 
   $\lambda_w \in \Lambda^A$ under which every $a \in A$'s 
   label is designated, which is such $\lambda$, as required. 
      
     \item Otherwise, we show that 
     $\lambda_z \in \Lambda^A_{\lambda_x}$ 
   with: $\lambda_z(a_x) = \undecL$ 
    for every $a_x \in \Delta(F, a)$ is such $\lambda'$. 
    Suppose $a_y \in \Delta(F, a)$ for which we assume the 
     following.
     \begin{itemize} 
      \item $(f_Q(a_y))^1 = (l_1^{a_y}, l_2^{a_y})$. \qquad 
            $(f_Q(a_y))^2 = (l_3^{a_y}, l_4^{a_y})$. 
     \end{itemize} 
     Suppose $n_1$ distinct members $a_1, \ldots, a_{n_1}$ 
     of $\pre(a_y)$ are such that 
     $\lambda_z(a_i) = \inL$ for every $1 \leq i \leq n_1$, 
     and also suppose $n_2$ distinct members 
     $a'_1, \ldots, a'_{n_2}$ of $\pre(a_y)$ are such that
     $\lambda_z(a'_j) = \outL$ for every $1 \leq j \leq n_2$.   
     $n_1$ ($n_2$) may be non-zero in case 
     some members of $\pre(a_y)$ have SCC-depth $k$. 

     \indent{\ }\indent{\ }\indent{\ }\indent{\ } 
    If neither $l_2^{a_y} \leq n_2$ nor 
     $l_4^{a_y} \leq n_1$, then it holds trivially 
     that $a_y$'s label is designated under $\lambda_z$. 
     But suppose $l_2^{a_4} \leq n_2$, then 
     $a_y$ satisfies must- acceptance condition 
     under any $\lambda_y \in \Lambda^A_{\lambda_x}$;
     thus, in order that $a_y$ not satisfy not- 
      rejection condition under every 
   $\lambda_w \in \Lambda^A_{\lambda_x}$, 
     it must be that 
     $l_3^{a_y} \leq n_1$. However, this means that 
     $a_y$'s label is designated under $\lambda_z$. 
     Similarly when $l_4^{a_4} \leq n_1$. 
     As $a_y$ is cosen arbitrarily, this applies to 
     every argument in $\Delta(F, a)$. Hence, 
     $\lambda_z$ is such $\lambda'$. 
   \end{enumerate}
   By Assumption \textsf{X}, there exists some $\lambda'' 
   \in \Lambda^A_{\lambda_x}$ under which  
     every $a \in A$'s label is designated. Thus,
    $\lambda''$ is 
     such $\lambda$, as required.

    This concludes the inductive case of the sub-induction
     for the inductive case of the main induction. 
   \hfill$\Box$ 
\end{proof} 

\section{Further Extension with Preference on Scales} 
In this section, we further extend $\mathcal{F}$  
by equipping  it with preference on scales. 
The idea is as follows. \ryuta{And the idea comes here.}  

\begin{definition}[Nuance tuple with preference]\label{def_nuance_tuple_with_preference} 
  We define a nuance tuple with preference to be 
  $(\textbf{X}_1, \textbf{o}, \textbf{X}_2)$ for 
   some $\textbf{X}_1, \textbf{X}_2 \in \mathbb{N} \times 
  \mathbb{N}$, and some $\textbf{o} \in \{\dot{<}, 
   \dot{=}, \dot{>}\}$. We denote the class of 
   all nuance tuples with preference by $\dot{\mathcal{Q}}$.
   For any $\dot{Q} \in \mathbb{Q}$, 
   we call $(Q)^1$ its may-must acceptance scale, 
  $(Q)^2$ preference, and $(Q)^3$ its may-must rejection
  scale. 
\end{definition} 

\begin{definition}[May-must argumentation with may-must 
 scales and preference]\label{def_abstract_argumentatin_with} 
  We define a (finite) abstract argumentation with may-must
  scales with preference to be a tuple $(A, R, \fqdot)$ with: 
  $A \subseteq_{\text{fin}} \mathcal{A}$; 
  $R \subseteq A \times A$; and $\fqdot: A \rightarrow 
   \dot{\mathcal{Q}}$, such that 
   $((\fqdot(a))^i)^1 \leq ((\fqdot(a))^i)^2$ 
   for every $a \in A$ and every $i \in \{1,3\}$.  
\end{definition} 
The reason for the given condition is the same 
as in the previous section. 
\begin{figure}[h]
$\begin{bNiceMatrix}[first-row,first-col]  
       [\dot{>}]     &[\text{must-r}]  & 
     [\text{may$_s$-r}] & [\text{not-r}] \\
           [\text{must-a}] & \inL & \inL  & 
	      \inL\\
           [\text{may$_s$-a}] & \inL, \outL & 
	    \inL, \outL  & \inL  \\
           [\text{not-a}] & \outL & \inL,\outL  & \inL 
         \end{bNiceMatrix}$ 
$\begin{bNiceMatrix}[first-row,first-col]  
        [\dot{=}]    &[\text{must-r}]  & 
     [\text{may$_s$-r}] & [\text{not-r}] \\
           [\text{must-a}] & \undecL & \inL ? & 
	      \inL\\
           [\text{may$_s$-a}] & \outL ? & 
	    \textsf{any}  & \inL ? \\
           [\text{not-a}] & \outL & \outL ?  & \undecL 
         \end{bNiceMatrix}$\\
\indent{\ }\!\!\! $\begin{bNiceMatrix}[first-row,first-col]  
        [\dot{<}]    &[\text{must-r}]  & 
     [\text{may$_s$-r}] & [\text{not-r}] \\
           [\text{must-a}] & \outL & \inL, \outL & 
	      \inL\\
           [\text{may$_s$-a}] & \outL  & 
	    \inL, \outL  & \inL, \outL  \\
           [\text{not-a}] & \outL & \outL   & \outL 
         \end{bNiceMatrix}$
\caption{Corresponding expected acceptance status  
of an argument for each combination of 
may- must- 
acceptance and rejection conditions. \textsf{any} 
is any of $\inL, \outL, \undecL$, \textsf{in}$?$
is any of $\inL, \undecL$, \textsf{out}$?$ 
is any of $\outL, \undecL$. Moreover, 
$\inL, \outL$ is any of $\inL, \outL$.}
\label{fig_table_2} 
\end{figure} 

}

\section{Conclusion}
We developed a family of may-must argumentations by 
broadening 
the labelling-based approach for Dung abstract 
argumentation \cite{Caminada06}. 
We presented several semantics and made comparisons between them.

Just as a complete labelling of  
$(A, R) \in \mathcal{F}^{\Dung}$ 
is one that  
assigns to each argument $a \in A$ the label 
that is expected from   the acceptance and the rejection 
conditions as determined by 
the labels it assigns to $\pre(a)$, 
so is `almost' a complete labelling of $(A, R, f_Q) \in 
\mathcal{F}$. As 
we have identified, however, 
such a labelling may not actually exist. We 
proposed a way of addressing it. We 
also noted the connection to   
Dung abstract argumentation labelling. 
Incidentally, we are not the first 
to have proposed the use of abstract interpretation 
in formal argumentation, as a preprint exists \cite{ArisakaDauphin18}. 
\subsection{Open problem}   
As we  identified, for a label-based argumentation with sufficiently 
general local labelling conditions, an exact labelling may not 
exist. However, for some interesting labelling conditions 
of Dung's \cite{Caminada06}, it does exist. Also, for any variation 
of may-must argumentation, if every argument could 
satisfy only may$_s$-a and may$_s$-r, then an exact labelling 
again obviously exists. 

One open problem concerning this matter is just how much 
flexibility we can then allow in the local labelling conditions 
before existence of an exact labelling breaks. 
Since the local labelling conditions 
can be trivially regarded as a specification of a global labelling, 
verification of the existence is tantamount to 
verification of the correctness of a specification 
if in programming. Thus, identification 
of a boundary is an important research problem 
to be undertaken. 

}

\bibliography{references}

\begin{thebibliography}{10}

\bibitem{Arieli12}
O.~Arieli.
\newblock {Conflict-Tolerant Semantics for Argumentation Frameworks}.
\newblock In {\em {JELIA}}, pages 28--40, 2012.

\bibitem{AHT19}
R.~Arisaka, M.~Hagiwara, and T.~Ito.
\newblock {Deception/Honesty Detection and (Mis)trust Building in Manipulable
  Multi-Agent Argumentation: An Insight}.
\newblock In {\em {PRIMA}}, pages 443--451, 2019.

\bibitem{Arisaka2020b}
R.~Arisaka and T.~Ito.
\newblock {Numerical Abstract Persuasion Argumentation for Expressing
  Concurrent Multi-Agent Negotiations}.
\newblock In {\em {IJCAI Best of Workshops 2019 (to appear)}}, 2019.

\bibitem{arisaka2020a}
R.~Arisaka and T.~Ito.
\newblock {Broadening Label-based Argumentation Semantics with May-Must
  Scales}.
\newblock In {\em {CLAR}}, pages 22--41, 2020.

\bibitem{ArisakaSantini19}
R.~Arisaka, F.~Santini, and S.~Bistarelli.
\newblock {Block Argumentation}.
\newblock In {\em {PRIMA}}, pages 618--626, 2019.

\bibitem{ArisakaSatoh18}
R.~Arisaka and K.~Satoh.
\newblock {Abstract Argumentation / Persuasion / Dynamics}.
\newblock In {\em {PRIMA}}, pages 331--343, 2018.

\bibitem{Baroni15}
P.~Baroni, M.~Romano, F.~Toni, M.~Aurisicchio, and G.~Bertanza.
\newblock Automatic evaluation of design alternatives with quantitative
  argumentation.
\newblock {\em {Argument and Computation}}, 6(1):24--49, 2015.

\bibitem{Barringer12}
H.~Barringer, D.~M. Gabbay, and J.~Woods.
\newblock {\em {Argument \& Computation}}, 3(2-3):143--202, 2012.

\bibitem{Hunter01}
P.~Besnard and A.~Hunter.
\newblock A logic-based theory of deductive arguments.
\newblock {\em {Artificial Intelligence}}, 128(1-2):203--235, 2001.

\bibitem{Bistarelli17}
S.~Bistarelli and F.~Santini.
\newblock {A Hasse Diagram for Weighted Sceptical Semantics with a
  Unique-Status Grounded Semantics}.
\newblock In {\em {LPNMR}}, pages 49--56, 2017.

\bibitem{Bogaerts19}
B.~Bogaerts.
\newblock {Weighted abstract dialectical frameworks through the lens of
  approximation fixpoint theory}.
\newblock In {\em {AAAI}}, pages 2686--2693, 2019.

\bibitem{Bonzon16}
E.~Bonzon, J.~Delobelle, S.~Konieczny, and N.~Maudet.
\newblock {A Comparative Study of Ranking-Based Semantics for Abstract
  Argumentation}.
\newblock In {\em {AAAI}}, pages 914--920, 2016.

\bibitem{Brewka18}
G.~Brewka, J.~P{\"{u}}hrer, H.~Strass, J.~P. Wallner, and S.~Woltran.
\newblock {Weighted Abstract Dialectical Frameworks: Extended and Revised
  Report}.
\newblock {\em CoRR}, abs/1806.07717, 2018.

\bibitem{Brewka13}
G.~Brewka, H.~Strass, S.~Ellmauthaler, J.~Wallner, and S.~Woltran.
\newblock {Abstract Dialectical Frameworks Revisited}.
\newblock In {\em {IJCAI}}, 2013.

\bibitem{Caminada06}
M.~Caminada.
\newblock {On the Issue of Reinstatement in Argumentation}.
\newblock In {\em {JELIA}}, pages 111--123, 2006.

\bibitem{Caminada09}
M.~Caminada and D.~Gabbay.
\newblock {A Logical Account of Formal Argumentation}.
\newblock {\em {Studia Logica}}, 93(2-3):109--145, 2009.

\bibitem{Cayrol05b}
C.~Cayrol and M.~C. Lagasquie-Schiex.
\newblock {Graduality in Argumentation}.
\newblock {\em {Journal of Artificial Intelligence Research}}, 23:245--297,
  2005.

\bibitem{daCostaPereira11}
C.~da~Costa~Pereira, A.~G.~B. Tettamanzi, and S.~Villata.
\newblock {Changing One's Mind: Erase or Rewind? Possibilistic Belief Revision
  with Fuzzy Argumentation Based on Trust}.
\newblock In {\em {IJCAI}}, pages 164--171, 2011.

\bibitem{Davey02}
B.~A. Davey and H.~A. Priestley.
\newblock {\em Introduction to Lattices and Order}.
\newblock Cambridge University Press, 2002.

\bibitem{Dimopoulos14}
Y.~Dimopoulos and P.~Moraitis.
\newblock {Advances in Argumentation-Based Negotiation}.
\newblock In {\em {Negotiation and Argumentation in Multi-agent systems:
  Fundamentals, Theories, Systems and Applications }}, pages 82--125, 2014.

\bibitem{Dung95}
P.~M. Dung.
\newblock On the {Acceptability} of {Arguments} and {Its} {Fundamental} {Role}
  in {Nonmonotonic} {Reasoning}, {Logic Programming}, and n-{Person} {Games}.
\newblock {\em Artificial {Intelligence}}, 77(2):321--357, 1995.

\bibitem{Dunne11}
P.~E. Dunne, A.~Hunter, P.~McBurney, S.~Parsons, and M.~Wooldridge.
\newblock {Weighted argument systems: Basic definitions, algorithms, and
  complexity results}.
\newblock {\em {Artificial Intelligence}}, 175(2):457--486, 2011.

\bibitem{Gabbay14}
D.~M. Gabbay and O.~Rodrigues.
\newblock {An equational approach to the merging of argumentation networks}.
\newblock {\em {Journal of Logic and Computation}}, 24(6):1253--1277, 2014.

\bibitem{Garson13}
J.~Garson.
\newblock Modal {Logic}.
\newblock In E.~N. Zalta, editor, {\em The {Stanford} {Encyclopedia} of
  {Philosophy}}. 2018.

\bibitem{Grossi19}
D.~Grossi and S.~Modgil.
\newblock {On the Graded Acceptability of Arguments in Abstract and
  Instantiated Argumentation}.
\newblock {\em {Artificial Intelligence}}, 275:138--173, 2019.

\bibitem{Grossi13}
D.~Grossi and W.~van~der Hoek.
\newblock {Audience-based uncertainty in abstract argument games}.
\newblock In {\em {IJCAI}}, pages 143--149, 2013.

\bibitem{Hadjinikolis13}
C.~Hadjinikolis, Y.~Siantos, S.~Modgil, E.~Black, and P.~McBurney.
\newblock {Opponent modelling in persuasion dialogues}.
\newblock In {\em {IJCAI}}, pages 164--170, 2013.

\bibitem{Hadoux15}
E.~Hadoux, A.~Beynier, N.~Maudet, P.~Weng, and A.~Hunter.
\newblock {Optimization of Probabilistic Argumentation with Markov Decision
  Models}.
\newblock In {\em {IJCAI}}, pages 2004--2010, 2015.

\bibitem{Hadoux17}
E.~Hadoux and A.~Hunter.
\newblock {Strategic Sequences of Arguments for Persuasion Using Decision
  Trees}.
\newblock In {\em {AAAI}}, pages 1128--1134, 2017.

\bibitem{Hunter18}
A.~Hunter.
\newblock {Towards a framework for computational persuasion with applications
  in behaviour change}.
\newblock {\em {Argument \& Computation}}, 9(1):15--40, 2018.

\bibitem{Jakobovits99}
H.~Jakobovits and D.~Vermeir.
\newblock {Robust Semantics for Argumentation Frameworks}.
\newblock {\em {Journal of Logic and Computation}}, 9:215--261, 1999.

\bibitem{Janssen08}
J.~Janssen, M.~D. Cock, and D.~Vermeir.
\newblock {Fuzzy Argumentation Frameworks}.
\newblock In {\em {IPMU}}, pages 513--520, 2008.

\bibitem{Kakas05}
A.~C. Kakas, N.~Maudet, and P.~Maritis.
\newblock {Modular Representation of Agent Interaction Rules through
  Argumentation}.
\newblock {\em {Autonomous Agents and Multi-Agent Systems}}, 11(2):189--206,
  2005.

\bibitem{Leite11}
J.~Leite and J.~Martins.
\newblock {Social Abstract Argumentation}.
\newblock In {\em {IJCAI}}, pages 2287--2292, 2011.

\bibitem{Nielsen06}
S.~H. Nielsen and S.~Parsons.
\newblock A generalization of {Dung's} {Abstract} {Framework} for
  {Argumentation}.
\newblock In {\em {ArgMAS}}, pages 54--73, 2006.

\bibitem{Parsons05}
S.~Parsons and E.~Sklar.
\newblock {How agents alter their beliefs after an argumentation-based
  dialogue}.
\newblock In {\em {ArgMAS}}, pages 297--312, 2005.

\bibitem{Rahwan09}
I.~Rahwan and K.~Larson.
\newblock {Argumentation and game theory}.
\newblock In {\em {Argumentation in Artificial Intelligence}}, pages 321--339.
  Springer, 2009.

\bibitem{Rahwan03}
I.~Rahwan, S.~D. Ramchurn, N.~R. Jennings, P.~Mcburney, S.~Parsons, and
  L.~Sonenberg.
\newblock {Argumentation-based Negotiation}.
\newblock {\em Knowledge Engineering Review}, 18(4):343--375, 2003.

\bibitem{Rendsvig19}
R.~Rendsvig and J.~Symons.
\newblock Epistemic logic.
\newblock In E.~N. Zalta, editor, {\em The Stanford Encyclopedia of
  Philosophy}. Metaphysics Research Lab, Stanford University, 2019.

\bibitem{Rienstra13}
T.~Rienstra, M.~Thimm, and N.~Oren.
\newblock {Opponent Models with Uncertainty for Strategic Argumentation}.
\newblock In {\em {IJCAI}}, pages 332--338, 2013.

\bibitem{Riveret08}
R.~Riveret and H.~Prakken.
\newblock {Heuristics in argumentation: A game theory investigation}.
\newblock In {\em {COMMA}}, pages 324--335, 2008.

\bibitem{Sakama12}
C.~Sakama.
\newblock {Dishonest Arguments in Debate Games}.
\newblock In {\em {COMMA}}, pages 177--184, 2012.

\bibitem{Thimm14}
M.~Thimm.
\newblock {Strategic Argumentation in Multi-Agent Systems}.
\newblock {\em {Künstliche Intelligenz}}, 28(3):159--168, 2014.

\end{thebibliography}
\bibliographystyle{abbrv} 
      
\end{document}